\documentclass[11pt]{article}

\usepackage{booktabs} 
\usepackage[normalem]{ulem}
\usepackage{authblk}

\usepackage[dvipsnames,usenames]{xcolor}

\usepackage{url}

\usepackage[colorlinks=true,citecolor=ForestGreen]{hyperref}

\usepackage[left=1.25in, top=1in, bottom=1in, right=1.25in]{geometry}
\title{Disentangling Sampling and Labeling Bias for Learning in Large-Output Spaces}

\author{Ankit Singh Rawat}
\author{Aditya Krishna Menon}
\author{Wittawat Jitkrittum}
\author{Sadeep Jayasumana}
\author{Felix X. Yu}
\author{Sashank Reddi}
\author{Sanjiv Kumar}
\affil{Google Research, New York \\
{\small \texttt{\{ankitsrawat,adityakmenon,wittawat,sadeep,felixyu,sashank,sanjivk\}@google.com}}}

\usepackage[utf8]{inputenc} 
\usepackage[T1]{fontenc}    
\usepackage{booktabs}       
\usepackage{amsfonts}       
\usepackage{nicefrac}       
\usepackage{microtype}      

\usepackage{natbib}

\usepackage{float,subcaption}
\usepackage{graphicx}
\usepackage[export]{adjustbox}
\usepackage[inline]{enumitem}

\usepackage{amsmath,amssymb,amsthm}
\usepackage[mathscr]{eucal}
\usepackage{stmaryrd}
\usepackage{accents}
\usepackage{soul}
\usepackage{tikz}
\usetikzlibrary{arrows,calc,shapes.geometric,decorations.markings,positioning}
\usepackage{upgreek}

\usepackage[normalem]{ulem}
\newcommand{\headSymb}{\uparrow}
\newcommand{\tailSymb}{\downarrow}

\usepackage{booktabs,colortbl,multirow}

\usepackage{algorithm,algorithmic}

\graphicspath{{figs//}}

\allowdisplaybreaks

\newcommand{\defEq}{\stackrel{.}{=}}

\newcommand{\indicator}[1]{\llbracket #1 \rrbracket}

\newcommand{\argmaxUnique}[2]{{\operatorname{argmax }\nolimits_{#1}}\, #2}

\renewcommand{\Pr}{\mathbb{P}}
\newcommand{\E}[2]{\underset{#1}{\mathbb{E}}\left[ #2 \right]}
\newcommand{\V}[2]{\underset{#1}{\mathbb{V}}\left[ #2 \right]}

\newcommand{\X}{{x}}
\newcommand{\Y}{{y}}

\newcommand{\NCal}{\mathscr{N}}

\newcommand{\XCal}{\mathscr{X}}
\newcommand{\YCal}{\mathscr{Y}}

\newcommand{\Real}{\mathbb{R}}

\newcommand{\precision}{\mathrm{Precision}@}
\newcommand{\recall}{\mathrm{Recall}@}

\newcommand{\expp}{{\sc Exp }}

\newtheorem{lemma}{Lemma}

\newtheorem{theorem}[lemma]{Theorem}

\newtheorem{proposition}[lemma]{Proposition}

\theoremstyle{definition}
\newtheorem{remark}{Remark}

\newcommand{\delicioussmall}{{\sc Delicious}}

\newcommand{\amazonsmall}{{\sc AmazonCat-13k}}

\newcommand{\wikilshtc}{{\sc WikiLSHTC-325k}}

\newcommand{\NegLabels}{\mathscr{N}}

\newcommand{\qDist}{q}
\newcommand{\qY}{q_{y}}
\newcommand{\qYPrime}{q_{y'}}

\newcommand{\changesNew}[1]{\textcolor{Black}{#1}}
\newcommand{\updates}[2]{\textcolor{Black}{#1}}

\newcommand{\omission}[1]{}

\begin{document}

\maketitle

\begin{abstract}
Negative sampling schemes enable efficient training given a large number of classes,
by
offering a means to approximate 
a computationally expensive loss function that takes all labels into account. 
In this paper, we 
present a new
connection between these schemes and loss modification techniques for countering \emph{label imbalance}. 
We show that different 
negative sampling schemes
\emph{implicitly}
trade-off
performance on dominant versus rare labels.
Further, we provide a unified means to 
explicitly
tackle both \emph{sampling bias}, arising from working with a subset of all labels, and \emph{labeling bias}, which is inherent to the data due to label imbalance. 
We empirically verify our findings on long-tail classification and retrieval benchmarks.
\end{abstract}

\section{Introduction}
\label{sec:intro}
Classification problems with a large number 
of labels
arise in language modelling~\citep{Mikolov:2013,Levy:2014}, 
recommender systems~\citep{Covington:2016,Xu:2016},
and information retrieval~\citep{Agrawal:2013,Prabhu:2014}.
Such 
\emph{large-output}
problems pose a core challenge:
losses
such as the softmax cross-entropy
can be prohibitive 
to optimise,
as they depend on the \emph{entire} set of labels. Several works have thus devised
\emph{negative sampling}
schemes for efficiently and effectively approximating such losses~\citep{Bengio:2008,Blanc:2018,Ruiz:2018,Bamler:2020}.

Broadly,
negative sampling techniques
\emph{sample} a subset of ``negative'' labels,
which are used to contrast against the observed ``positive'' labels.
One 
further applies a
suitable \emph{weighting} on these ``negatives'', which ostensibly corrects the 
\emph{sampling bias} introduced by 
the dependence on a random subset of labels.
Intuitively, 
such bias
assesses how closely a scheme
approximates the 
unsampled loss on the full label set.
This bias is well understood for \emph{sampled softmax} schemes (see, e.g.,~\citet{Bengio:2008});
surprisingly,
however,
far less is understood about other popular schemes, 
e.g.,
within-batch and uniform sampling (cf.~\S\ref{sec:background-neg-sampling}).

In this paper,
we systematically study the sampling bias of 
generic negative sampling schemes,
with two main findings.
First, we precisely characterise the \emph{implicit losses} optimised by such schemes (\S\ref{sec:sampling-bias}).
This reveals that, e.g., uniform and within-batch sampling
do \emph{not} faithfully approximate the unsampled loss;
thus, 
measured by the yardstick of sampling bias,
such schemes appear woefully sub-optimal.

Intriguingly,
however,
our analysis further reveals a novel \emph{labeling bias} perspective on negative sampling.
Specifically, 
we show that 
when the label distribution is skewed
---
as is common with a large number of labels~\citep{Jain:2016}
---
different
negative sampling schemes 
\emph{implicitly} trade-off 
performance on \emph{dominant versus rare} labels (\S\ref{sec:labelling-bias}).
Consequently,
while
certain schemes may incur a sampling bias, 
they may prove highly effective 
in
modeling rare labels.

Put together, our analysis
reveals that
negative sampling can account for \emph{both} sampling and labeling bias,
with the latter done \emph{implicitly}.
Note that the two concerns fundamentally differ:
the former results from working with a random subset of labels,
while 
the latter results from an inherent property of the data.
By disentangling these concerns,
we arrive at a unified means to \emph{explicitly} tackle both biases (\S\ref{sec:unified}).

In summary, our contributions are:
\begin{enumerate}[label=(\roman*),itemsep=0pt,topsep=0pt,leftmargin=16pt]
    \item we precisely characterise the effect of 
    generic negative sampling schemes
    by explicating
    the \emph{implicit} losses they optimise (\S\ref{sec:sampling-bias}; Theorem~\ref{thm:implicit-xent}).
    This shows that popular schemes
    such as uniform and within-batch negative sampling~\citep{Hidasi:2016,Wang:2017b}
    do \emph{not} faithfully approximate the unsampled loss.

    \item we 
    provide a novel connection between 
    negative sampling and long-tail learning,
    and show that
    negative sampling
    can
    \emph{implicitly}
    trade off performance on dominant versus rare labels
    (\S\ref{sec:long-tail-view}; Proposition~\ref{prop:loss-connection}). This 
    implies
    that,
    e.g., within-batch sampling
    can boost performance on rare labels
    despite not approximating the unsampled loss.
    Further, this yields a unified means to tailor sampling schemes to focus on
    particular label slices (\S\ref{sec:unified}).

    \item we empirically verify 
    our analysis of sampling schemes'
    implicit
    dominant versus rare label tradeoffs (\S\ref{sec:experiments}).
\end{enumerate}

\section{Background and notation}
\label{sec:background}
Consider a multiclass classification problem with instances $\XCal$, labels $\YCal = [ L ] \defEq \{ 1, 2, \ldots, L \}$,
and distribution $\Pr$ over $\XCal \times \YCal$. 
Given a sample $\{ ( x_n, y_n ) \}_{n = 1}^N \sim \Pr^{N}$, 
we seek a scorer $f \colon \XCal \to \Real^L$ minimising the 
misclassification error 
\begin{equation}
    \label{eqn:zero-one}
    L_{\rm err}( f ) = {\Pr}_{\X, \Y}\big( y \notin \argmaxUnique{y' \in [L]}{ f_{y'}( x ) } \big).
\end{equation}
To achieve this,
a common strategy is to 
minimise
a surrogate loss
$\ell \colon \YCal \times \Real^L \to \Real_+$,
where $\ell( y, f( x ) )$ is the loss incurred for predicting $f( x )$ when the true label is $y$. 
A ubiquitous choice is
the softmax cross-entropy,
\begin{align}
    \label{eqn:softmax-xent}
    \ell( y, f( x ) ) 
    &= -f_y( x ) + \log\Big[ \sum\nolimits_{y' \in [L]} e^{f_{y'}( x )} \Big] =\log \Big[1 + \sum\nolimits_{y' \neq y}e^{f_{y'}(x) - f_y(x)} \Big].
\end{align}
Equivalently, this is the log-loss under the \emph{softmax distribution} $p_y( x ) \propto \exp( f_y( x ) )$
with \emph{logits} $f_y( x )$.
Alternately, one may use a \emph{decoupled} loss~\citep{zhang2004statistical}
\begin{equation}
    \label{eqn:decoupled}
    \ell( y, f( x ) ) = \phi( f_y( x ) ) + \sum\nolimits_{y' \neq y} \varphi( -f_{y'}( x ) ),
\end{equation}
where $\phi, \varphi \colon \Real \to \Real_+$
are \emph{margin losses} for binary classification,
e.g., hinge loss. In~\eqref{eqn:softmax-xent} and~\eqref{eqn:decoupled},
we may consider $y$ to be a ``positive'' label, and each $y' \neq y$ a ``negative'' label.

In \emph{large-output} settings where number of labels $L$ is large,
there are at least two additional considerations in designing a loss.
First, most labels may have only a few associated samples~\citep{HeGa09,Buda:2017}.
Second, computing the losses~\eqref{eqn:softmax-xent} and~\eqref{eqn:decoupled} may be prohibitively expensive
owing to their
linear complexity in $L$~\citep{Bengio:2008,Gutmann:2012,Reddi19a}.
These problems have been separately addressed in the literatures on \emph{long-tail learning} and \emph{negative sampling}.

\subsection{Long-tail learning methods}
\label{sec:background-long-tail}

In \emph{long-tail} learning settings,
$\pi_y \defEq \Pr( y )$ is highly skewed,
and so many labels have few associated samples.
Here, 
achieving low misclassification error~\eqref{eqn:zero-one} 
may mask poor performance on rare classes,
which is undesirable.
To cope with this problem,
existing approaches include 
post-processing model outputs~\citep{Fawcett:1996,Zhou:2006},
re-balancing the training data~\citep{Chawla:2002,Wallace:2011},
and modifying the loss function~\citep{Xie:1989,Morik:1999}.
One strategy is to 
encourage larger classification margins for rare labels.
Such approaches augment the softmax cross-entropy with
\emph{pairwise label margins} $\rho_{yy'}$~\citep{Menon:2020}:
\begin{align}
    \label{eqn:unified-margin-loss}
    \ell^{\rho}_{\mathrm{mar}}( y, f( x ) ) = \log\Big[ 1 + \sum\nolimits_{y' \neq y} {\rho_{y y'}} \cdot e^{f_{y'}( x ) - f_{y}( x )} \Big],
\end{align}
where $\log \rho_{y y'}$
represents the
desired gap or ``margin'' between scores for $y$ and $y'$.
For suitable $\rho_{yy'}$, this loss allows for greater emphasis on rare classes' predictions. For example,~\citet{Cao:2019} proposed 
an ``adaptive'' loss with $\rho_{yy'} \propto {\pi_{y}^{-1/4}}$.
This upweights rare ``positive'' labels $y$ to encourage a larger gap $f_{y}(x) - f_{y'}(x)$ for such labels.
\citet{Tan:2020} proposed
an ``equalised'' loss with
$\rho_{y y'} = F( \pi_{y'} )$,
for increasing $F \colon [0, 1] \to \Real_+$.
This downweights rare ``negatives'' $y'$, thus preventing $f_{y'}$ from being too small. Finally,~\citet{Menon:2020} proposed 
a ``logit adjusted'' loss with
$\rho_{yy'} = \frac{\pi_{y'}}{\pi_y}$.
This encourages larger margin between positives $y$ that are \emph{relatively} rare compared to negative $y'$.

Another popular strategy is to re-weight the individual samples, so that rare labels receive a higher penalty;
however, these are generally outperformed by margin approaches~\citep{Cao:2019}.
Further, we shall demonstrate the latter have a strong connection to negative sampling.

\subsection{Negative sampling methods}
\label{sec:background-neg-sampling}
\begin{table*}[!t]
    \centering
    \footnotesize
    \renewcommand{\arraystretch}{1.0}
    \begin{tabular}{@{}ll@{}}
        \toprule
        {\bf Sampling distribution $q$} & {\bf Description} \\
        \midrule
        Current model probabilities & Model-based negatives \\        
        Uniform distribution & Uniform negatives \\
        Training distribution & Within-batch negatives \\
        \bottomrule
    \end{tabular}
    \qquad
    \begin{tabular}{@{}ll@{}}
        \toprule
        {\bf Weight $w_{yy'}$} & {\bf Description} \\
        \midrule
        ${1}/{(m \cdot \qYPrime)}$ & Importance weighting \\        
        $1/m$ & Constant weighting \\
        ${\qY}/{\qYPrime}$ & Relative weighting \\
        \bottomrule
    \end{tabular}
    \qquad    
    
    \caption{Summary of sampling distributions $\qDist \in \Delta_{[L]}$
    for an example $(x, y)$, 
    and weights $w_{yy'}$ on the negatives $y'$.
    Here, $m$ is the number of sampled negatives. Note that one may exclude the positive label $y$ either while sampling, or via weighting.
    }
    \label{tbl:summary-decoupled}
    \vspace{-\baselineskip}
\end{table*}

At a high level,
negative sampling proceeds as follows:
given an example $( x, y )$,
we 
draw an i.i.d.~sample 
$\NegLabels = \{ y'_1, \ldots, y'_m \}$ of $m$ ``negative'' labels from some distribution $\qDist \in \Delta_{[L]}$.
We now compute a loss using $\NegLabels$ in place of $[ L ]$,
applying suitable \emph{weights} $w_{yy'} \geq 0$.
For example, applied to the softmax cross-entropy in \eqref{eqn:softmax-xent},
we get
\begin{equation}
    \label{eqn:weighted-sampled-softmax}
    \ell( y, f( x ); \NegLabels, w ) = \log\Big[ 1 + \sum\nolimits_{y' \in \NegLabels} w_{yy'} \cdot e^{f_{y'}( x ) - f_{y}( x )} \Big].
\end{equation}
Similarly, the decoupled loss in~\eqref{eqn:decoupled} becomes
\begin{equation}
    \label{eqn:weighted-sampled-decoupled}
    \ell( y, f( x ); \NegLabels, w ) = \phi( f_y( x ) ) + \sum_{y' \in \NegLabels} w_{yy'} \cdot \varphi( -f_{y'}( x ) ).
\end{equation}

\begin{remark}
The distribution $\qDist$ may or may not include the positive label $y$ amongst the negative sample $\NegLabels$.
In the latter case, $\qY = 0$. See Appendix~\ref{app:additional-q-w} for further discussion.
\end{remark}

\begin{remark}
We write $q, w$ for brevity;
in general,
each of these may additionally depend on $(x,y)$.
\end{remark}

There are several canonical $(q, w)$ choices in the literature. For the sampling distribution $q$,
popular choices are:
\begin{itemize}[label={--},itemsep=0pt,topsep=0pt,leftmargin=12pt]
    \item 
    \textbf{Model-based}.~Here, $q$ is taken to be the
    current model probabilities;
    e.g., 
    when training $f \colon \XCal \to \Real^L$
    to minimise the softmax cross-entropy~\eqref{eqn:softmax-xent},
    $\qYPrime \propto \exp( f_{y'}( x ) )$,
    or an 
    efficient approximation thereof~\citep{Bengio:2008,Blanc:2018,Rawat:2019}.

    \item 
    \textbf{Uniform}.~Here,
    $q$ is uniform over all labels~\citep{Hidasi:2016,Wang:2017b},
    i.e.,
    $\qYPrime = \frac{1}{L}$
    for every $y' \in [L]$.

    \item 
    \textbf{Within-batch}.~Suppose we perform minibatch-based model training,
    so that we iteratively make updates based on a random subset $\{ ( x_b, y_b ) \}_{b = 1}^B$ of  training samples.
    In within-batch negative sampling, 
    we use all labels in the minibatch as our negatives, 
    i.e., $\NegLabels = \bigcup_{b=1}^{B} \{ y_b \}$~\citep{Chen:2017,Broscheit:2020}. In expectation, this is equivalent to $q = \pi$, the training label distribution.
\end{itemize}

For the weighting scheme $w_{yy'}$, popular choices are:
\begin{itemize}[label={--},itemsep=0pt,topsep=0pt,leftmargin=12pt]
    \item 
    \textbf{Importance weighting}.~Here,
    $w_{yy'} = \frac{1}{m \cdot \qYPrime}$; this choice is motivated from the importance weighting identity,
    as shall subsequently be explicated.
    
    \item
    \textbf{Constant}.~Here,
    $w_{yy'}$ is a constant for all 
    sampled negatives;
    typically, this is just $w_{yy'} \equiv \frac{1}{m}$.

    \item 
    \textbf{Relative weighting}.~Here,
    $w$
    is a variant of importance weighting,
    given by the density ratio 
    $w_{yy'} = {\qY}/{\qYPrime}$.
\end{itemize}

We summarize these common choices for $(q,w)$ in Table~\ref{tbl:summary-decoupled}. 
Particular combinations of $( q, w )$ yield common negative sampling schemes.
For example, the \emph{sampled softmax}~\citep{Bengio:2008,Blanc:2018}
employs a model-based sampling distribution $q$,
coupled with the importance weighting $w$.
Within-batch and uniformly sampled negatives are typically coupled with a constant weight~\citep{Hidasi:2016,Wang:2017b,Broscheit:2020}. \citet{YiRecSys19} employed within-batch sampling with relative weighting.
Appendix~\ref{app:additional-q-w} discusses further variants.

\omission{For example,
sampled softmax methods typically exclude the positive label $y$ from the sampling domain~\citep{Blanc:2018,Rawat:2019}.
However,
implementations in popular libraries\footnote{e.g., TensorFlow's \href{https://www.tensorflow.org/api_docs/python/tf/nn/sampled_softmax_loss}{\tt tf.nn.sampled\_softmax\_loss}.}
do 
\emph{not} explicitly exclude the positive label from sampling,
but rather set $w_{yy} = 0$,
and employ the relative weights $w_{yy'} = \frac{q_{y'}}{q_{y}}$.
\citet{Yi:2019} proposed to further not enforce that $w_{yy} = 0$.}

\subsection{Efficiency versus efficacy of negative sampling}

There are two primary considerations in assessing a negative sampling scheme.
The first is \emph{efficiency}:
to facilitate large-scale training,
it ought to be cheap to sample from the distribution $q$.
Uniform and within-batch negatives
are particularly favourable from this angle,
and have thus enjoyed wide practical use~\citep{Hidasi:2016,Yi:2019}. 

The other side of the coin is \emph{efficacy}:
to ensure good model performance,
the weights $w$ ostensibly ought to yield a good approximation to the original (unsampled) loss.
Sampled softmax schemes have been extensively studied through this lens;
surprisingly, however, 
less attention has been spent on uniform and within-batch sampling,
despite their popularity.

We now study the efficacy of generic sampling schemes, 
with the aim of understanding their {\em sampling bias} 
(i.e., deviation from the unsampled loss). 
In a nutshell, we explicate the \emph{implicit loss} that a 
given sampling scheme optimises (\S\ref{sec:sampling-bias}).
Beyond elucidating their sampling bias,
this gives a novel 
\emph{labeling bias} view of negative sampling,
and connects them to the long-tail literature (\S\ref{sec:background-long-tail}).
In particular,
the implicit losses bear strong similarity with those designed to ensure balanced performance for long-tail settings. We thus propose a unified approach to design negative sampling schemes that overcome {sampling} \emph{and} {labeling} bias~(\S\ref{sec:labelling-bias}).

\section{The implicit losses of negative sampling}
\label{sec:sampling-bias}
Given a negative sampling scheme parameterised by  $( q, w )$, 
we now ask: what loss does this scheme \emph{implicitly} optimise? 
Comparing this {\em implicit loss} with the unsampled loss (e.g., the softmax cross-entropy~\eqref{eqn:softmax-xent}) quantifies the {\em sampling bias},
which we shall see is non-trivial for popular schemes.

\begin{table*}[!t]
    \centering
    \renewcommand{\arraystretch}{1.5}
    
    \resizebox{0.99\linewidth}{!}{
    \begin{tabular}{@{}llll@{}}
        \toprule
        \textbf{Sampling distribution} & \textbf{Weighting $w_{yy'}$} & \textbf{\changesNew{Implicit loss}} & \textbf{Comment} \\
        \toprule
        Uniform
        &
        Constant ($\frac{1}{m}$)
        & 
        $\log\left[ 1 + \frac{1}{L} \sum_{y' \neq y}{ e^{{f}_{y'}( x ) - {f}_{y}( x )} } \right]$ 
         &
        Softmax with downweighted negatives  \\
        Uniform &  
        Importance 
        ($\frac{1}{m \cdot \qYPrime}$)
        & 
        $\log\left[ 1 + \sum_{y' \neq y}{ e^{{f}_{y'}( x ) - {f}_{y}( x )} } \right]$ 
         &
        Softmax cross-entropy \\
        Uniform &  
        Relative 
        ($\frac{q_{y}}{q_{y'}}$)
        & 
        $\log\left[ 1 + \frac{m}{L} \sum_{y' \neq y} e^{{f}_{y'}( x ) - {f}_{y}( x )} \right]$ 
         &
        Softmax with downweighted negatives \\    
        Uniform &  
        ``Tail'' 
        ($\frac{\pi_{y'}}{m \cdot q_{y'} \cdot \pi_{y}}$)
        & 
        $\log\left[ 1 + \sum_{y' \neq y} \frac{\pi_{y'}}{\pi_{y}} \cdot e^{{f}_{y'}( x ) - {f}_{y}( x )} \right]$ 
         &
        Logit-adjusted loss of~\citet{Menon:2020}${\color{RoyalBlue}^\tailSymb}$ \\        
        \midrule
        Within-batch
        &
        Constant ($\frac{1}{m}$)
        & 
        $\log\left[ 1 + \sum_{y' \neq y}{ \pi_{y'} \cdot e^{{f}_{y'}( x ) - {f}_{y}( x )} } \right]$ 
         &
        Equalised loss of~\citet{Tan:2020}${\color{RoyalBlue}^\tailSymb}$ \\
        Within-batch &  
        Importance 
        ($\frac{1}{m \cdot \qYPrime}$)
        & 
        $\log\left[ 1 + \sum_{y' \neq y}{ e^{{f}_{y'}( x ) - {f}_{y}( x )} } \right]$ 
         &
        Softmax cross-entropy \\
        Within-batch &  
        Relative 
        ($\frac{q_{y}}{q_{y'}}$)
        & 
        $\log\left[ 1 + {m \cdot \pi_{y}} \cdot \sum_{y' \neq y} e^{{f}_{y'}( x ) - {f}_{y}( x )} \right]$ 
         &
        Softmax with upweighted head labels${\color{ForestGreen}^\headSymb}$ \\  
        Within-batch &  
        ``Tail'' 
        ($\frac{\pi_{y'}}{m \cdot q_{y'} \cdot \pi_{y}}$)
        & 
        $\log\left[ 1 + \sum_{y' \neq y} \frac{\pi_{y'}}{\pi_{y}} \cdot e^{{f}_{y'}( x ) - {f}_{y}( x )} \right]$ 
         &
        Logit-adjusted loss of~\citet{Menon:2020}${\color{RoyalBlue}^\tailSymb}$ \\
        \bottomrule
    \end{tabular}
    }

    \caption{Implicit losses 
    for
    sampled softmax cross-entropy
    $\log\left[ 1 + \sum_{y' \in \NegLabels} w_{yy'} \cdot e^{f_{y'}( x ) - f_{y}( x )} \right]$ 
    on
    example $(x, y)$. Here, negatives $\NegLabels = \{ y'_1, \ldots, y'_m \}$ are sampled from $q$, 
    and weighted with $w_{yy'}$. Different $(q, w)$ 
    from Table~\ref{tbl:summary-decoupled},
    and a ``tail'' weighting from~\eqref{eqn:logit-weight},
    yield \changesNew{implicit losses} (cf.~Lemma~\ref{lemm:expected-sampled-loss}) which equate to
    pairwise margin losses~\eqref{eqn:unified-margin-loss} from the long-tail literature,
    e.g.,~\citet{Tan:2020,Menon:2020}. Losses denoted ${\color{ForestGreen}^\headSymb}$ (${\color{RoyalBlue}^\tailSymb}$) are expected to  benefit ``head'' (``tail'') classes,
    i.e.,
    yield better classification on data points with frequent (rare) labels. For the case of decoupled losses, see Table~\ref{tbl:sampling-summary-contrastive} (Appendix). 
    }
    \label{tbl:sampling-summary-v2}
    \vspace{-\baselineskip}
\end{table*}

\subsection{Implicit losses for generic negative samplers}
\label{sec:implicit-sampling}

Negative sampling schemes are characterised by the sampling distribution $q$ and weighting scheme $w$.
Each $(q, w)$ pair defines the sampled loss minimised during the training based on a random subset of negatives $\NCal$ (cf.~\eqref{eqn:weighted-sampled-softmax} and~\eqref{eqn:weighted-sampled-decoupled}). Such sampled losses are random variables,
owing to their dependence on $\NCal$. We may thus study their corresponding \emph{expected losses}, which we view
as the \emph{implicit loss} being optimised during training. The form of this loss helps us understand the effect of sampling on model performance. 

In the following, 
we assume for simplicity that $\qY = 0$,
so that the positive label $y$ is excluded from the negative sample. 
We first consider the setting of decoupled losses (cf.~\eqref{eqn:weighted-sampled-decoupled}),
and then the more challenging softmax cross-entropy.

\textbf{Decoupled losses}.
\label{sec:implicit-decoupled}
For a generic $(q, w)$, 
we now characterise the behavior of the sampled decoupled loss~\eqref{eqn:weighted-sampled-decoupled}.

\begin{lemma}
\label{lemm:contrastive-expectation}
Pick any $(q, w)$
where $q \in \Delta_{[L]}$ has $\qY = 0$. 
For the sampled decoupled loss $\ell$ in~\eqref{eqn:weighted-sampled-decoupled}
with $m \in \mathbb{Z}_+$ negatives, 
\begin{align}
\label{eq:implicit-decoupled}
    &\underset{\NegLabels \sim q^m}{\mathbb{E}}\left[\ell(y, f(x); \NegLabels, w) \right] = \underbrace{\phi( f_y(x) ) + 
    \sum_{y' \neq y} \rho_{y y'} \cdot \varphi( -f_{y'}( x ) )}_{:= \ell^{q, w}_{m}(y, f(x))}  \\
    &\underset{\NegLabels \sim q^m}{\mathbb{V}}\left[ \ell(y, f(x); \NegLabels, w) \right] = \sum_{y' \neq y} w_{yy'} \cdot \rho_{y y'} \cdot \varphi( -f_{y'}( x ) )^2  - \frac{1}{m} \cdot \Big(\sum_{y' \neq y} \rho_{y y'} \cdot \varphi( -f_{y'}( x ) ) \Big)^2, \label{eq:implicit-decoupled-var}
\end{align}
where 
$\rho_{y y'} \defEq m \cdot w_{yy'} \cdot \qYPrime$.
\end{lemma}
Here, one can view $\ell^{q, w}_{m}(y, f(x))$ as the implicit loss being optimized during the training.
Compared to the unsampled loss $\ell$ in~\eqref{eqn:decoupled},
we see that for general 
$( q, w )$, $\ell^{q, w}_{m} \neq \ell$
owing to the former involving non-constant weights $\rho_{yy'}$ on each negative.
Thus,
negative sampling introduces a {\em sampling bias} in general. Eliminating this bias 
requires that $\rho_{y y'} = 1$,
which
for a given sampling distribution $q$ requires that 
$w_{yy'} = \frac{1}{m \cdot \qYPrime}$,
i.e.,
the importance weighting $w$ in Table~\ref{tbl:summary-decoupled}.

\textbf{Softmax cross-entropy loss}.
\label{sec:implicit-xent}
For the softmax cross-entropy, a closed-form for the implicit loss is elusive owing to the non-linear $\log$ in \eqref{eqn:weighted-sampled-softmax}. 
Nonetheless, 
to draw qualitative conclusions about different negative sampling schemes,
in this setting
we treat the following natural upper bound on the expected loss $\mathbb{E}_{\NegLabels}\left[\ell(y, f(x); \NegLabels, w) \right]$ as the implicit loss.

\begin{lemma}
\label{lemm:expected-sampled-loss}
Pick any $(q, w)$
where $q \in \Delta_{[L]}$ has $\qY = 0$.
For the sampled softmax cross-entropy $\ell$ in~\eqref{eqn:weighted-sampled-softmax},
\begin{align}
    \label{eqn:expected-sampled-loss-special}
    &\E{\NegLabels \sim q^m}{ \ell( y, f( x ); \NegLabels, w ) } \leq  \underbrace{\log\Big[ 1 + \sum_{y' \neq y} \rho_{yy'} \cdot e^{f_{y'}( x ) - f_{y}( x )} \Big]}_{:=\ell^{q, w}_m(y, f(x))},
\end{align}
where $\rho_{yy'} \defEq {m \cdot q_{y'} \cdot w_{yy'}}$ for $m \in \mathbb{Z}_+$ negatives.
\end{lemma}

Similar to the decoupled case,
the implicit loss $\ell^{q, w}_m$ (i.e., the upper bound in Lemma~\ref{lemm:expected-sampled-loss})
diverges from the softmax cross-entropy in general;
thus, there is sampling bias.
The two 
coincide for the 
importance weights
$w_{yy'} = \frac{1}{m \cdot \qYPrime}$.
{In fact, when $q_{y'} \propto e^{f_{y'}}$, 
this choice of weight ensures \emph{equality} in \eqref{eqn:expected-sampled-loss-special}. 
For general $q$, 
it further
guarantees asymptotic unbiased estimate of the \emph{gradient} of the softmax cross-entropy as $m\to \infty$~\cite{Bengio:2008}. 
See also~\citet{Rawat:2019} for non-asymptotic bounds on the gradient bias.}

Before proceeding,
it is natural to ask if the upper bound $\ell_m^{q, w}$ in  \eqref{eqn:expected-sampled-loss-special} 
is tight enough to base any further analysis upon. 
The following result establishes that 
for large $m$, $\ell_m^{q, w}$
indeed approximates the true behavior of the sampled softmax cross-entropy loss in \eqref{eqn:weighted-sampled-softmax}
in a squared-error sense. 

\begin{theorem}
\label{thm:implicit-xent}
Pick any $(q, w)$ 
such that
$q \in \Delta_{[L]}$
has 
$q_y = 0$,
and $w_{yy'}$ is expressible as 
$\frac{1}{m} \cdot \eta_{yy'}$ for $\eta_{yy'}$ independent of $m$. 
Let 
$\ell$ be the softmax cross-entropy~\eqref{eqn:softmax-xent}, 
$\ell_m^{q, w}$ be as per Lemma~\ref{lemm:expected-sampled-loss}, $\mu_y(x) \defEq e^{f_y(x)}+  \mathbb{E}_{y' \sim q}[ \eta_{yy'} \cdot e^{f_{y'}(x)} ]$,
and $\sigma^2_y(x) \defEq \mathbb{V}_{y' \sim q}\left[\eta_{yy'} \cdot e^{f_{y'}(x)}\right] $. Assume that $\mu_y(x), \sigma_y^2(x) \in ( 0, +\infty )$.
Then, for large enough $m$, 
\begin{align}
&\mathbb{E}_{\NegLabels}\big[\big( \ell(y, f(x); \NegLabels, w)  - \ell^{q, w}_m(y, f(x)) \big)^2 \big]  = \frac{\sigma^2_y(x)}{m \cdot \mu_{y}^2(x)} + E_m,
\end{align}
where $E_m = o_p( 1 )$ converges to $0$ in probability.
\end{theorem}

In the above,
$\mu_{y}(x)$
equals the standard
partition function
$\sum_{y'}e^{f_{y'}(x)}$
under the importance weighting $w_{yy'}=\frac{1}{m \cdot q_{y'}}$.
The above guarantees 
convergence of the sampled loss 
to
$\ell_m^{q, w}$,
the upper bound
in~\eqref{eqn:expected-sampled-loss-special},
\emph{provided} $w_{yy'}$ takes the form ${\eta_{yy'}} / {m}$;
equally,
in the language of Lemma~\ref{lemm:expected-sampled-loss},
this requires that \updates{${\rho_{yy'}}/{q_{y'}}$ is a constant}{} independent of $m$.
The \emph{rate} of convergence
is governed by
$\sigma_y^2( x ) / \mu_y^2( x )$,
which can be seen as an inverse signal-to-noise ratio.
Intuitively, 
when the losses 
$e^{f_{y'}( x )}$
on negative labels $y' \neq y$ are highly concentrated (i.e., $\sigma_y^2( x )$ is small relative to $\mu_{y}( x )$),
sampling even a few negatives can reliably estimate the true loss.

\subsection{Implicit losses of existing negative samplers}
\label{sec:sampling-bias-popular-methods}

Equipped with our understanding of the implicit losses for generic negative sampling methods, 
we now focus on specific methods that are prevalent in practice,
namely,
uniform and within-batch sampling distributions $q$
with different weights $w$. Table~\ref{tbl:sampling-summary-v2} explicates the impact of specific choices of $(q, w)$ \updates{through their implicit losses}{}.
For example,
under uniform sampling, a constant weight of $\frac{1}{m}$ implicitly down-weights negatives in the standard loss by a factor of $\frac{1}{L}$. Indeed, except for the case of 
$w_{yy'} = \frac{1}{m \cdot q_{y'}}$,
the implicit sampling losses subtly differ from the softmax cross-entropy: in such losses,
the ``negative'' labels receive a variable penalty.
Thus, in general, such schemes result in a sampling bias.

\subsection{Discussion and implications}

We reiterate that
while 
negative sampling
schemes such as uniform and within-batch sampling (with constant weights) are popular,
they have primarily been motivated as a cheap means of approximating the original (unsampled) loss.
Quantifying \emph{how} well they approximate this loss, however, has not (to our knowledge) been previously explicated.

Having specified such schemes' implicit losses above, we make two remarks.
First, 
per~\S\ref{sec:sampling-bias-popular-methods} and Table~\ref{tbl:sampling-summary-v2},
these schemes induce a sampling bias,
and do \emph{not} faithfully approximate the original loss. This appears to suggest that,
despite their popularity,
such schemes' efficacy is limited.

However,
our second observation
hints at a different conclusion:
note that the implicit loss in~\eqref{eqn:expected-sampled-loss-special} 
bears a striking similarity to 
the softmax with pairwise label margin~\eqref{eqn:unified-margin-loss}
from the long-tail learning literature.
This raises an intriguing,
hitherto unexplored
question:
could such negative sampling schemes be effective at modelling \emph{tail} labels?

\section{Negative sampling meets long-tail learning}
\label{sec:labelling-bias}
Leveraging the implicit loss derivation above, 
we now show that 
negative sampling schemes 
can mitigate \emph{labeling} bias by
implicitly trading-off performance on dominant versus rare classes.
Subsequently, we present a unified view of negative sampling that addresses \emph{both} {labeling} and \emph{sampling} bias.

\subsection{Implicit sampling losses: head versus tail tradeoffs}
\label{sec:long-tail-view}

In~\S\ref{sec:sampling-bias-popular-methods}, we have seen how 
from a \emph{sampling bias} perspective,
to obtain an unbiased estimate of the original loss,
the only admissible choice of $(q, w)$ is $w_{yy'} = \frac{1}{m \cdot q_{y'}}$. Interestingly,
from a \emph{labeling bias} perspective,
a different picture emerges:
other choices of $( q, w )$
may \emph{implicitly} trade off performance on rare (``tail'') versus dominant (``head'') classes,
which can be desirable.
We illustrate this point via two examples.

\textbf{Example: within-batch sampling and constant weights}.
Consider within-batch sampling, so that $q = \pi$ for training label distribution $\pi$.
Coupled with
a constant
$w_{yy'} = \frac{1}{m}$,
the {sampled softmax cross-entropy~\eqref{eqn:expected-sampled-loss-special} has implicit loss:}
\begin{equation}
    \label{eq:implicit-decoupled-within}
    \ell^{q, w}_{m}(y, f(x)) =  \log\Big[ 1 + \sum_{y' \neq y} \pi_{y'} \cdot e^{f_{y'}( x ) - f_{y}( x )} \Big].
\end{equation}
Clearly, this disagrees with the standard softmax cross-entropy~\eqref{eqn:softmax-xent}.
However,
the loss down-weights the contribution of rare labels when they are encountered as negatives.
Since rare labels infrequently appear as positives during training,
such down-weighting prevents them from being overwhelmed by negative gradients.
The implicit loss (and thus the sampling scheme) is therefore expected to boost classification performance on samples with rare labels.

\textbf{Example: within-batch sampling and relative weights}.
Suppose again that $q = \pi$, but that we instead use the relative weighting scheme $w_{yy'} = \frac{\pi_{y}}{\pi_{y'}}$.
Here,
the
implicit loss for the sampled softmax cross-entropy is
\begin{align}
    \label{eq:implicit-decoupled-within-relative}
    \ell^{q, w}_m(y, f(x)) = \log\Big[ 1 + \sum_{y' \neq y} \pi_{y} \cdot e^{f_{y'}( x ) - f_{y}( x )} \Big].
\end{align}
Compared to~\eqref{eq:implicit-decoupled-within},
this has a subtly different negative weighting of $\pi_y$, rather than $\pi_{y'}$.
This places greater emphasis on \emph{positive} labels that occur more frequently in the training set,
and thus has an \emph{opposite} effect to~\eqref{eq:implicit-decoupled-within}:
{we encourage large model scores for \emph{dominant} compared to rare labels.}

\subsection{Relating long-tail and implicit sampling losses}
\label{sec:unified}

To generalise the above,
recall that in the \emph{long-tail learning}
problem (\S\ref{sec:background-long-tail}), 
several popular losses to cope with label imbalance
are special cases of the softmax cross-entropy with pairwise margin~\eqref{eqn:unified-margin-loss}, i.e.,
for a given $\rho_{yy'} \geq 0$,
$$ \ell^{\rho}_{\rm mar}( y, f( x ) ) = \log\Big[ 1 + \sum_{y' \neq y} \rho_{y y'} \cdot e^{f_{y'}( x ) - f_{y}( x )} \Big]. $$
Intriguingly,
both the implicit sampling losses 
from~\S\ref{sec:long-tail-view}
are special cases of this loss;
in particular, \eqref{eq:implicit-decoupled-within} is a special case of the equalised loss of~\citet{Tan:2020}. Thus,
such negative sampling schemes \emph{implicitly} cope with label imbalance.
This connection
is no accident:
the following establishes a simple equivalence between the two loss families.

\begin{proposition}
\label{prop:loss-connection}
Pick any $\rho_{y y'} \geq 0$,
and let 
$\ell^{\rho}_{\mathrm{mar}}$
be the softmax cross-entropy with pairwise margin $\rho$~\eqref{eqn:unified-margin-loss}.
Given
any
$q \in \Delta_{[L]}$ with $q_y = 0$ and $m \in \mathbb{Z}_+$,
suppose we set
\begin{equation}
    \label{eqn:weight-constraint}
    w_{yy'} = \frac{\rho_{y y'}}{m \cdot \qYPrime}.
\end{equation}
Then, the implicit $\ell^{q, w}_{m}$ from~\eqref{eqn:expected-sampled-loss-special} satisfies:
$$ ( \forall x \in \XCal, z \in \Real^L ) \, \ell^{q, w}_{m}( y, z ) = \ell^{\rho}_{\mathrm{mar}}( y, z ). $$
\end{proposition}

The above illustrates an intimate connection between
negative sampling schemes,
and losses for long-tail learning.
Remarkably,
while the former is ostensibly a means of coping with sampling bias,
it can \emph{implicitly} address the labeling bias problem arising from a skewed distribution of classes.

\subsection{Discussion and implications}

For generic $( q, w )$, 
we may {understand the performance impact of negative sampling by identifying the 
implicit losses (Lemma~\ref{lemm:contrastive-expectation} and~\ref{lemm:expected-sampled-loss})
optimised during training}.
Both \S\ref{sec:long-tail-view} and \S\ref{sec:unified} demonstrate the dual utility of these implicit losses. 
{We comment on these and further implications of our results.}

\textbf{Overcoming sampling and labeling bias}.
From~\S\ref{sec:long-tail-view},
common negative sampling schemes \emph{implicitly}
trade-off performance on dominant versus rare classes. 
Boosting performance on samples with rare labels
is an important goal~\citep{VanHorn:2017,Buda:2017},
and relates to ensuring \emph{fairness}~\citep{Sagawa:2020}.

Proposition~\ref{prop:loss-connection} also gives a way to \emph{explicitly} control both sampling \emph{and} labelling bias:
i.e.,
operate on a subset of labels 
(for computational efficiency),
but also 
ensure good performance on rare labels
(to ensure ``fairness'' across classes).
Given a target set of label margins $\rho_{yy'}$ (as in~\eqref{eqn:unified-margin-loss})
---
which can suitably
balance dominant versus rare class performance
---
one may use~\eqref{eqn:weight-constraint}
to pick a suitable combination of $(q, w)$ to achieve this balance. For example,
suppose we wish to approximate the logit-adjustment loss of~\citet{Menon:2020},
where $\rho_{y y'} = \frac{\pi_{y'}}{\pi_{y}}$.
Then, 
for a given $q$, we may set
\begin{equation}
    \label{eqn:logit-weight}
    w_{yy'} = \frac{\pi_{y'}}{m \cdot q_{y'} \cdot \pi_{y}}.
\end{equation}

\textbf{Within-batch sampling can boost the tail}.
Within-batch sampling has the virtue of simplicity:
it creates negatives from each sampled minibatch, which involves minimal additional overhead. Our above analysis
further shows that 
with a constant weighting,
within-batch sampling can implicitly boost performance on rare classes.
Thus, while this method may not be competitive in terms of standard \emph{class-weighted} error,
we expect it to have good \emph{balanced} error (where dominant and rare labels are treated equally).

\textbf{Bias-variance trade-off}.
The above apparently settles the issue of what a good choice of $(q, w)$ is:
for a target $\rho_{y y'}$,
$(q, w)$ should satisfy~\eqref{eqn:weight-constraint}. However, 
how can one choose amongst the infinitude of such pairs? One desideratum is to minimise the \emph{variance} of the estimate~\eqref{eq:implicit-decoupled-var},
per the following.

\begin{proposition}
\label{lemm:variance-minimiser}
Pick any $\rho_{y y'} \geq 0$,
and
let $w_{yy'}$ satisfy~\eqref{eqn:weight-constraint}. The sampling distribution $q^*$ with minimum loss variance is
\begin{equation}
    \label{eqn:q-ast-loss}
    ( \forall x \in \XCal ) \, q^*( y' \mid x, y ) \propto \rho_{y y'} \cdot \varphi( -f_{y'}( x ) ).
\end{equation}
\end{proposition}
According to~\eqref{eqn:q-ast-loss}, to minimise the variance of the loss estimate, one would like to sample the negative that contribute the most to the original decoupled loss. 
However, this would require evaluating the
loss on all classes, which defeats the purpose of sampling the negatives in the first place. Investigating efficient approximations 
(e.g., per~\citet{Bamler:2020})
is an interesting direction for future work.

\section{Experiments}
\label{sec:experiments}
We now present experiments 
on 
benchmarks for both 
long-tail learning and retrieval,
illustrating our main finding: existing negative sampling schemes, such as within-batch sampling with constant weighting,
implicitly trade-off performance on dominant versus rare labels. Further,
we show the viability of
designing custom sampling schemes to specifically target performance on rare labels.

\begin{figure*}[!t]
    \centering
    
    \resizebox{\linewidth}{!}{
    \subcaptionbox{CIFAR-100-LT ({\sc Step}).}{
    \includegraphics[scale=0.20,valign=t]{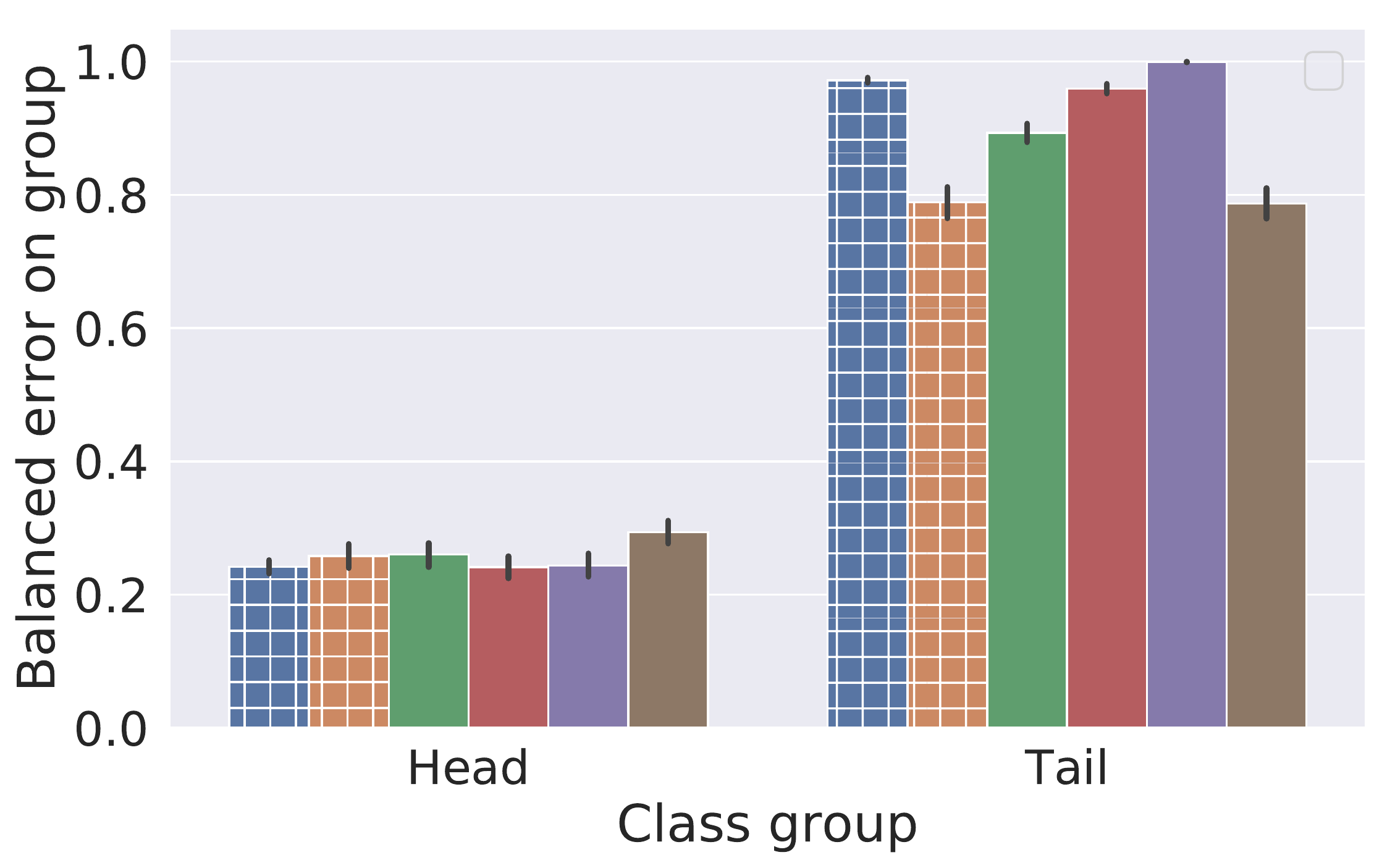}
    }    
    \qquad
    \subcaptionbox{CIFAR-100-LT ({\sc Exp}).}{
    \includegraphics[scale=0.20,valign=t]{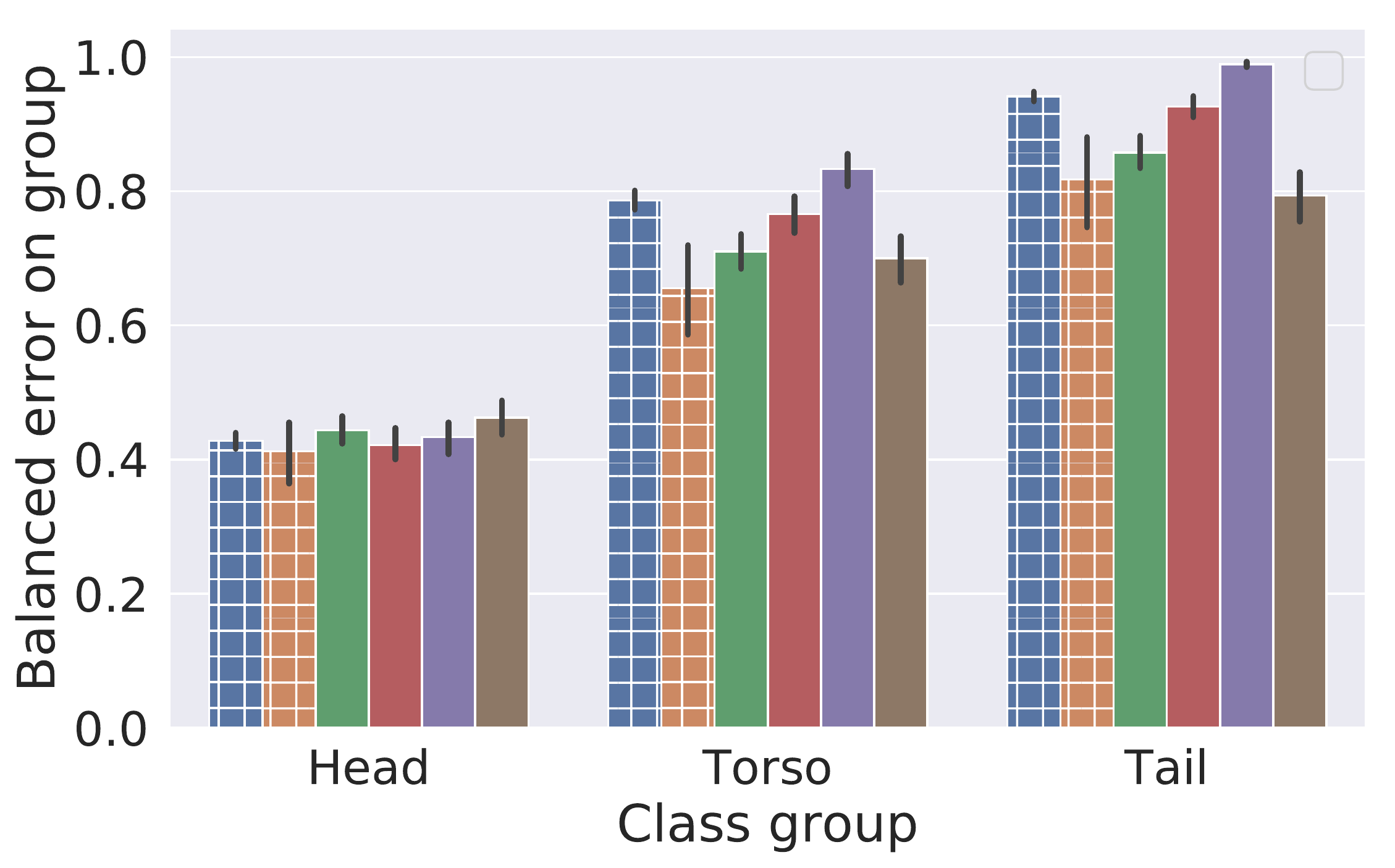}
    }
    \qquad
    \subcaptionbox{ImageNet-LT.}{
    \includegraphics[scale=0.20,valign=t]{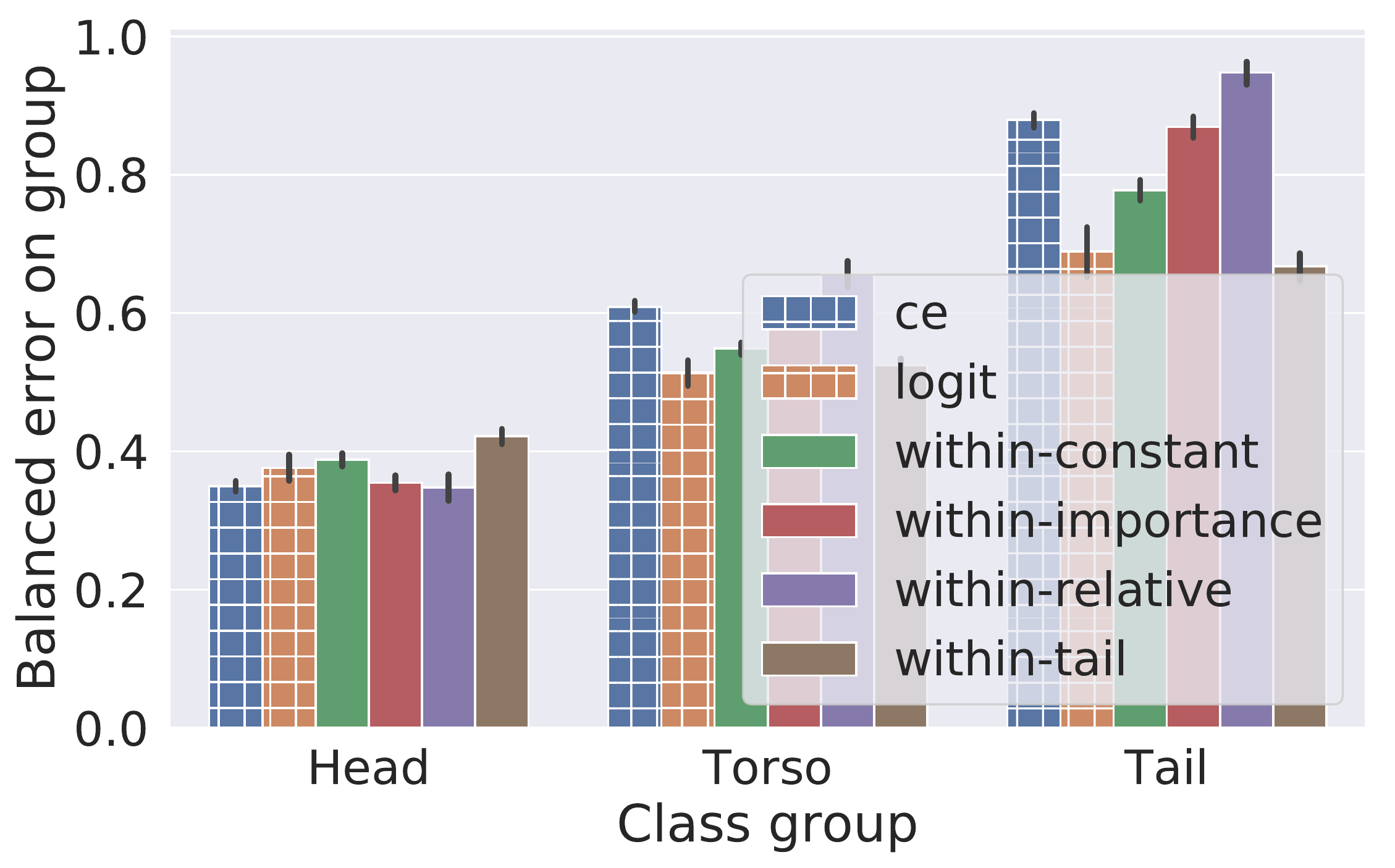}
    }
    }
    
    \resizebox{\linewidth}{!}{
    \subcaptionbox{CIFAR-100-LT ({\sc Step}).}{
    \includegraphics[scale=0.20,valign=t]{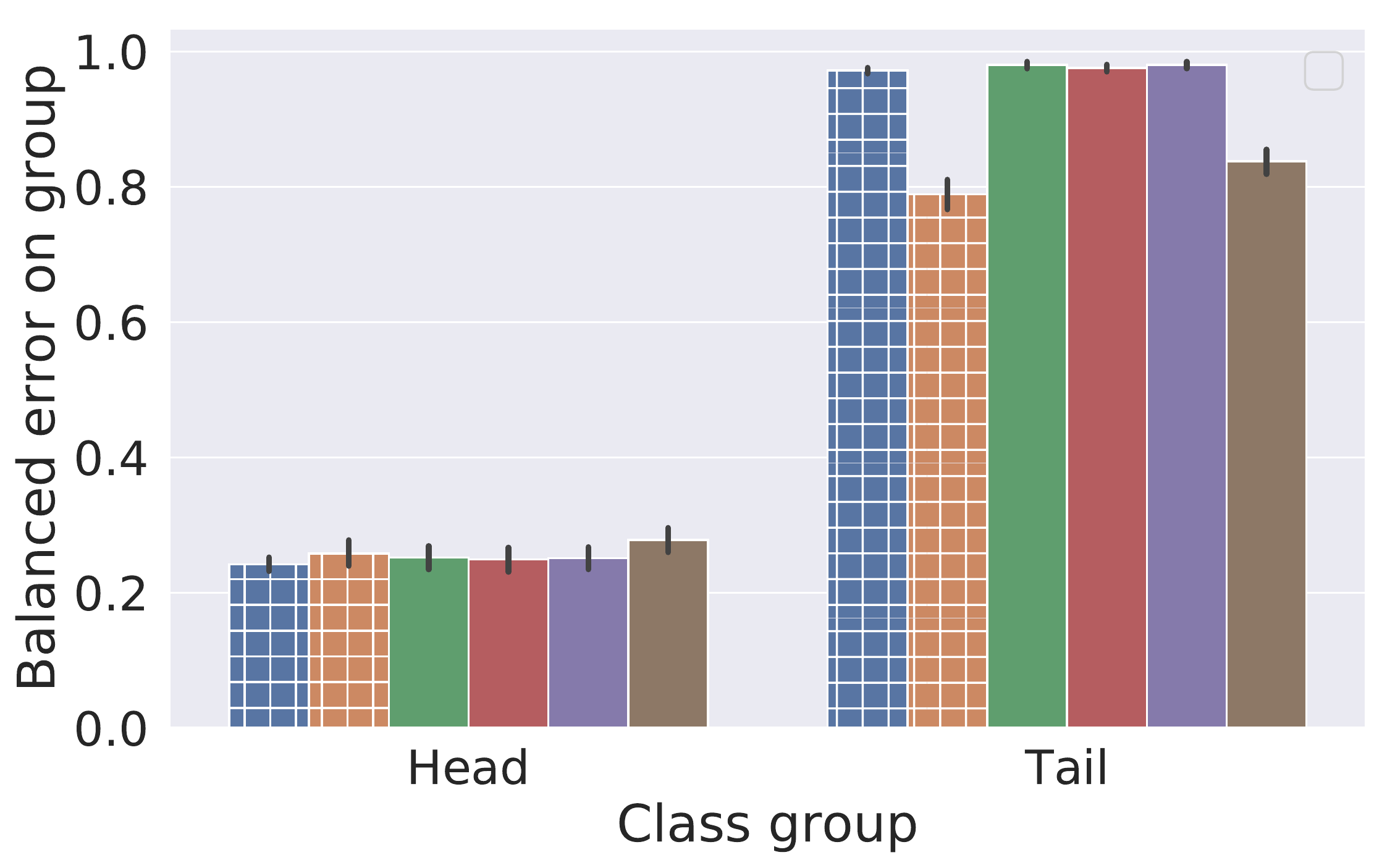}
    }
    \qquad
    \subcaptionbox{CIFAR-100-LT ({\sc Exp}).}{
    \includegraphics[scale=0.20,valign=t]{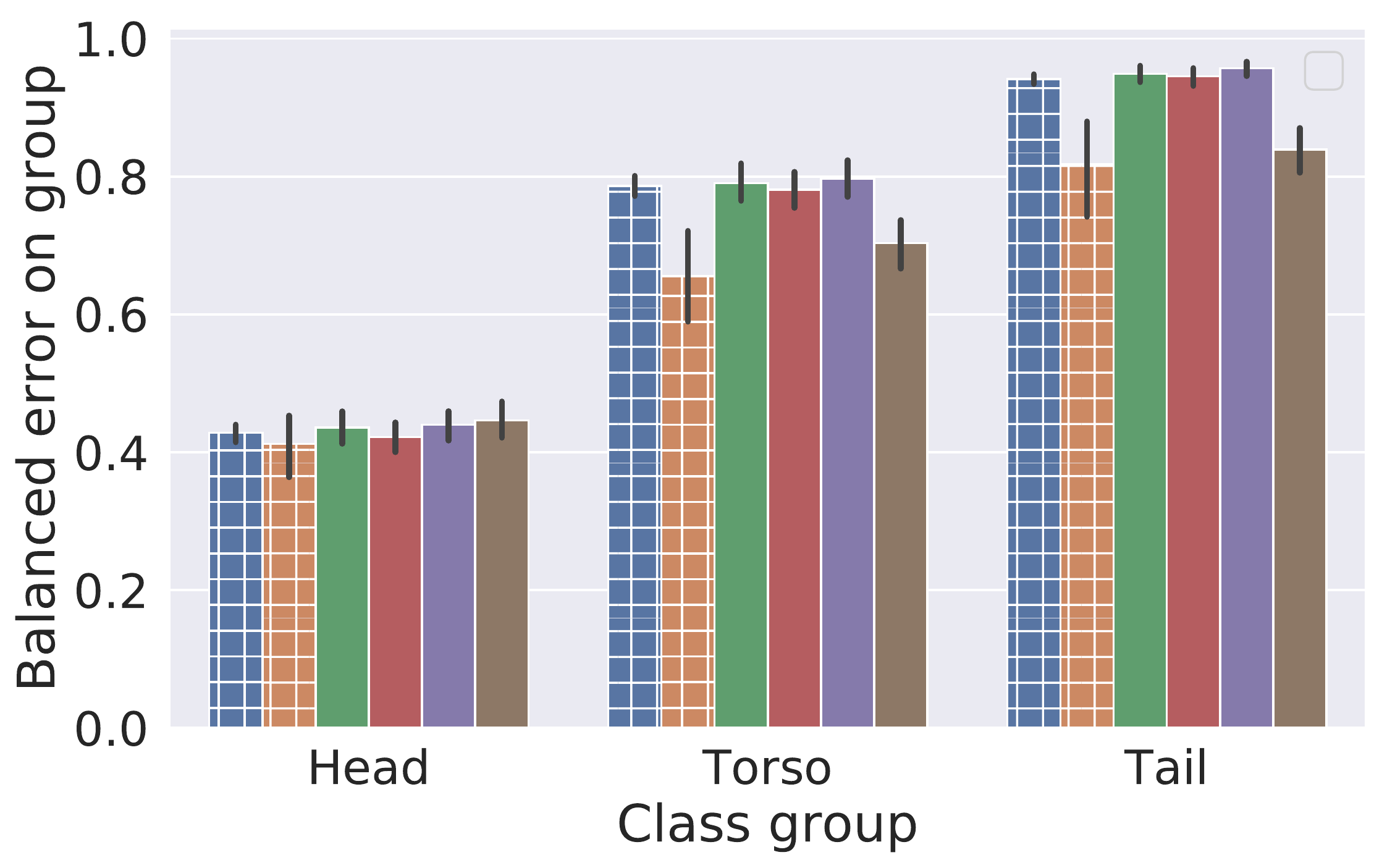}
    }
    \qquad
    \subcaptionbox{ImageNet-LT.}{
    \includegraphics[scale=0.20,valign=t]{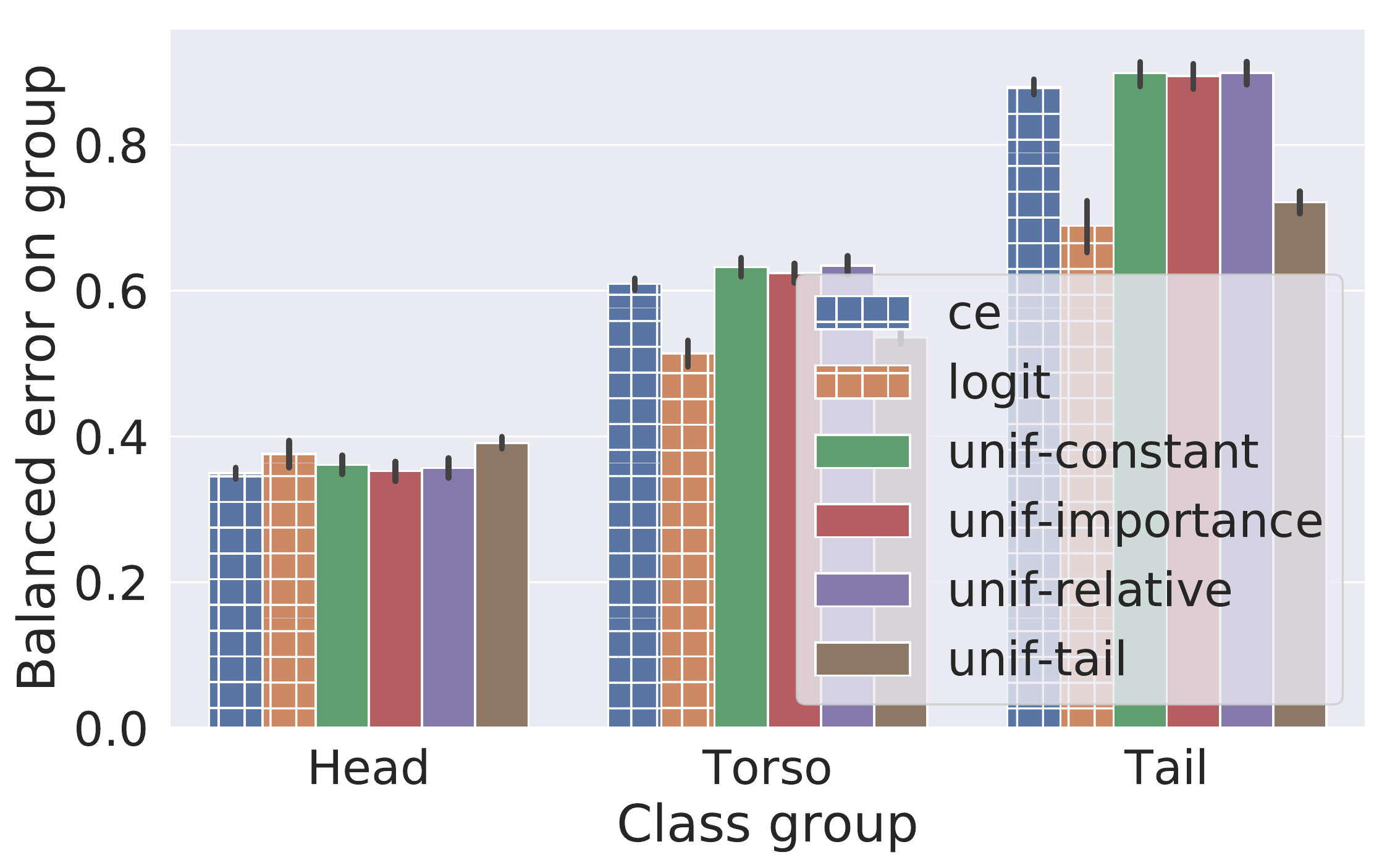}
    }
    }    
    
    \caption{Balanced error on Head, Torso and Tail labels on long-tail learning benchmarks;
    see~\S\ref{sec:long-tail-results} for definitions of the {\sc Step} and {\sc Exp} profiles.
    We employ within-batch (${\tt within}$)
    and uniform (${\tt unif}$) negative sampling,
    using the constant, 
    importance, 
    relative,
    and ``tail'' weighting
    schemes from Table~\ref{tbl:summary-decoupled}.
    We also include the results of standard softmax cross-entropy (${\tt ce}$) and logit adjustment (${\tt logit}$),
    which use \emph{all} labels, and are thus shaded. Amongst the standard schemes from Table~\ref{tbl:summary-decoupled},
using within-batch negatives with a constant weighting
performs the best on Tail classes. This is despite such a scheme providing a biased estimate of the cross-entropy;
indeed, on CIFAR-100-LT,
sampling does \emph{better} than the cross-entropy on tail classes,
despite the latter employing all labels during training.
Our newly proposed ``tail'' weighting scheme~\eqref{eqn:logit-weight} consistently outperforms all schemes on the Tail, confirming the viability of handling sampling \emph{and} labeling bias.
}
    \label{fig:lt_benchmarks}
\end{figure*}

\newcommand{\figwithcap}[2]{
  \begin{tikzpicture}[scale=0.4, inner sep=0]
        \node[anchor=south west,inner sep=0] (fig) at (0,0) {
		\includegraphics[scale=0.30,valign=t]{#1}%
	    };     
	    \node[anchor=north,inner sep=0,above left=-4mm and -38mm of fig] {\textcolor{black!70}{{\footnotesize #2}}};     
    \end{tikzpicture}
}

\begin{figure*}[t!]
    \centering
    \small
    \resizebox{0.9\linewidth}{!}{
    \subcaptionbox{$\recall$1.\label{fig:am-r1}}{
    \figwithcap{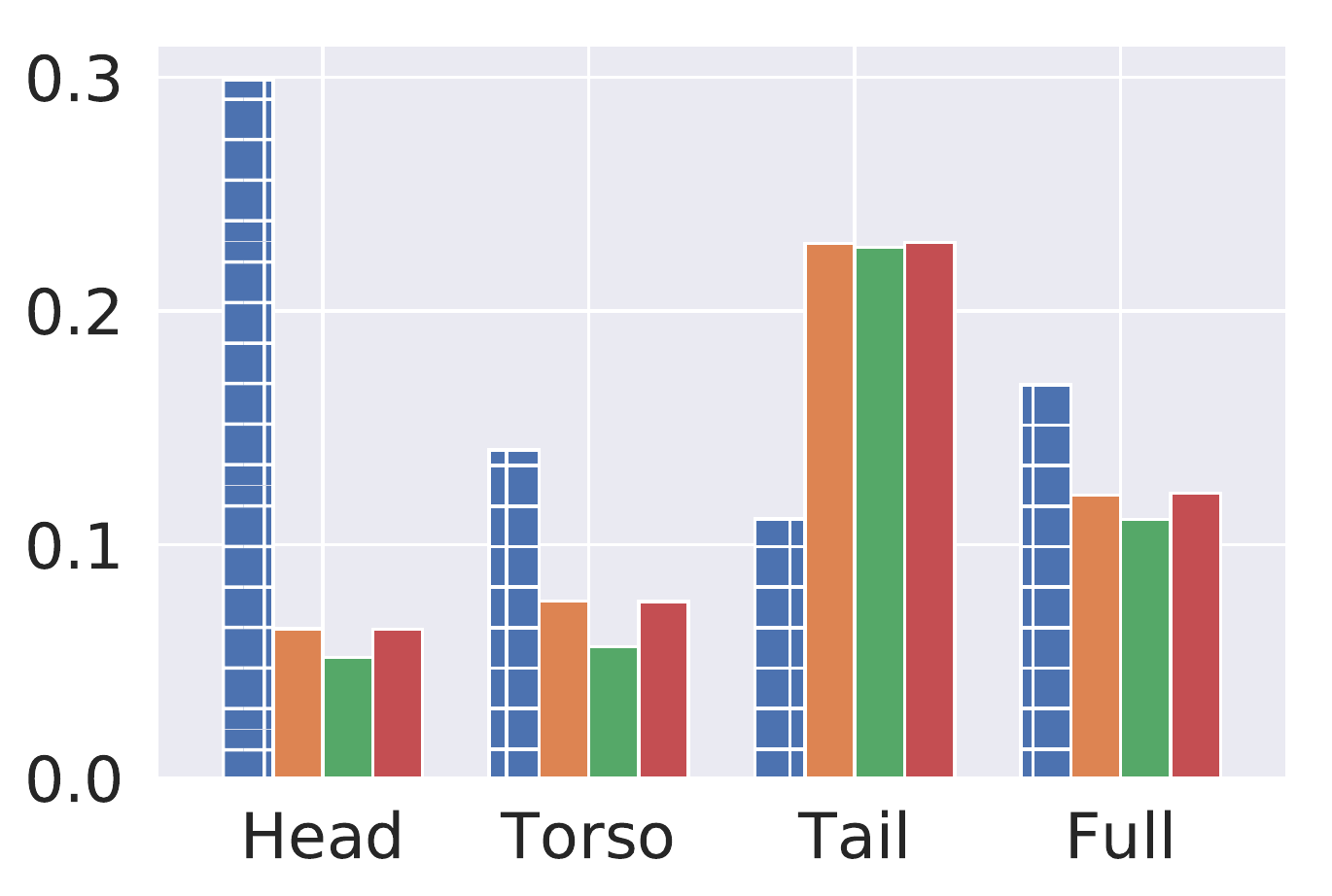}{\amazonsmall}
    }
    \;
    \subcaptionbox{$\recall$10.\label{fig:am-r10}}{
    \figwithcap{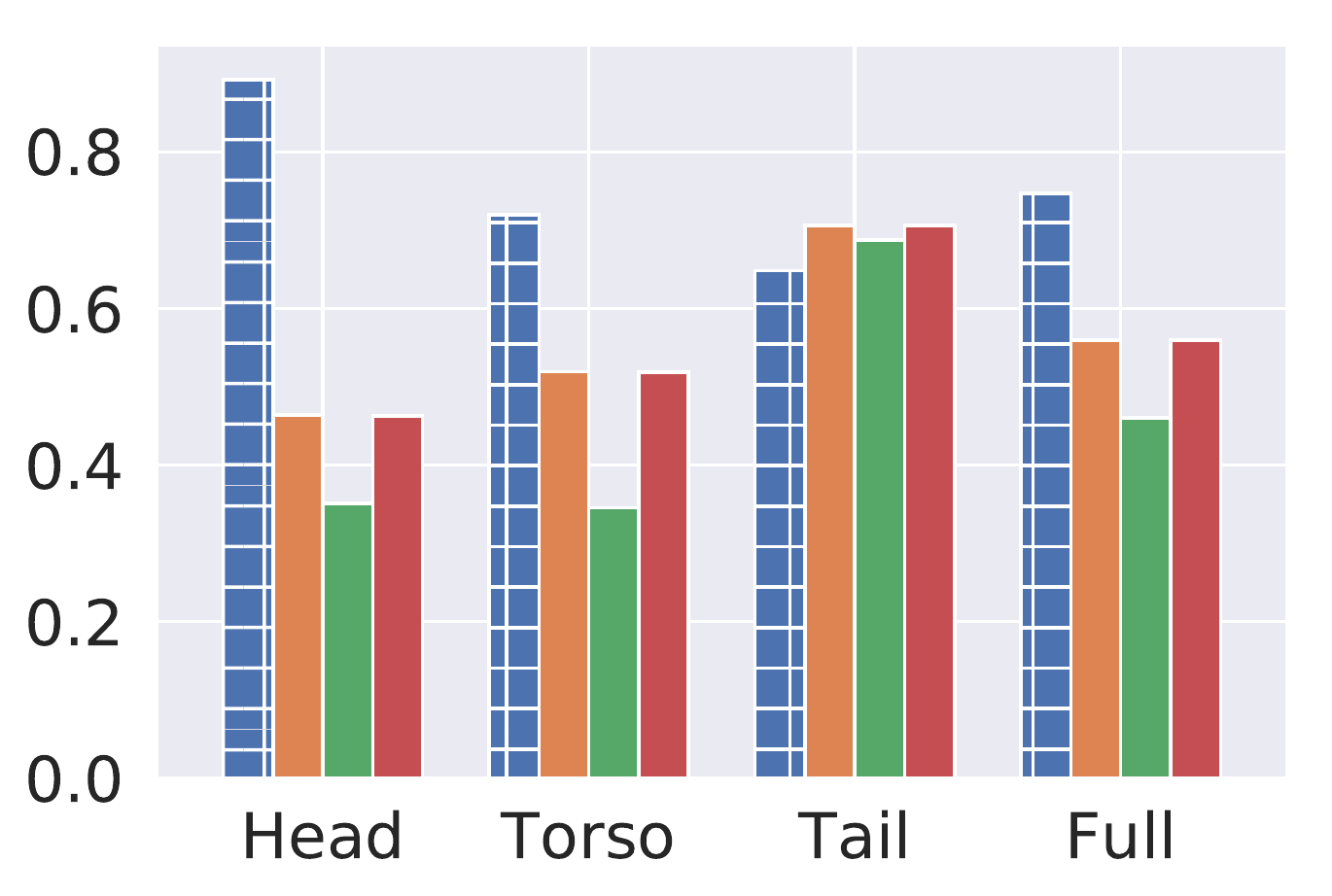}{\amazonsmall}
    }
    \;
    \subcaptionbox{$\recall$50.\label{fig:am-r50}}{
    \figwithcap{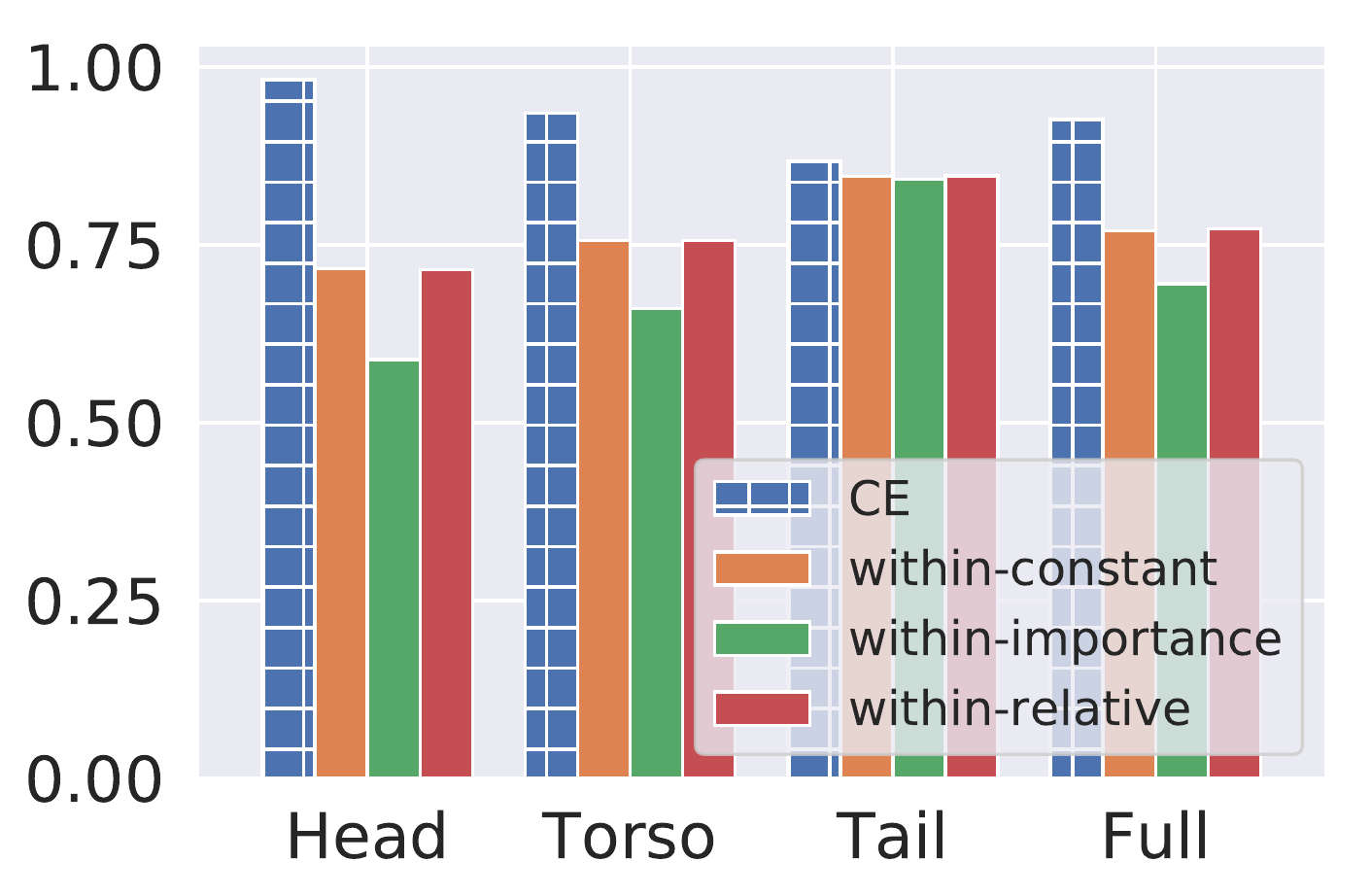}{\amazonsmall}
    }
    }

    \resizebox{0.9\linewidth}{!}{
    \subcaptionbox{$\recall$1.\label{fig:wiki-r1}}{
    \figwithcap{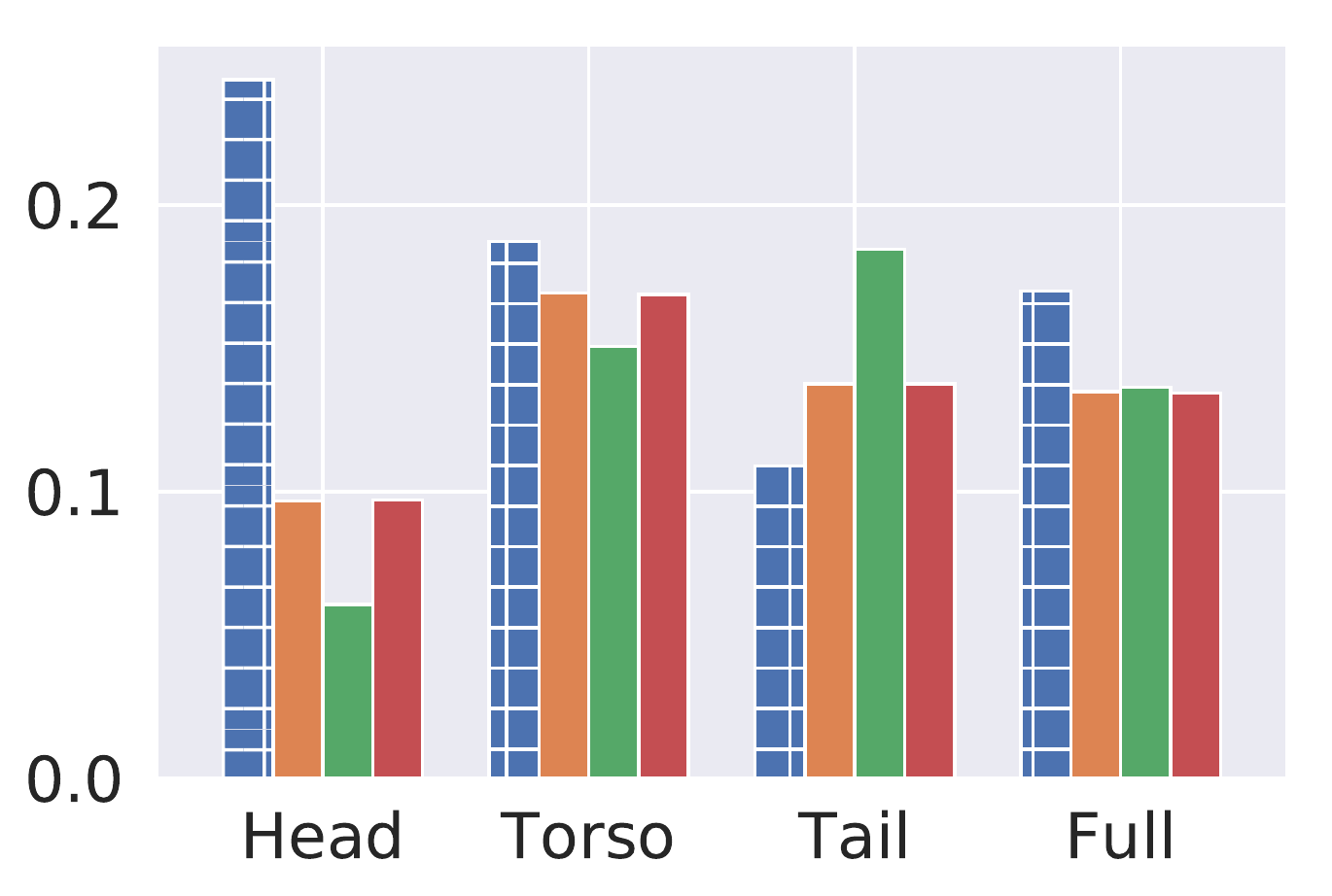}{\wikilshtc}
    }    
    \;
    \subcaptionbox{$\recall$10.\label{fig:wiki-r10}}{
    \figwithcap{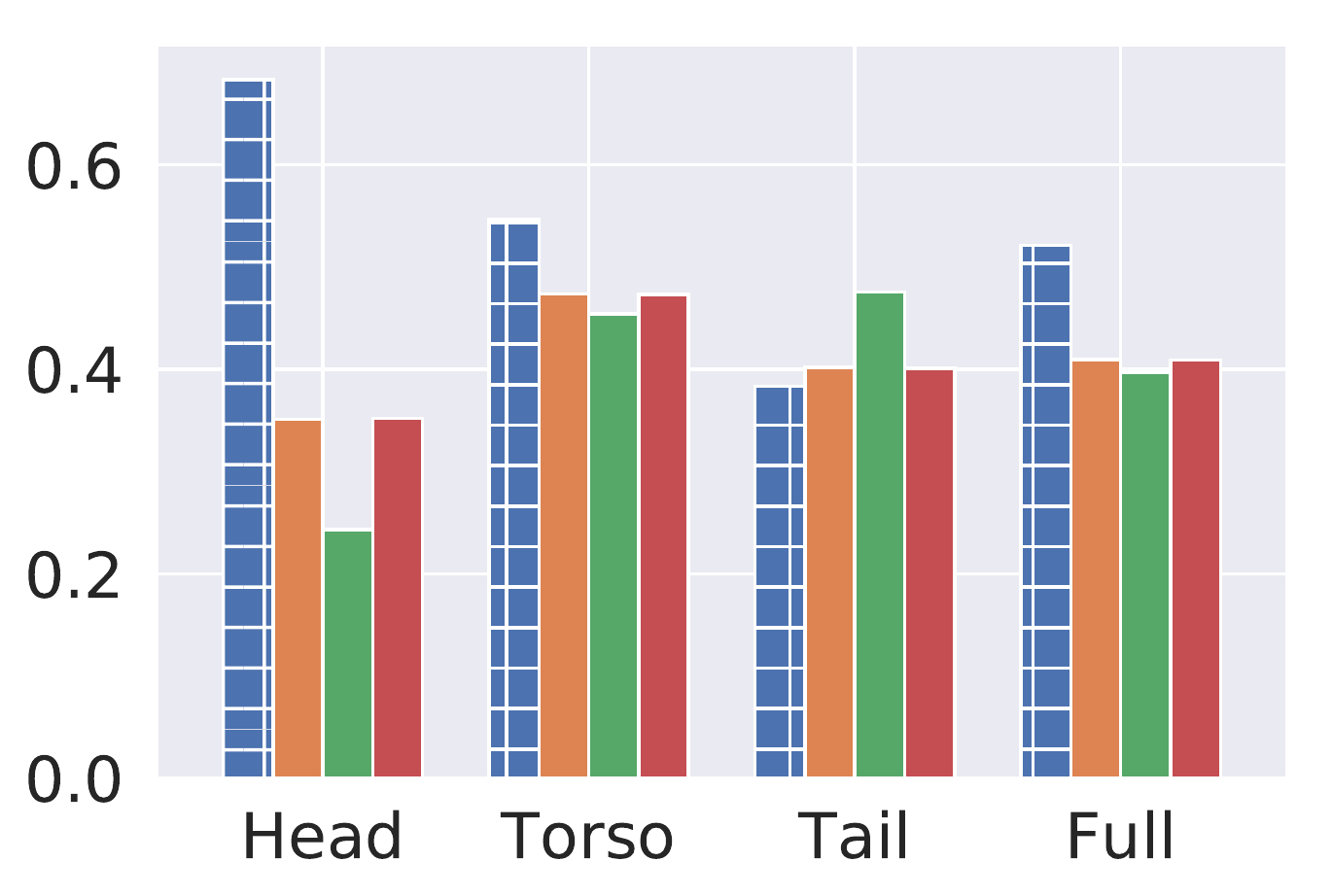}{\wikilshtc}
    }
    \;
    \subcaptionbox{$\recall$50.\label{fig:wiki-r50}}{
    \figwithcap{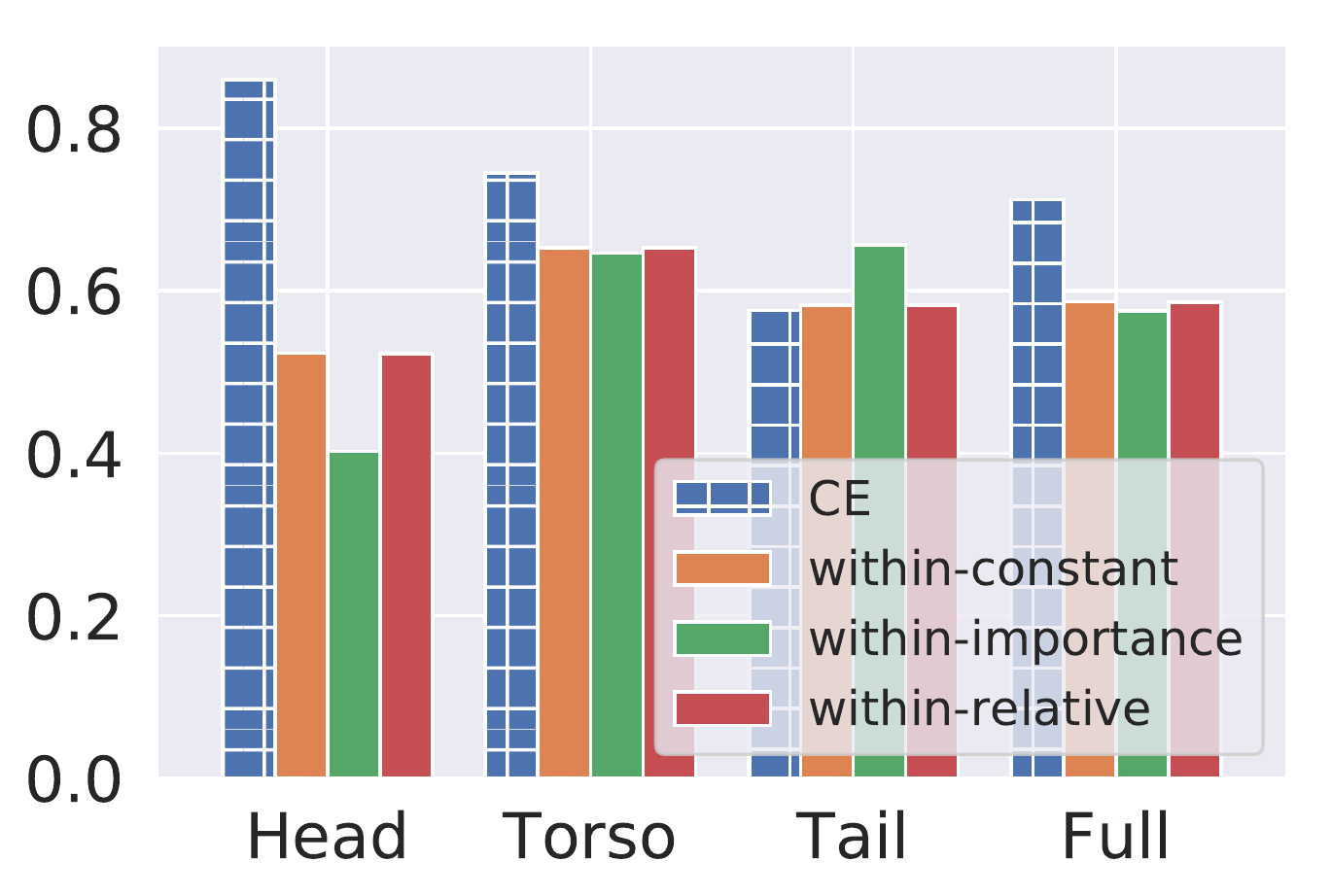}{\wikilshtc}
    }
    }
    \resizebox{0.9\linewidth}{!}{
    \subcaptionbox{$\recall$1.\label{fig:deli-r1}}{
    \figwithcap{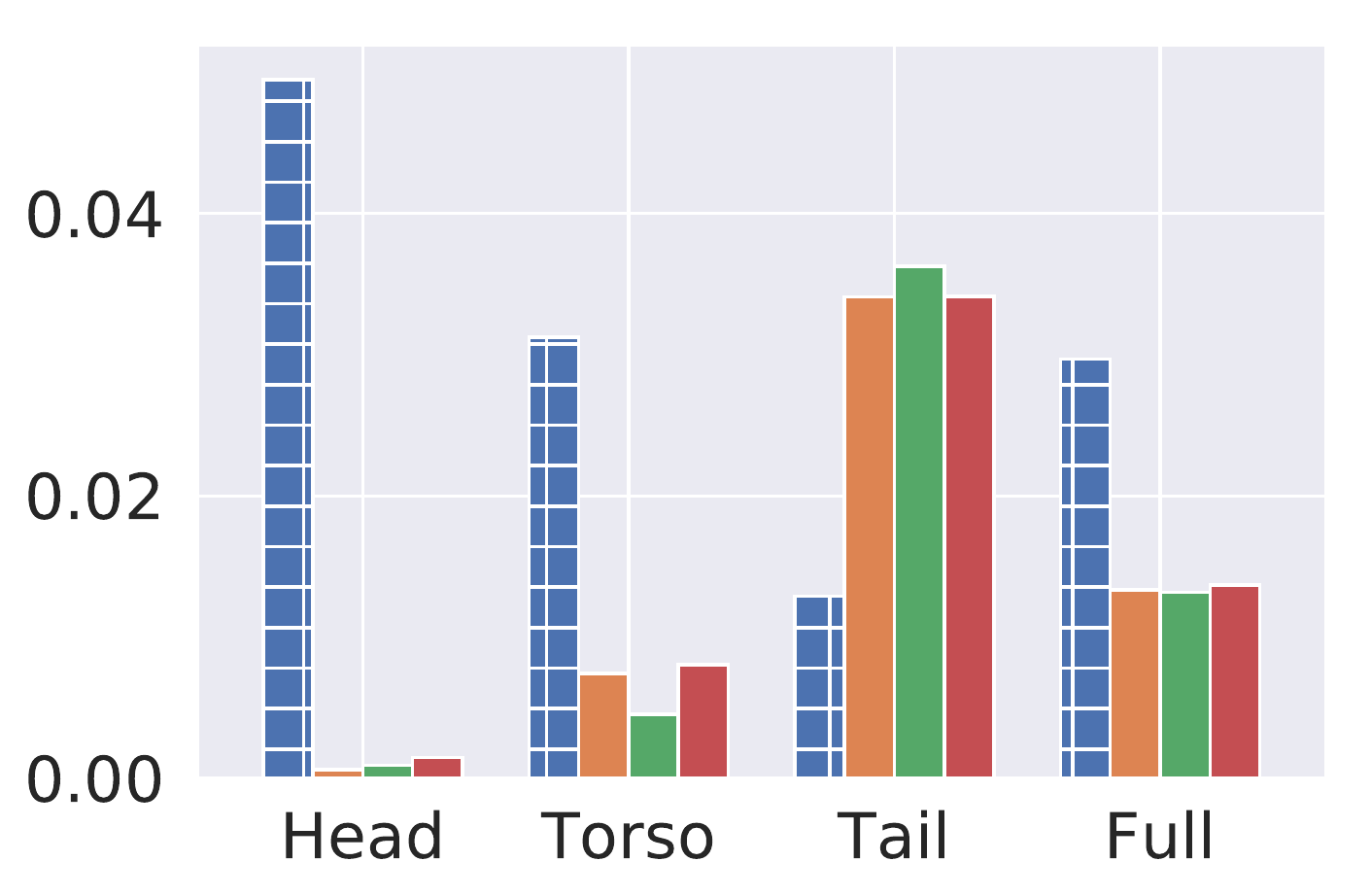}{\delicioussmall}
    }    
    \;
    \subcaptionbox{$\recall$10.\label{fig:deli-r10}}{
    \figwithcap{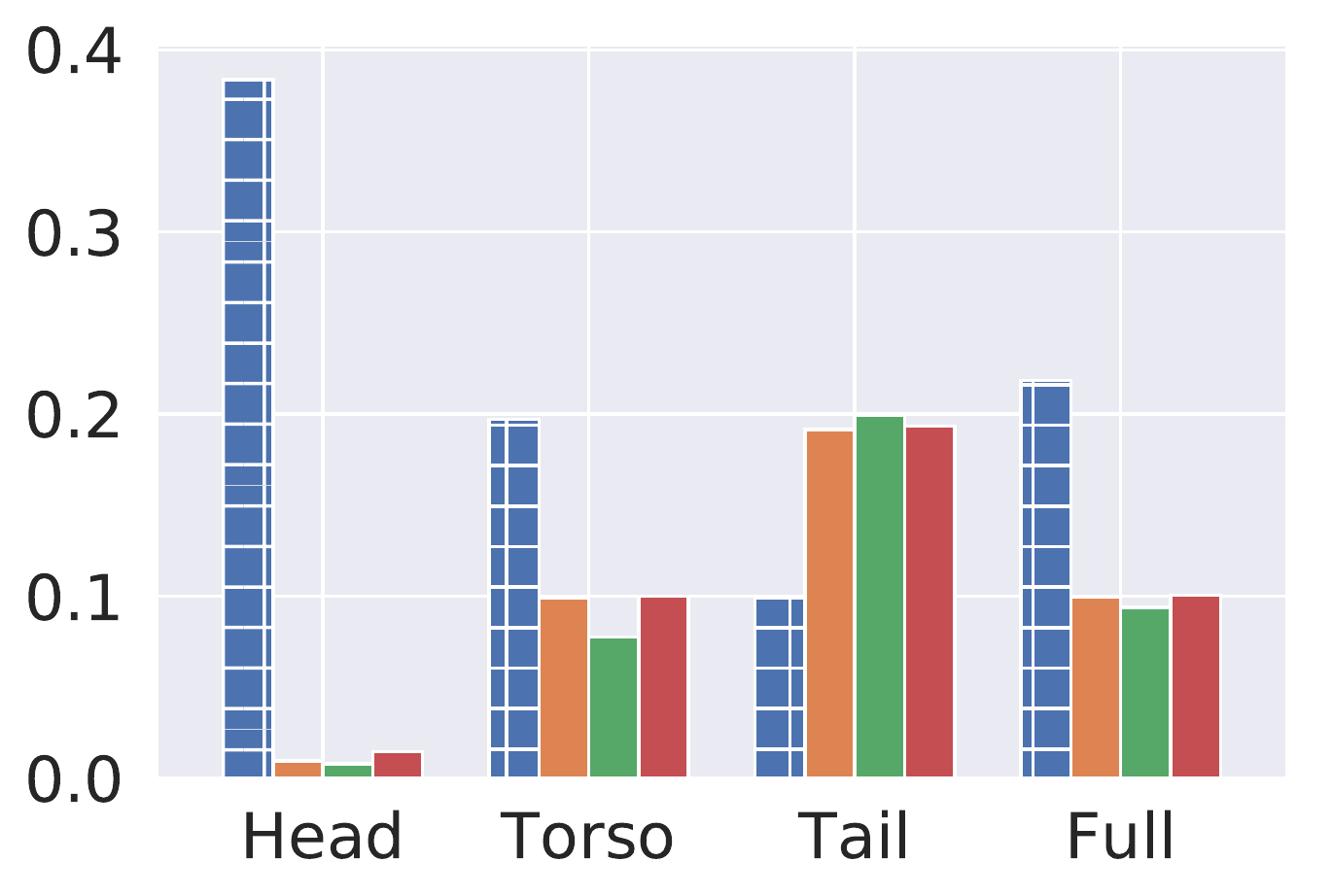}{\delicioussmall}
    }
    \;
    \subcaptionbox{$\recall$50.\label{fig:deli-r50}}{
    \figwithcap{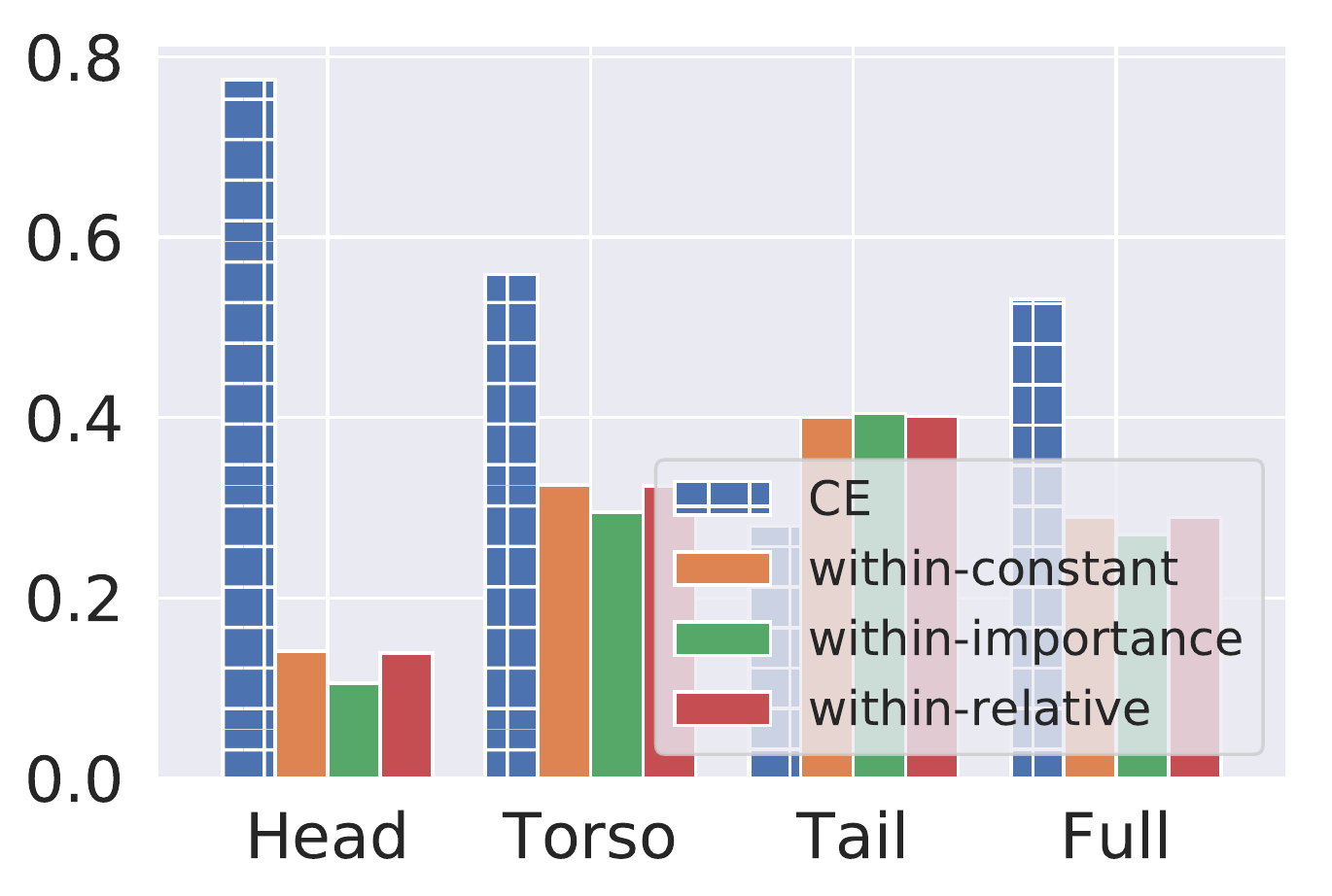}{\delicioussmall}
    }
    }
    
    \caption{Performance of 
    cross-entropy loss
    with
    within-batch negative sampling 
     (cf.~\eqref{eqn:weighted-sampled-softmax}) on \amazonsmall~(Figure~\ref{fig:am-r1} - \ref{fig:am-r50}), \wikilshtc~(Fig.~\ref{fig:wiki-r1} - \ref{fig:wiki-r50}), and \delicioussmall~(Fig.~\ref{fig:deli-r1} - \ref{fig:deli-r50}). 
     These experiments utilize $m=256$ negative for \amazonsmall{} and \wikilshtc{}, and $m = 64$ negatives for \delicioussmall{}. 
     We report $\recall{k}$ for $k \in \{ 1, 10, 50 \}$ 
     on three subpopulations (Head, Torso, and Tail) and the entire test set (Full).
     We combine within-batch sampling with constant, importance, and relative weighting schemes. 
     For reference, we include the results of standard softmax cross-entropy loss (${\tt ce}$). For all weight choices,
     within-batch sampling focuses on Tail subpopulation.
   }
    \label{fig:rb}
\end{figure*}

\subsection{Results on long-tail benchmarks}
\label{sec:long-tail-results}

We present results on 
long-tailed (``LT'') versions of the CIFAR-100
and ImageNet datasets.
All datasets feature a rapidly decaying label distribution. Note that 
negative sampling is hardly necessary on these datasets,
as the number of labels $L$ ($100$ and $1000$ respectively) 
is not prohibitively large.
Nonetheless,
our aim is to verify
that
under standard long-tail learning setups,
different negative sampling schemes trade-off performance on dominant versus rare classes;
we do \emph{not} seek to improve upon existing baselines.

For CIFAR-100,
we
downsample labels
via the
\expp 
and
{\sc Step}
profiles of~\citet{Cui:2019,Cao:2019},
with imbalance ratio
${\max_{y} \Pr( y )} / {\min_{y} \Pr( y )} = 100$. 
For ImageNet, 
we use the long-tailed version from~\citet{Liu:2019}.
We train a ResNet-56 for CIFAR and a ResNet-50 for ImageNet,
using SGD with momentum;
see Appendix~\ref{app:architectures} for details.

For each dataset, we consider the following techniques:
\begin{itemize}[itemsep=0pt,topsep=0pt,label=--,leftmargin=12pt]
    \item standard training with the softmax cross-entropy~\eqref{eqn:softmax-xent},
    decoupled ``cosine contrastive'' loss~\citep{Hadsell:2006} with margin $\rho = 0$,
    and the logit adjustment loss of~\citet{Menon:2020},
    a representative long-tail learning baseline.
    These techniques use \emph{all} $L$ labels.

    \item negative sampling using
    uniform \& within-batch
    negatives;
    and
    constant, 
    importance,
    \&
    relative
    weighting.
    Additionally, we employ the 
    ``tail'' weights 
    for each sampler
    that mimic the logit-adjusted loss, as per Table~\ref{tbl:sampling-summary-v2}.
    We use $m = 32$ negatives on CIFAR-100,
    and $m = 512$ negatives on ImageNet.
\end{itemize}

We report the test set \emph{balanced} error,
which averages the per-class error rates;
this emphasises tail classes more.
We slice this error 
by breaking down classes into three groups
---
``Head'', ``Torso'', and ``Tail''
---
depending on whether the number of associated samples is $\geq 100$, between $20$ and $100$, or $< 20$ respectively, per~\citet{Kang:2020}.
We confirm that different sampling techniques trade-off performance on dominant (Head) versus rare (Tail) classes.

\textbf{Summary of findings}.
Figure~\ref{fig:lt_benchmarks} summarises our results
using within-batch and uniform sampling, and the softmax cross-entropy.
(See Appendix~\ref{app:additional-results} for contrastive loss results.)
We observe the following.

\emph{Within-batch sampling \emph{and} constant weighting helps on the tail}.
Per our analysis,
using within-batch negatives with a constant weighting
---
which we argued approximates the equalised loss of~\citet{Tan:2020}
---
performs well on tail classes.
Conversely, 
with relative weighting, 
performance on head classes is superior.
{This is despite both schemes providing biased estimates of the softmax cross-entropy.}

\emph{The sampling distribution matters}.~Note that different weighting schemes show markedly different trends on within-batch versus uniform sampling. This highlights that the choice of $(q, w)$ crucially dictates the nature of the trade-off on dominant versus rare classes. 

\emph{Handling sampling and labeling bias}.
Note that the weights given by~\eqref{eqn:logit-weight} significantly improve performance on tail classes. This confirms that, given a target loss (in this case, logit adjustment), one can design sampling schemes to mimic the loss. Thus, as per our analysis, one can jointly handle sampling and labeling bias.

\emph{The limits of sampling}.
The logit adjustment baseline significantly outperforms the cross-entropy on tail classes
---
in keeping with~\citet{Menon:2020}
---
but also slightly improves over the ``tail'' sampling proposal which seeks to approximate it. This indicates
the variance introduced by working on a subset of labels comes at a price. In practice,
one can reduce the variance by
choosing as large a number of sampled labels as is computationally feasible.

\subsection{Results on retrieval benchmarks}

We next focus on the retrieval benchmarks from the extreme classification literature~\citep{Agrawal:2013,Bengio:2019}, where due to a large number of labels it is common to employ negative sampling. In particular, we experiment with \amazonsmall and \wikilshtc{} datasets from the extreme classification repository~\citep{V18}, with 13K and 325K labels, respectively. 
{In addition, we also explored a small scale dataset \delicioussmall{} from the repository to make our conclusions more general.} See Appendix~\ref{appen:dataset-details} for various statistics of these datasets. Given these inherently {\rm multilabel} datasets, we first construct their multiclass versions via a standard reduction procedure~\citep{Reddi19a, Wydmuch:2018}: for a multilabel example $(x, \{y_1,\ldots, y_t\})$ with $t$ positive labels, 
we create $t$ multiclass examples $\{(x, y_1),\ldots, (x, y_t)\}$.

We employ a fully connected neural network with a single hidden layer of width $512$ using a linear activation as model architecture. We experiment with both the
softmax cross-entropy and cosine contrastive loss, and employ the negative sampling techniques described in the previous section.

To assess the impact of different sampling and weighting schemes on different subpopulations of these datasets, we partition the multiclass test set (``Full'') into three subpopultations based on the label frequency. These three subpopulations corresponds ``Head'', ``Torso'', and ``Tail'' labels, with decreasing frequency in training set.

\textbf{Summary of findings}.
Figure~\ref{fig:rb}  shows the results of the $\recall{k}$ metric for all three retrieval datasets when softmax cross-entropy loss and within-batch sampling are employed. We defer the results for uniform sampling to Appendix~\ref{sec:additional-retrieval}. Furthermore, we also present the $\precision{k}$ results on the original (multilabel) datasets in Appendix~~\ref{sec:additional-retrieval}. We make the following observations.

\emph{Within-batch sampling helps on the tail}.
Our results show that within-batch sampling consistently boosts the performance on the Tail subpopulation. This is in contrast with the uniform sampling (cf.~Appendix~\ref{sec:additional-retrieval}).

\emph{Tradeoff of Head versus Tail labels}.
It is evident from our results that there is a trade-off in between the model's performance on Head versus Tail subpopulation. In particular, within batch sampling and uniform sampling (which is aligned with the standard softmax cross-entropy loss on the retrieval datasets) boost the Tail and Head respectively.

Overall, the results on the long-tail and retrieval benchmarks confirm the implicit trade-offs associated with negative sampling, and
illustrate that 
through suitable choice of $(q, w)$,
one may effectively tackle both sampling and labeling bias.

\section{Future work}
We have shown that negative sampling methods
---
while devised for the problem of learning in large-output spaces
---
implicitly trade-off performance on dominant versus rare classes,
and 
in particular, 
minimise losses that cope with long-tail label distributions.
Potential future work includes devising instance-dependent samplers that can mimic~\eqref{eqn:q-ast-loss}, e.g.,
per~\citet{Bamler:2020},
and 
considering the labeling bias effects of generic instance-dependent weights.

\section*{Acknowledgements}

The authors thank Ed Chi for valuable discussions and suggestions.

\bibliographystyle{plainnat}
\bibliography{references}

\appendix
\appendix

\section{Proofs of results in body}

\begin{proof}[Proof of Lemma~\ref{lemm:contrastive-expectation}]
Observe that
\begin{align*}
    \E{\NegLabels \sim q^m}{ \ell(y, f(x); \NegLabels) } &= \E{\NegLabels \sim q^m}{ \phi( f_y( x ) ) + \sum_{i = 1}^m w_{y, s_i} \cdot \varphi( -f_{s_i}( x ) ) } \\
    &= \phi( f_y( x ) ) + \sum_{i = 1}^m \E{s_i \sim q}{ w_{y, s_i} \cdot \varphi( -f_{s_i}( x ) ) } \\
    &= \phi( f_y( x ) ) + m \cdot \E{y' \sim q}{ w_{y y'} \cdot \varphi( -f_{y'}( x ) ) } \\
    &= \phi( f_y( x ) ) + m \cdot \sum_{y' \neq y}{ q( y' \mid y ) \cdot w_{y y'} \cdot \varphi( -f_{y'}( x ) ) } \\
    &= \phi( f_y(x) ) +
    m \cdot \sum_{y' \neq y} \rho_{yy'} \cdot \varphi( -f_{y'}( x ) ),
\end{align*}
where $\rho_{yy'} = m \cdot w_{yy'} q(y' | y)$. Similarly,
\begin{align*}
    \V{\NegLabels \sim q^m}{ \ell(y, f(x); \NegLabels, w) } &= \V{\NegLabels \sim q^m}{ \phi( f_y( x ) ) + \sum_{i = 1}^m w_{y, s_i} \cdot \varphi( -f_{s_i}( x ) ) } \\
    &= \V{\NegLabels \sim q^m}{ \sum_{i = 1}^m w_{y, s_i} \cdot \varphi( -f_{s_i}( x ) ) } \\
    &= m \cdot \V{y' \sim q}{ w_{y, y'} \cdot \varphi( -f_{y'}( x ) ) } \\
    &= m \cdot \E{y' \sim q}{ \left[ w_{y, y'} \cdot \varphi( -f_{y'}( x ) )  \right]^2 } - m \cdot \left[ \E{y' \sim q}{ w_{y, y'} \cdot \varphi( -f_{y'}( x ) ) } \right]^2 \\
    &= m \cdot \sum_{y' \neq y}{ q(y' \mid y) \cdot \left[ w_{y, y'} \cdot \varphi( -f_{y'}( x ) )  \right]^2 } - \\
    & \qquad \qquad m \cdot \left[ \sum_{y' \neq y}{ q( y' \mid y ) \cdot w_{y, y'} \cdot \varphi( -f_{y'}( x ) ) } \right]^2 \\
    &= \sum_{y' \neq y} w_{yy'} \cdot \rho_{yy'} \cdot \varphi( -f_{y'}( x ) )^2 -
    \frac{1}{m} \cdot \left( \sum_{y' \neq y} \rho_{yy'} \cdot \varphi( -f_{y'}( x ) ) \right)^2.
\end{align*}
\end{proof}

\begin{proof}[Proof of Lemma~\ref{lemm:expected-sampled-loss}]
Observe that
\begin{align*}
    \E{\NegLabels \sim q^m}{ \sum_{y' \in \NegLabels} w_{yy'} \cdot e^{f_{y'}( x ) - f_{y}  ( x )} } &= m \cdot \E{y' \sim q}{ w_{yy'} \cdot e^{f_{y'}( x ) - f_{y}( x )} } \\
    &= m \cdot \sum_{y' \neq y} q_{y'} \cdot w_{yy'} \cdot e^{f_{y'}( x ) - f_{y}( x )}.
\end{align*}
By Jensen's inequality,
the expected loss is thus bounded by
\begin{align*}
    & \log\left[ 1 + \sum_{y' \neq y} m \cdot q_{y'} \cdot w_{yy'} \cdot e^{f_{y'}( x ) - f_{y}( x )} \right].
\end{align*}
\end{proof}

\begin{proof}[Proof of Theorem~\ref{thm:implicit-xent}]
Let $G_{y,m}(x) \defEq e^{f_{y}(x)}+\sum_{y'\in\mathcal{\mathscr{N}}}w_{yy'}e^{f_{y'}(x)}$ be an estimate of the partition function.
From Lemma~\ref{lem:normality_partition_estimate} (Appendix~\ref{sec:helper}), 
\begin{align*}
    & \sqrt{m}\left(\log G_{y,m}(x)-\log\mu_{y}(x)\right)\stackrel{d}{\to}\mathcal{N}\left(0,\sigma_{y}^{2}(x)/\mu_{y}^{2}(x)\right) \\
    \implies &  \sqrt{m}\left( -f_y(x) + \log G_{y,m}(x)- [ -f_y(x) + \log\mu_{y}(x)] \right)\stackrel{d}{\to}\mathcal{N}\left(0,\sigma_{y}^{2}(x)/\mu_{y}^{2}(x)\right) \\
    \implies &  \sqrt{m}\left( \ell(y,f(x);\mathscr{N}, w)  
    -  \ell^{q, w}_m(y, f(x))  
    \right)\stackrel{d}{\to}\mathcal{N}\left(0,\sigma_{y}^{2}(x)/\mu_{y}^{2}(x)\right), 
\end{align*}
where 
$$\ell(y,f(x);\mathscr{N}, w) 
=-f_{y}(x)+\log\Big[e^{f_{y}(x)}+\sum_{y'\in\mathscr{N}}w_{yy'}e^{f_{y}(x)}\Big],$$
$$\ell^{q, w}_m(y, f(x)) = -f_{y}(x)+\log\Big[e^{f_{y}(x)}+\sum_{y' \neq y } \rho_{yy'}e^{f_{y}(x)} \Big], $$ and $\rho_{yy'} \defEq {m \cdot q_{y'} \cdot w_{yy'}}$ from Lemma \ref{lemm:expected-sampled-loss}.
This implies that for
sufficiently large $m$ and for a given $x\in\mathcal{X}$, we have
\begin{align*}
&\mathbb{E}_{\NegLabels \sim q^m}\big[\big( \ell(y, f(x); \NegLabels, w)  - \ell^{q, w}_m(y, f(x)) \big)^2 \big] 
= \frac{\sigma^2_y(x)}{m \cdot \mu_{y}^2(x)} + o_p(1).
\end{align*}
where $o_{p}(1)$ is a random variable that converges to 0 in probability.

\end{proof}

\begin{proof}[Proof of Proposition~\ref{prop:loss-connection}]
From~\eqref{eqn:expected-sampled-loss-special},
for a given $(q, w)$,
the implicit loss for the sampled softmax cross-entropy is
$$ \ell^{q, w}_m(y, f(x)) = \log\Big[ 1 + \sum_{y' \neq y} \rho_{yy'} \cdot e^{f_{y'}( x ) - f_{y}( x )} \Big], $$
where $\rho_{yy'} = m \cdot q_{y'} \cdot w_{yy'}$.
This exactly equals the pairwise margin loss~\eqref{eqn:unified-margin-loss}.
Thus, for a fixed $\rho_{yy'}$, picking $w_{yy'} = \frac{\rho_{yy'}}{m \cdot q_{y'}}$ guarantees the implicit and pairwise margin losses coincide.
\end{proof}

\begin{proof}[Proof of Lemma~\ref{lemm:variance-minimiser}]
This is a simple consequence of the fact that when applying importance weighting $\E{q}{ \frac{p( x )}{q( x )} \cdot f( x ) }$ to approximate an expectation $\E{p}{ f( x ) }$,
the minimum variance choice of $q$ is $q^*( x ) \propto p( x ) \cdot f( x )$ (e.g., see~\citet{Alain:2015}).
\end{proof}

\section{A helper lemma}
\label{sec:helper}

\begin{lemma}
\label{lem:normality_partition_estimate} 

Define $G_{y,m}(x) \defEq e^{f_{y}(x)}+\sum_{y'\in\mathcal{\mathscr{N}}}w_{yy'}e^{f_{y'}(x)}$
as a random variable that estimates the partition function, where
$\mathscr{N}=\{y_{1}',\ldots,y_{m}'\}\stackrel{i.i.d.}{\sim}q$. 
Assume $q \in \Delta_{[L]} $ has $q_y = 0$, and  $w_{yy'}=\frac{1}{m}\eta_{yy'}$
where $\eta_{yy'}$ is independent of $m$. Let $\mu_{y}(x):=\mathbb{E}_{\mathscr{N}\sim q^{m}}G_{y,m}(x)$,
and $\sigma_{y}^{2}(x)=\mathbb{V}_{y'\sim q}[\eta_{yy'}e^{f_{y'}(x)}]$,
both assumed to be strictly positive and finite. Then, for any $x\in\mathcal{X}$, the following
statements hold: 
\begin{enumerate}
\item $\sqrt{m}(G_{y,m}(x)-\mu_{y}(x))\stackrel{d}{\to}\mathcal{N}(0,\sigma_{y}^{2}(x))$; 
\item If $\mu_{y}(x)>0$, then $\sqrt{m}\left(\log G_{y,m}(x)-\log\mu_{y}(x)\right)\stackrel{d}{\to}\mathcal{N}\left(0,\sigma_{y}^{2}(x)/\mu_{y}^{2}(x)\right)$. 
\end{enumerate}
Here, $\stackrel{d}{\to}$ denotes convergence in distribution as
$m\to\infty$. 
\end{lemma}

\begin{proof}

We first note that 
\begin{align*}
\mathbb{E}_{y' \sim q}[G_{y,m}(x)] & =e^{f_{y}(x)}+m\mathbb{E}_{y'\sim q}w_{yy'}e^{f_{y}(x)}\\
& =e^{f_{y}(x)}+\sum_{y' \neq y} m \cdot q_{y'} \cdot w_{yy'}e^{f_{y}(x)} \\
 & =e^{f_{y}(x)}+\sum_{y' \neq y} q_{y'} \eta_{yy'}e^{f_{y}(x)}  = \mu_{y}(x),
\end{align*}
where we use the fact that samples in $\mathscr{N}$ are i.i.d. Equivalently $G_{y,m}(x)=e^{f_{y}(x)}+\frac{1}{m}\sum_{i=1}^{m}\eta_{y,y'_{i}}e^{f_{y'}(x)}$.
Note that $e^{f_{y}(x)}$ is not random, and $\frac{1}{m}\sum_{i=1}^{m}\eta_{y,y'_{i}}e^{f_{y'}(x)}$
is an average of i.i.d. random variables. 
By the central limit theorem, it follows that 
$\sqrt{m}(G_{y,m}(x)-\mu_{y}(x))\stackrel{d}{\to}\mathcal{N}(0,\sigma_{y}^{2}(x))$ where $\sigma_{y}^{2}(x)=\mathbb{V}_{y'\sim q}[\eta_{yy'}e^{f_{y'}(x)}]$.

Recall the Delta method which states that
if a sequence of random variables $( X_{m} )_{m \in \mathbb{Z}+}$ satisfies $\sqrt{m}(X_{m}-\theta)\stackrel{d}{\to}\mathcal{N}(0,\sigma^{2})$
for some $\theta\in\mathbb{R}$ and $\sigma^{2}>0$, then $\sqrt{m}(g(X_{m})-g(\theta))\stackrel{d}{\to}\mathcal{N}(0,\sigma^{2}g'(\theta)^{2})$
for any function $g$ whose derivative $g'(\theta)$ exists at $\theta$,
and is non-zero. Choosing $g(x)=\log(x)$ and applying the Delta method
to $\sqrt{m}(G_{y,m}(x)-\mu_{y}(x))\stackrel{d}{\to}\mathcal{N}(0,\sigma_{y}^{2}(x))$ gives
\begin{align*}
    & \sqrt{m}\left(\log G_{y,m}(x)-\log\mu_{y}(x)\right)\stackrel{d}{\to}\mathcal{N}\left(0,\sigma_{y}^{2}(x)/\mu_{y}^{2}(x)\right) 
\end{align*}

\end{proof}

\section{Additional discussion on negative sampling schemes}
\label{app:additional-q-w}

\begin{remark}
\label{remark:logit-correction}
{For the softmax cross-entropy,
one way of reasoning about the weighting on negative samples is as a \emph{logit correction}.
Observe that~\eqref{eqn:weighted-sampled-softmax} may be rewritten
\begin{equation}
    \ell( y, f( x ); \NegLabels ) 
    = \log\Big[ 1 + \sum\nolimits_{y' \in \NegLabels} e^{\bar{f}_{y'}( x ) - f_{y}( x )} \Big],
\end{equation}
where $\bar{f}_{y'}( x ) = f_{y'}( x ) - \log w_{yy'}$
are \emph{corrected} versions of the original logits or scores.
Note further that if $\NCal$ can include the positive label $y$,
this is tantamount to additionally correcting the positive logit as well.}
\end{remark}

\begin{remark}
Further to Remark~\ref{remark:logit-correction},
the difference between the importance and relative weighting schemes may be understood as follows.
Suppose we employ the softmax cross-entropy 
with explicit exclusion of the positive label from $\NegLabels$.
Further, if we modify the positive logit to $\tilde{f}_y( x )$,
and negative logit to $\bar{f}_y( x )$:
\begin{align*}
    \ell( y, f( x ); \NegLabels )
    &= \log\Big[ 1 + \sum_{y' \in \NCal - \{ y \}} e^{\bar{f}_{y'}( x ) - \tilde{f}_y( x )} \Big].
\end{align*}
By setting $\tilde{f}_y( x ) = f_y( x )$,
and $\bar{f}_{y'}( x ) = f_{y'}( x ) - \log (m \cdot \qYPrime)$, 
we obtain the importance weighting scheme.
On the other hand,
if we additionally set
$\tilde{f}_y( x ) = f_y( x ) - \log (m \cdot \qY)$,
---
i.e.,
apply \emph{positive} logit correction as well
---
then we arrive at the relative weighting scheme.
\end{remark}

\begin{remark}
\label{rem:remove_acc_hits}
As stated,
the sampling distribution
$q$ may place non-zero mass on the ``positive'' label $y$;
thus, one may include $y$ amongst the ``negative'' labels.
As this is intuitively undesirable,
the domain of $q$ may be additionally restricted so as to exclude this possibility.
Further,
one may explicitly discount this label from consideration by zeroing out its weight;
e.g.,
we may apply $w_{yy'} = \indicator{ y' \neq y }$
in place of constant weighting.
This is similar yet \emph{distinct} to forcing $q$ to exclude $y$ from its sampling domain,
as the former implicitly modifies the distribution of negatives.
In practice, however, the two approaches have similar performance.
\end{remark}

\section{Expected decoupled losses under negative sampling}
\label{app:decoupled-table}

Table~\ref{tbl:sampling-summary-contrastive} summarises expected losses under negative sampling for the decoupled case.

\begin{table*}[!ht]
    \centering
    \renewcommand{\arraystretch}{1.5}

    \resizebox{0.99\linewidth}{!}{
    \begin{tabular}{@{}llll@{}}
        \toprule
        \textbf{Sampling distribution} & \textbf{Weighting} & \textbf{Expected loss on negatives} & \textbf{Comment} \\
        \toprule
        Uniform
        &
        Constant ($\frac{1}{m}$)
        &
        $\frac{1}{L-1} \sum_{y' \neq y}{ \varphi( -f_{y'}( x ) ) }$
         &
        Scaled decoupled loss \\
        Uniform &
        Importance weighting ($\frac{L}{m}$)
        &
        $\sum_{y' \neq y}{ \varphi( -f_{y'}( x ) ) }$
         &
        Decoupled loss \\
        Uniform &
        Relative weighting ($1$)
        &
        $\frac{m}{L} \sum_{y' \neq y} \varphi( -f_{y'}( x ) )$
         &
        Scaled decoupled loss \\
        Uniform &
        {$\frac{L}{m } \cdot \frac{\pi_{y'}}{\pi_{y}}$}
        &
        $\sum_{y' \neq y} \frac{\pi_{y'}}{\pi_{y}} \cdot \varphi( -f_{y'}( x ) )$
         &
        Tail-heavy loss \\
        \midrule
        Within-batch
        &
        Constant ($\frac{1}{m}$)
        &
        $\sum_{y' \neq y}{ \pi_{y'} \cdot \varphi( -f_{y'}( x ) ) }$
         &
        Tail-heavy loss \\
        Within-batch &
        Importance weighting ($\frac{1}{m \cdot \pi_{y'}}$)
        &
        $\sum_{y' \neq y}{ \varphi( -f_{y'}( x ) ) }$
         &
        Decoupled loss \\
        Within-batch &
        Relative weighting ($\frac{\pi_{y}}{\pi_{y'}}$)
        &
        ${m \cdot \pi_{y}} \cdot \sum_{y' \neq y} \varphi( -f_{y'}( x ) )$
         &
        Head-heavy loss \\
        Within-batch &
        {$\frac{1}{m \cdot \pi_{y}}$}
        &
        $\sum_{y' \neq y} \frac{\pi_{y'}}{\pi_{y}} \cdot \varphi( -f_{y'}( x ) )$
         &
        Tail-heavy loss \\
        \bottomrule
    \end{tabular}
    }

    \caption{Expectation of
    loss of negatives
    $\sum_{y' \in \NegLabels} w_{yy'} \cdot \varphi( -f_{y'}( x ) )$
    for an example $(x, y)$.
    Here, negatives $\NegLabels$ are sampled from $q$ with $q_y > 0$,
    and
    weighting scheme $w$
    satisfies $w_{yy} = 0$.
    Different choices of $(q, w)$ yield upper bounds which resemble
    losses from the long-tail learning literature,
    such as the equalised loss of~\citet{Tan:2020} and the logit-adjusted loss of~\citet{Menon:2020}.
    }
    \label{tbl:sampling-summary-contrastive}
\end{table*}

\section{Details of long-tail experiments}
\label{app:architectures}

For all datasets,
we use SGD with momentum $0.9$.
Dataset specific settings are given below.

\textbf{CIFAR-100}:
We use a CIFAR ResNet-56
with weight decay of $10^{-4}$
trained for 256 epochs,
using a minibatch size of 128.
We use a stepwise annealed learning rate,
with a base learning rate of 0.1
that is decayed by 0.1 at the 160th epoch,
and by 0.01 at the 180th epoch.
We apply standard CIFAR data augmentation per~\citet{Cao:2019, He:2016}.

\textbf{ImageNet}:
We use a ResNet-50
with weight decay of $5 \times 10^{-4}$
trained for 90 epochs,
using a minibatch size of 512.
We use a cosine learning rate
with a base learning rate of 0.4.
We apply standard ImageNet data augmentation per~\citet{Goyal:2017}.

\section{Additional results: Long-tail datasets}
\label{app:additional-results}

We present additional results on the long-tail learning benchmarks
using a contrastive loss,
and
compare the overall (non-sliced) balanced errors of various methods.

\subsection{Results on contrastive loss}

Figure~\ref{fig:lt_benchmarks_contrastive} shows results using the contrastive loss on the long-tail benchmarks.
Here, the performance of different sampling schemes is more variable compared to the softmax cross-entropy.
In particular, on Tail classes, the performance of sampling is generally poor compared to the baseline.
Note that the latter is the de-facto choice of loss function for long-tail settings.
Consequently, the default hyperparameters (e.g., learning rate and batch size) are generally attuned to this loss.
Further tuning of these may improve the results for the contrastive loss.

\begin{figure*}[!hpt]
    \centering

    \resizebox{\linewidth}{!}{
    \subcaptionbox{CIFAR-100-LT Step (contrastive loss).}{
    \includegraphics[scale=0.25,valign=t]{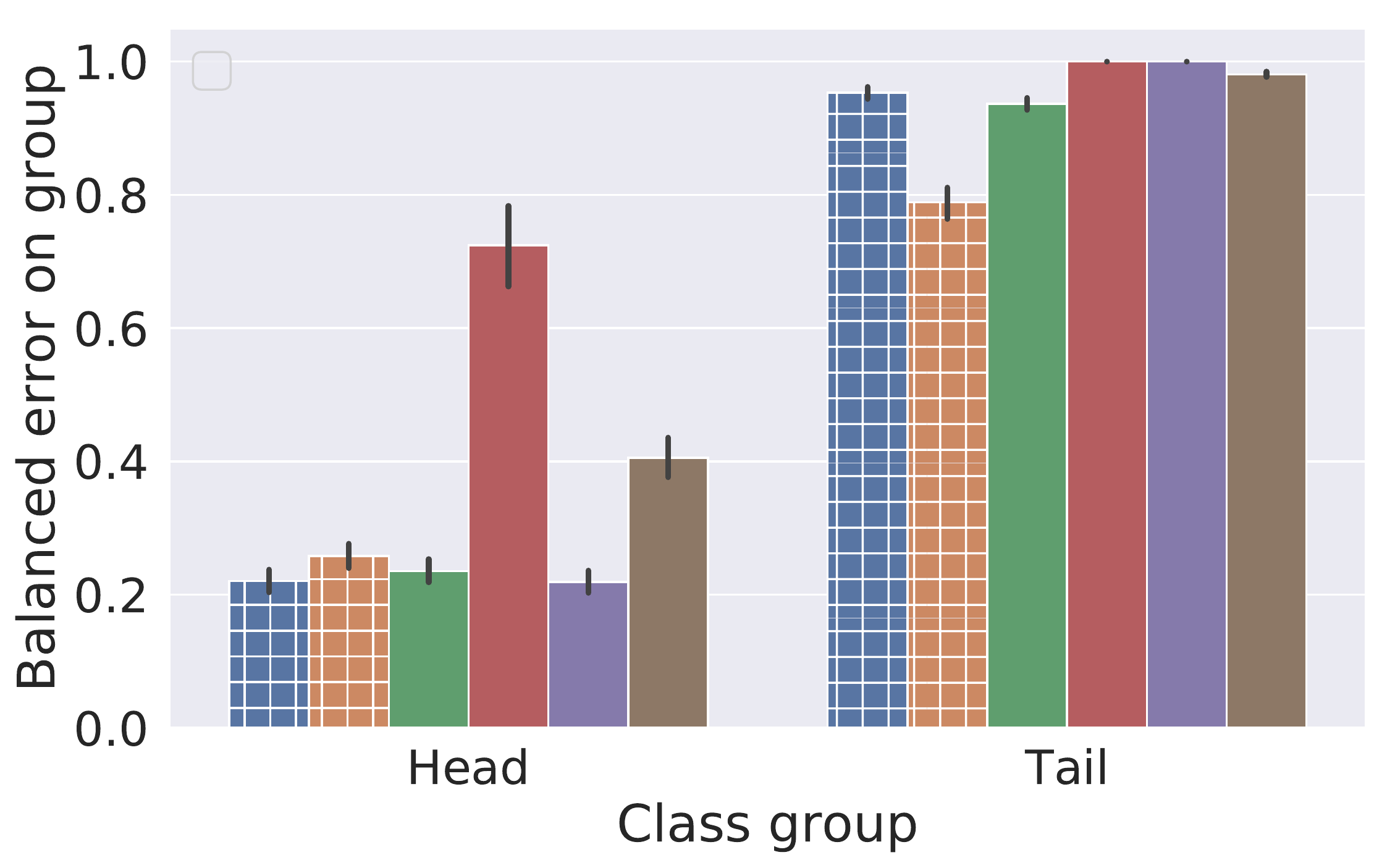}
    }
    \qquad
    \subcaptionbox{CIFAR-100-LT Exp (contrastive loss).}{
    \includegraphics[scale=0.25,valign=t]{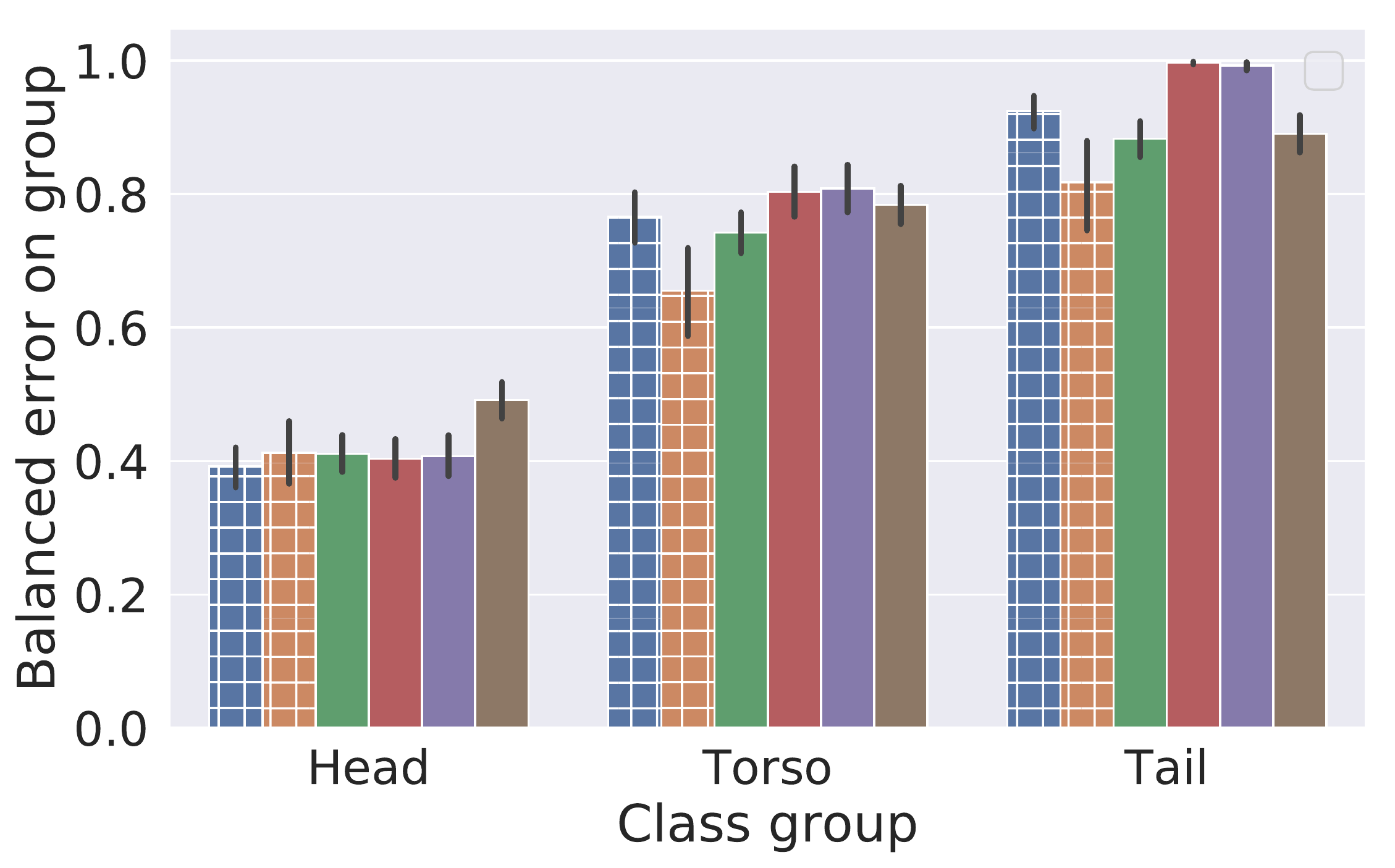}
    }
    \qquad
    \subcaptionbox{ImageNet-LT (contrastive loss).}{
    \includegraphics[scale=0.25,valign=t]{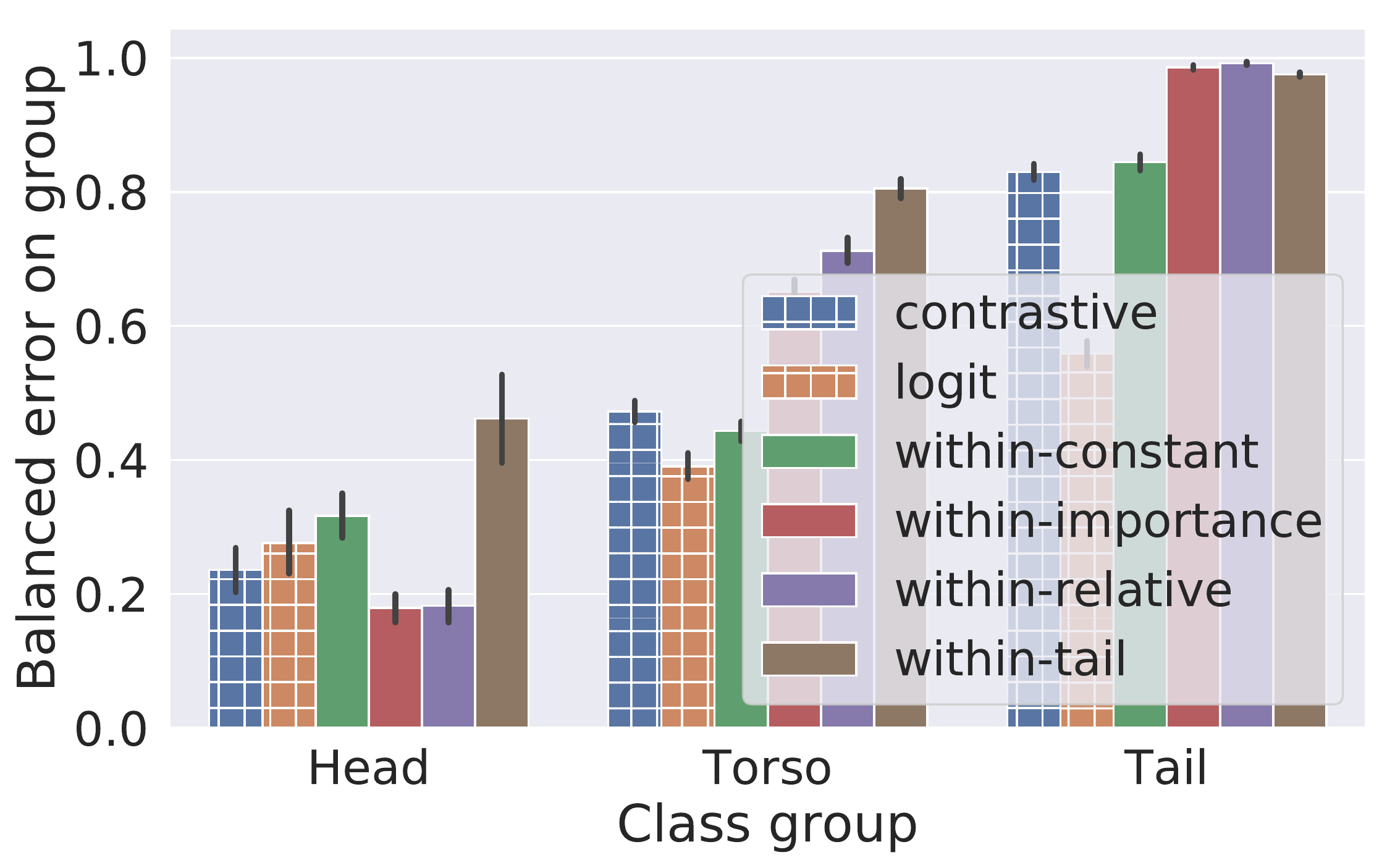}
    }
    }

    \resizebox{\linewidth}{!}{
    \subcaptionbox{CIFAR-100-LT Step (contrastive loss).}{
    \includegraphics[scale=0.25,valign=t]{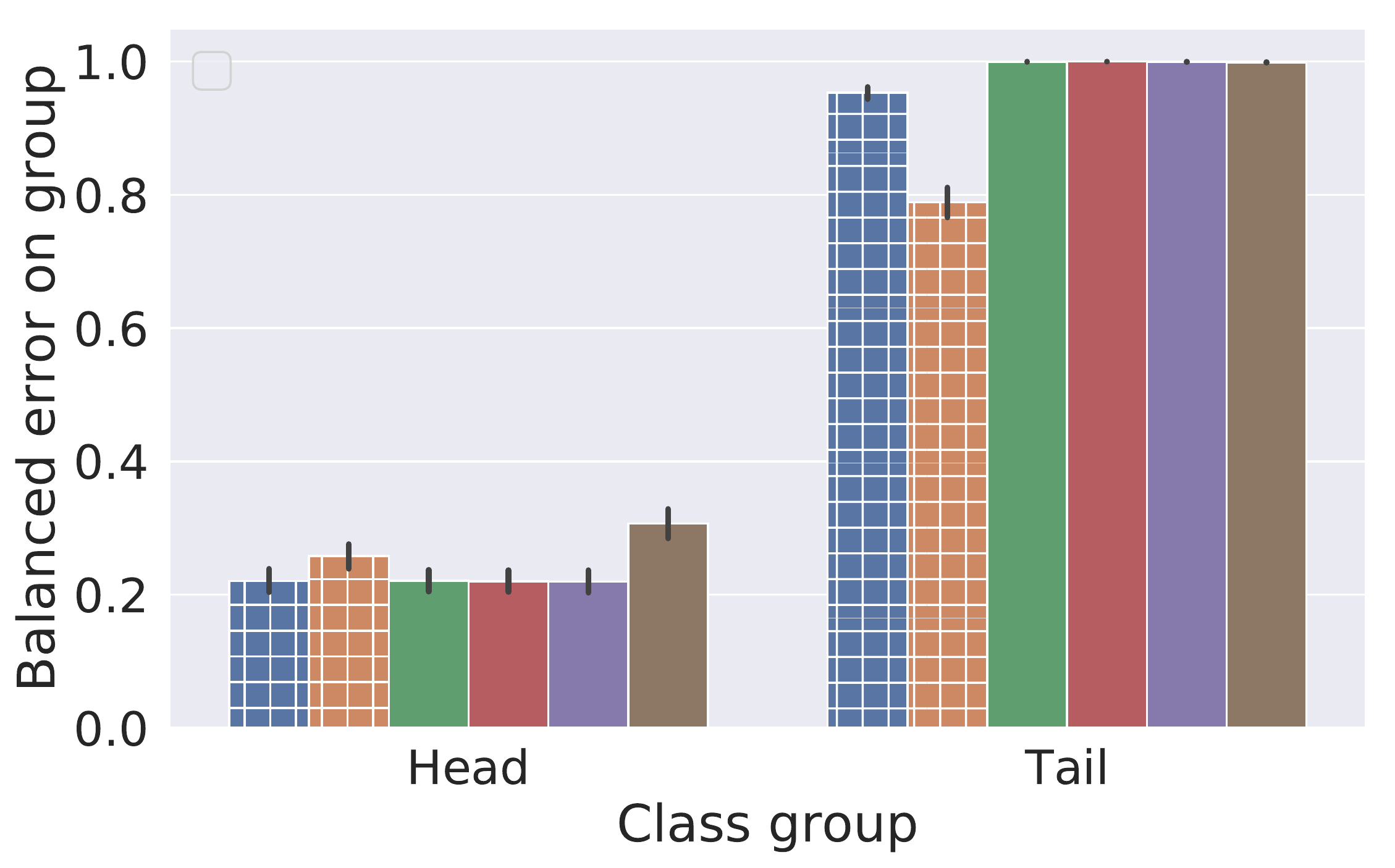}
    }
    \qquad
    \subcaptionbox{CIFAR-100-LT Exp (contrastive loss).}{
    \includegraphics[scale=0.25,valign=t]{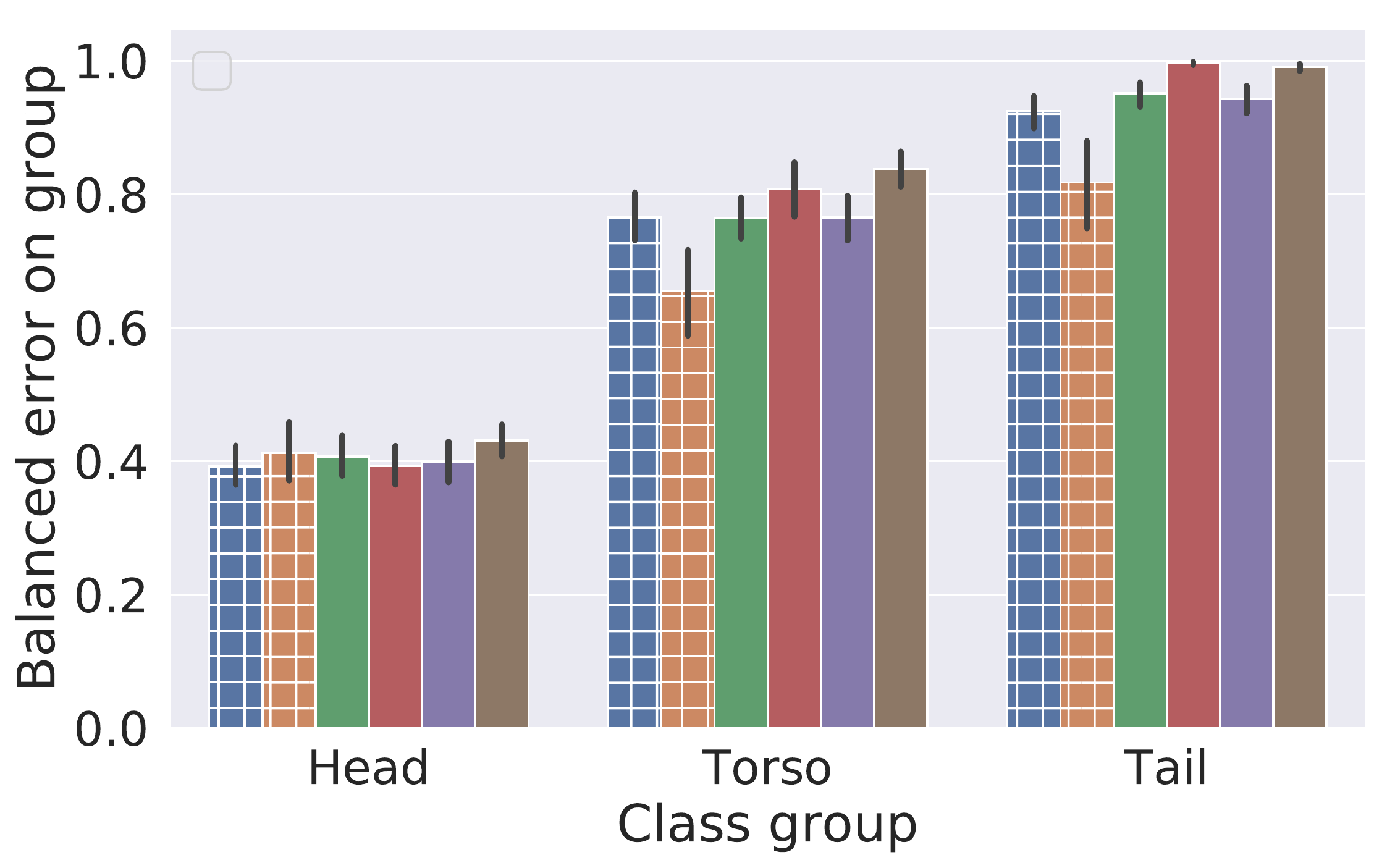}
    }
    \qquad
    \subcaptionbox{ImageNet-LT (contrastive loss).}{
    \includegraphics[scale=0.25,valign=t]{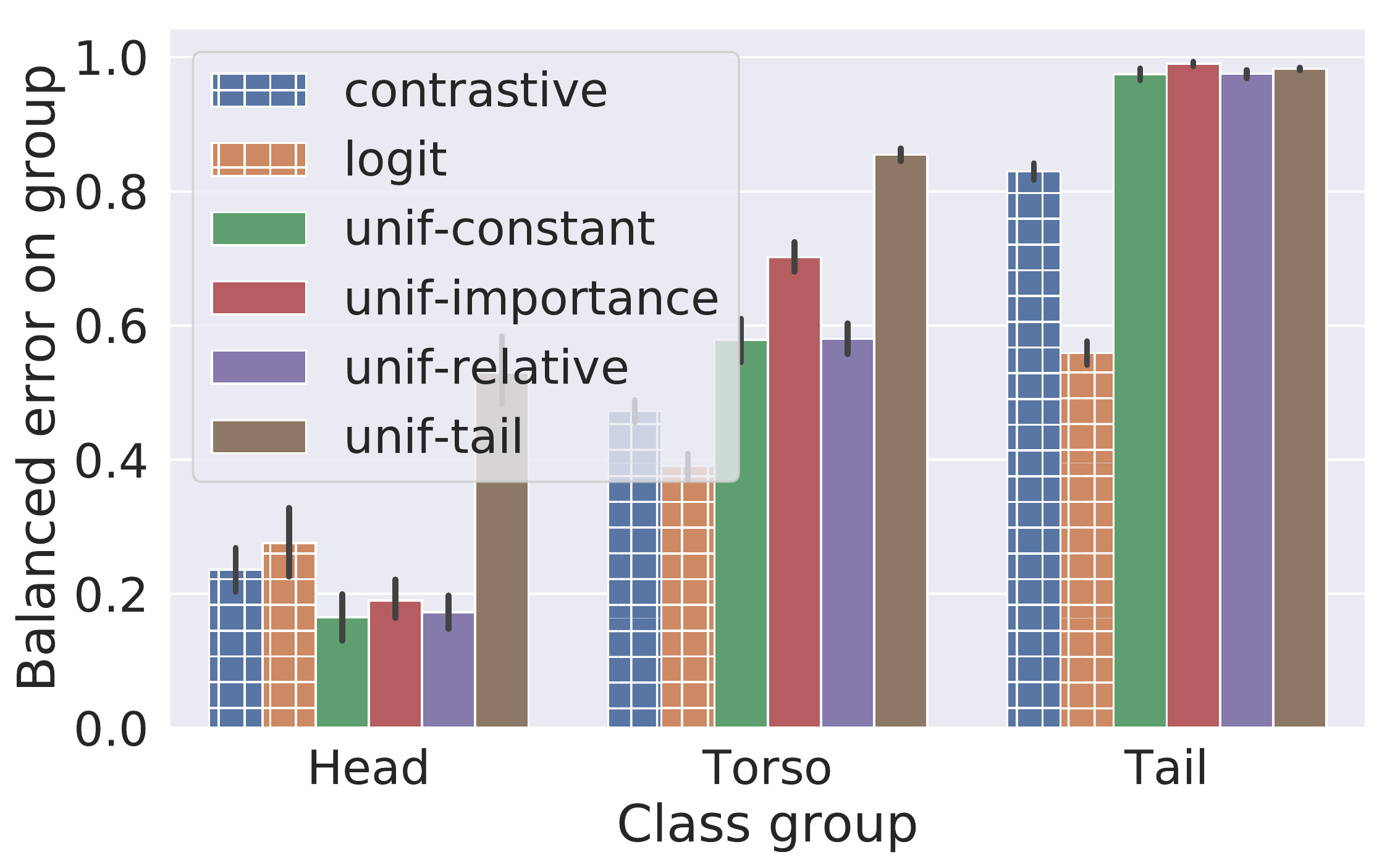}
    }
    }

    \caption{Results on head, torso and tail labels on long-tail learning benchmarks, using the contrastive loss.}
    \label{fig:lt_benchmarks_contrastive}
\end{figure*}

\subsection{Balanced error plots}

Figures~\ref{fig:lt_benchmarks_ber} and~\ref{fig:lt_benchmarks_ber_contrastiv_e_unif}
present the balanced errors of the various choices of sampling and weighting schemes,
for the softmax cross-entropy and contrastive loss respectively.
We see that the gains of within-batch sampling with constant weighting are such that it can improve over the standard loss using \emph{all} the labels.
In general, performance is superior using the softmax cross-entropy versus contrastive loss;
this is in keeping with the former's extensive use as a foundation in long-tail problems.

\begin{figure*}[!t]
    \centering

    \resizebox{\linewidth}{!}{
    \subcaptionbox{CIFAR-100-LT (Step).}{ \includegraphics[scale=0.25,valign=t]{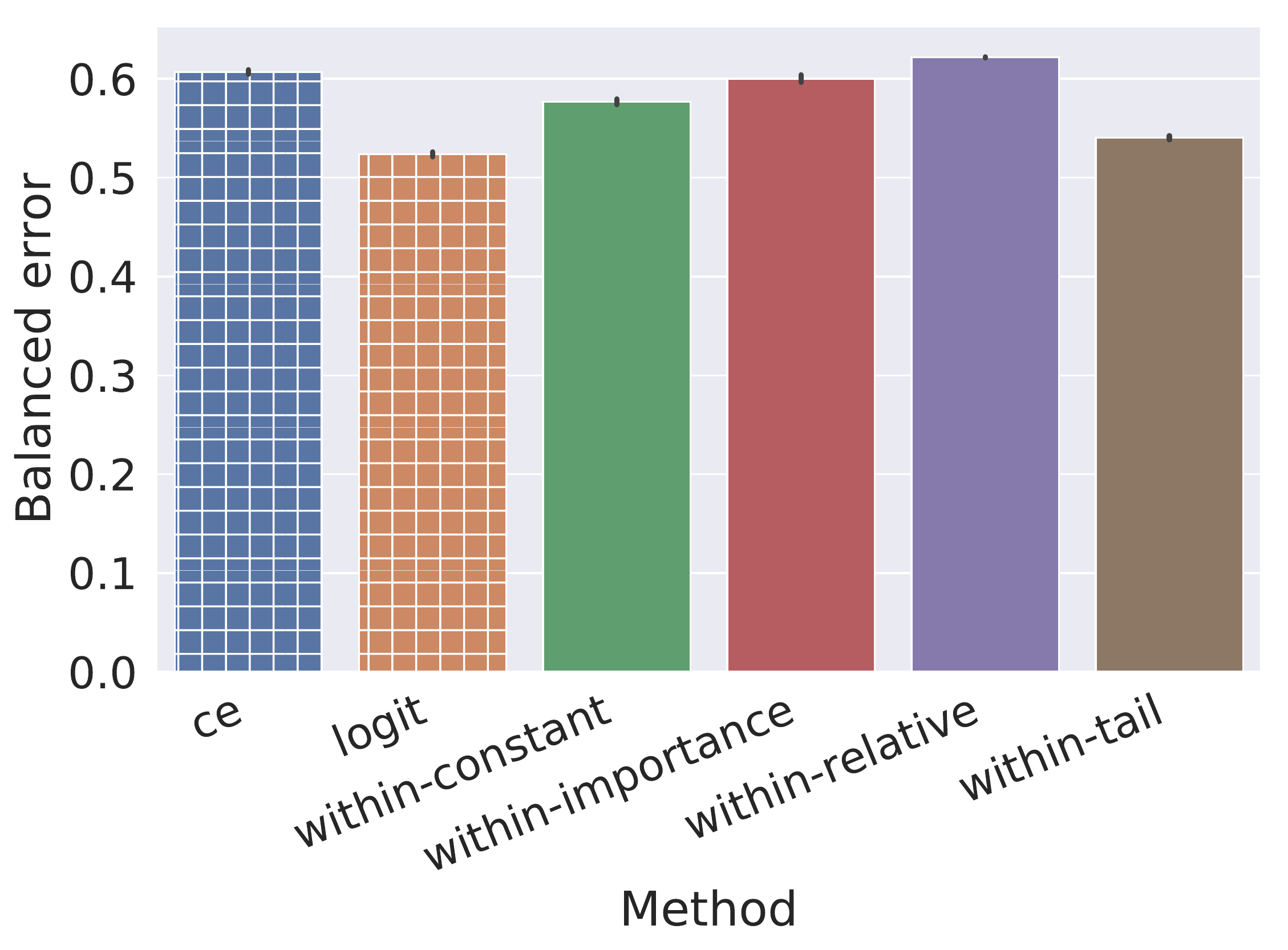}}
    \qquad
    \subcaptionbox{CIFAR-100-LT (Exp).}{
    \includegraphics[scale=0.25,valign=t]{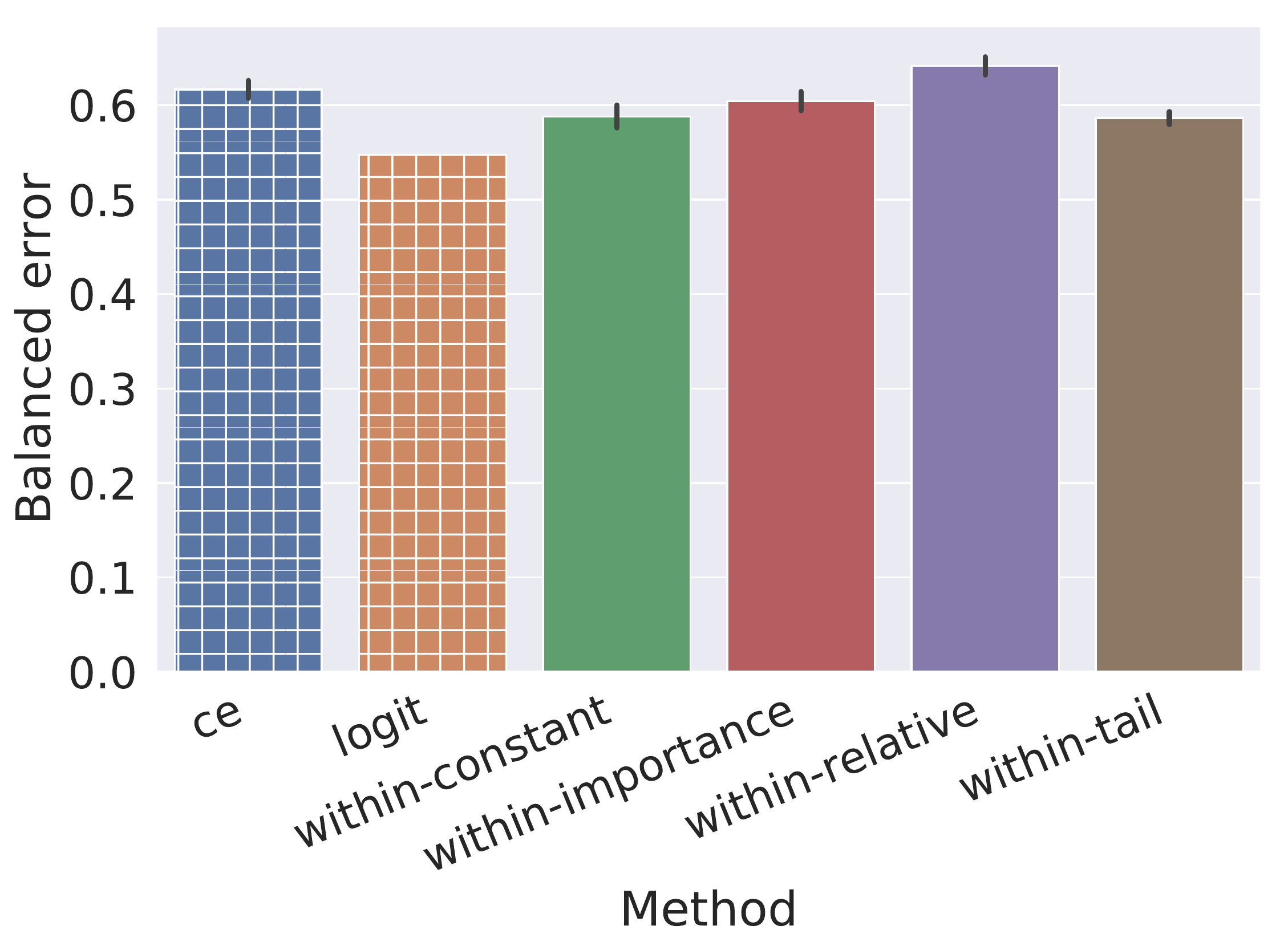}}
    \qquad
    \subcaptionbox{ImageNet-LT.}{
    \includegraphics[scale=0.25,valign=t]{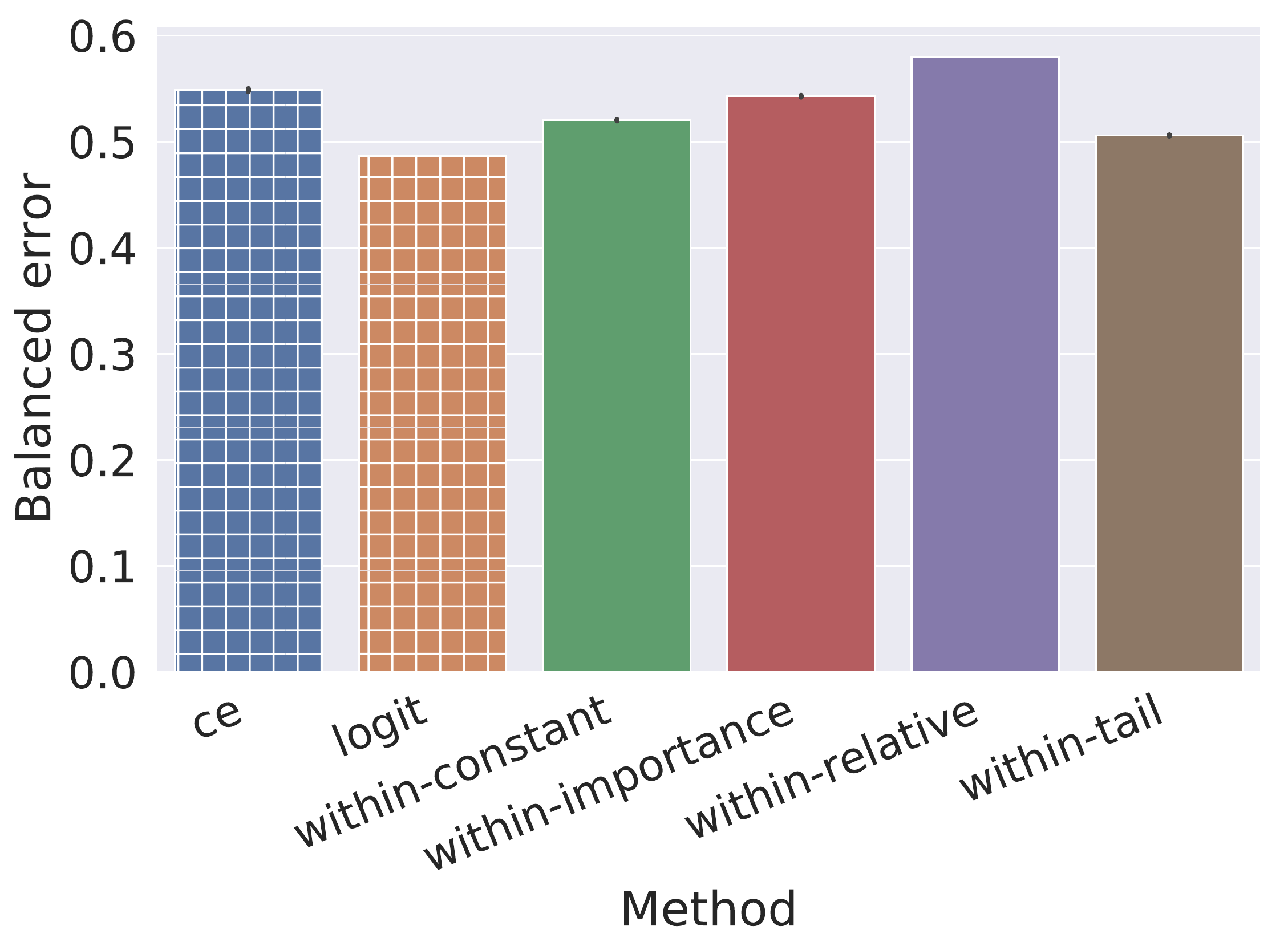}}
    }

    \resizebox{\linewidth}{!}{
    \subcaptionbox{CIFAR-100-LT (Step).}{ \includegraphics[scale=0.25,valign=t]{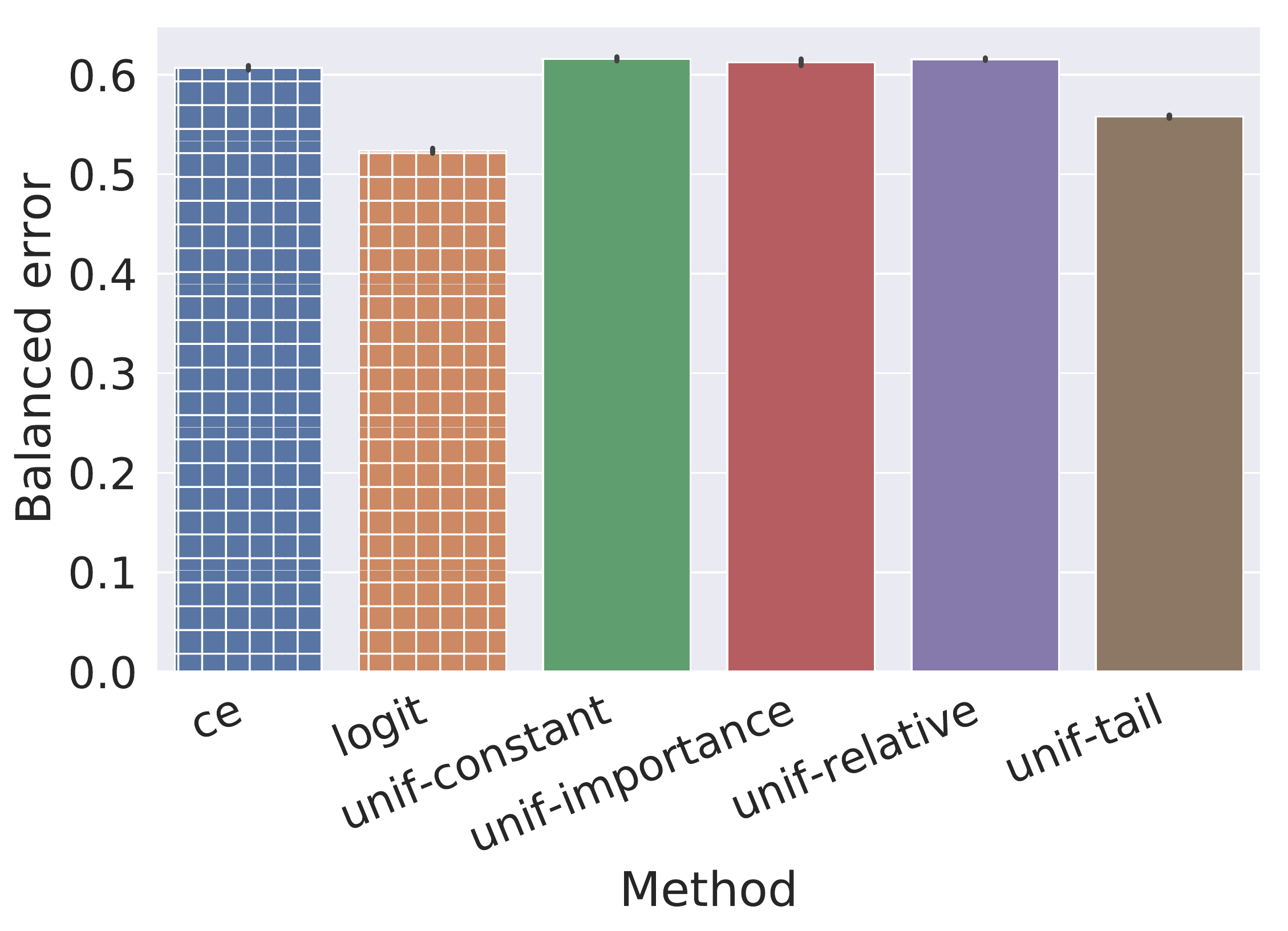}}
    \qquad
    \subcaptionbox{CIFAR-100-LT (Exp).}{
    \includegraphics[scale=0.25,valign=t]{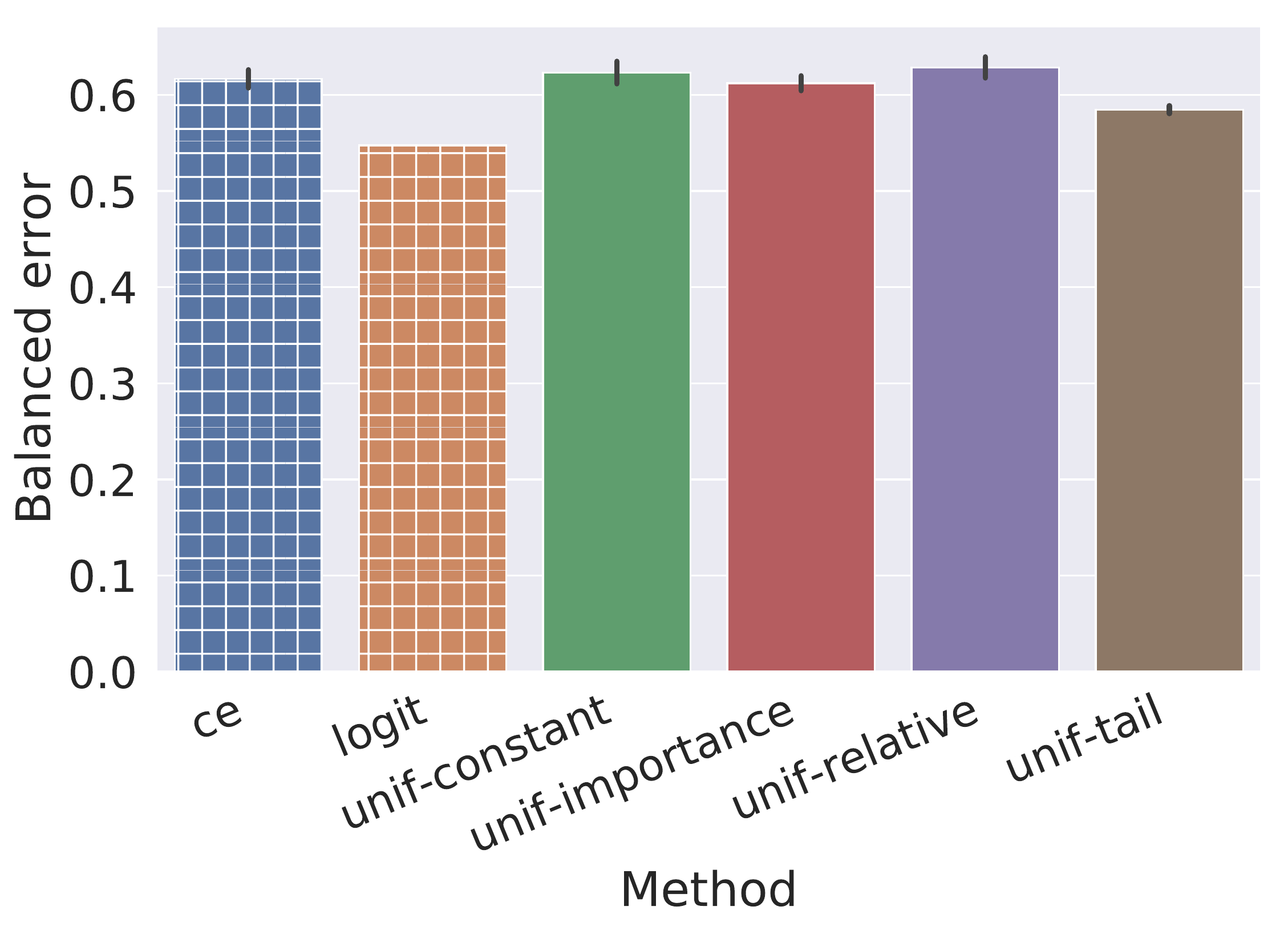}}
    \qquad
    \subcaptionbox{ImageNet-LT.}{
    \includegraphics[scale=0.25,valign=t]{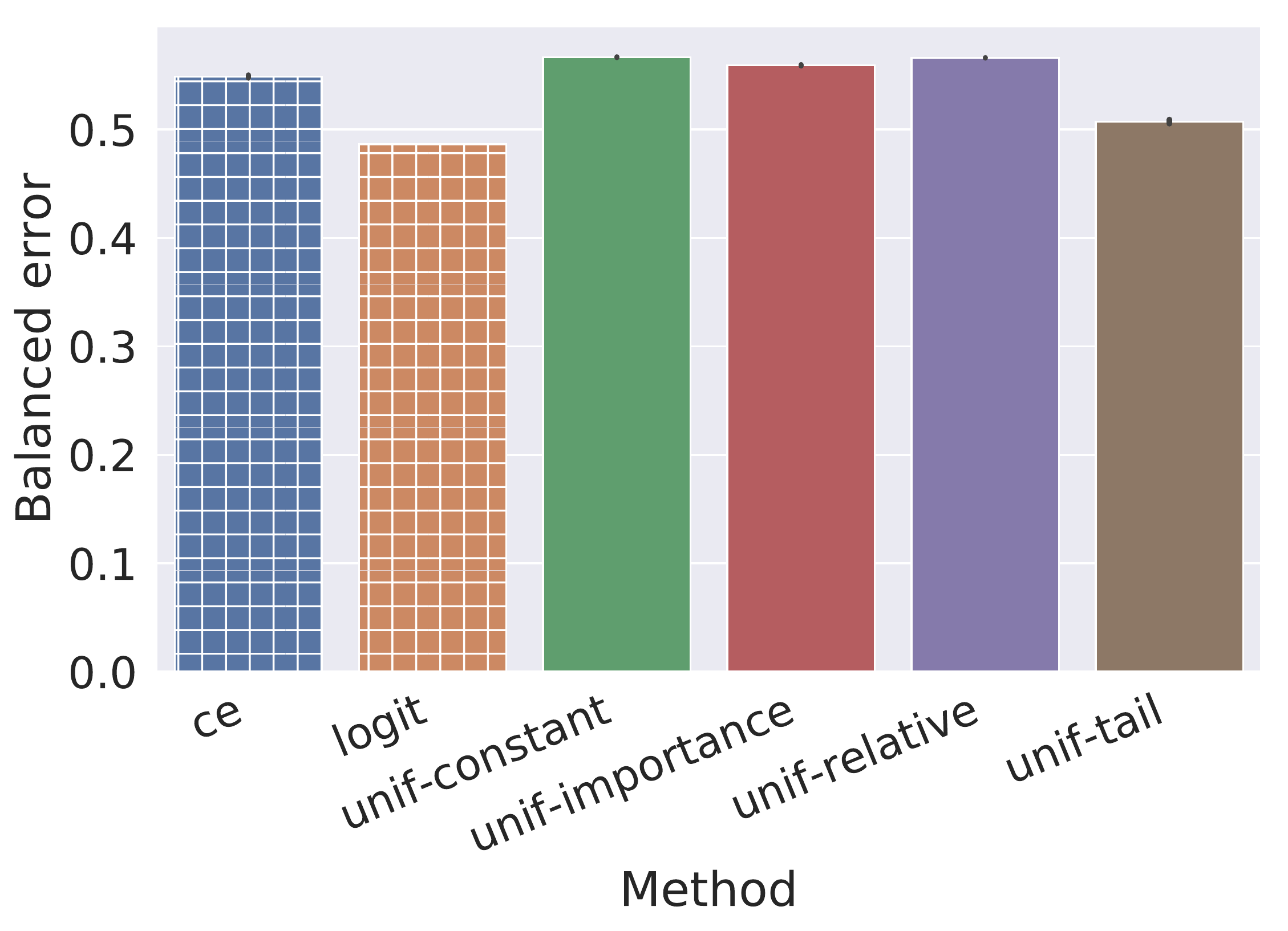}}
    }

    \caption{Balanced error
    on long-tail learning benchmarks
    using the softmax cross-entropy.
    We present results for within-batch (${\tt within}$) and uniform
    (${\tt unif}$)
    negative sampling ,
    using the constant weight (${\tt const}$), importance weighting (${\tt importance}$), and relative weighting (${\tt relative}$) schemes from Table~\ref{tbl:summary-decoupled}.}
    \label{fig:lt_benchmarks_ber}
\end{figure*}

\begin{figure*}[!pt]
    \centering

    \resizebox{\linewidth}{!}{
    \subcaptionbox{CIFAR-100-LT Step.}{ \includegraphics[scale=0.25,valign=t]{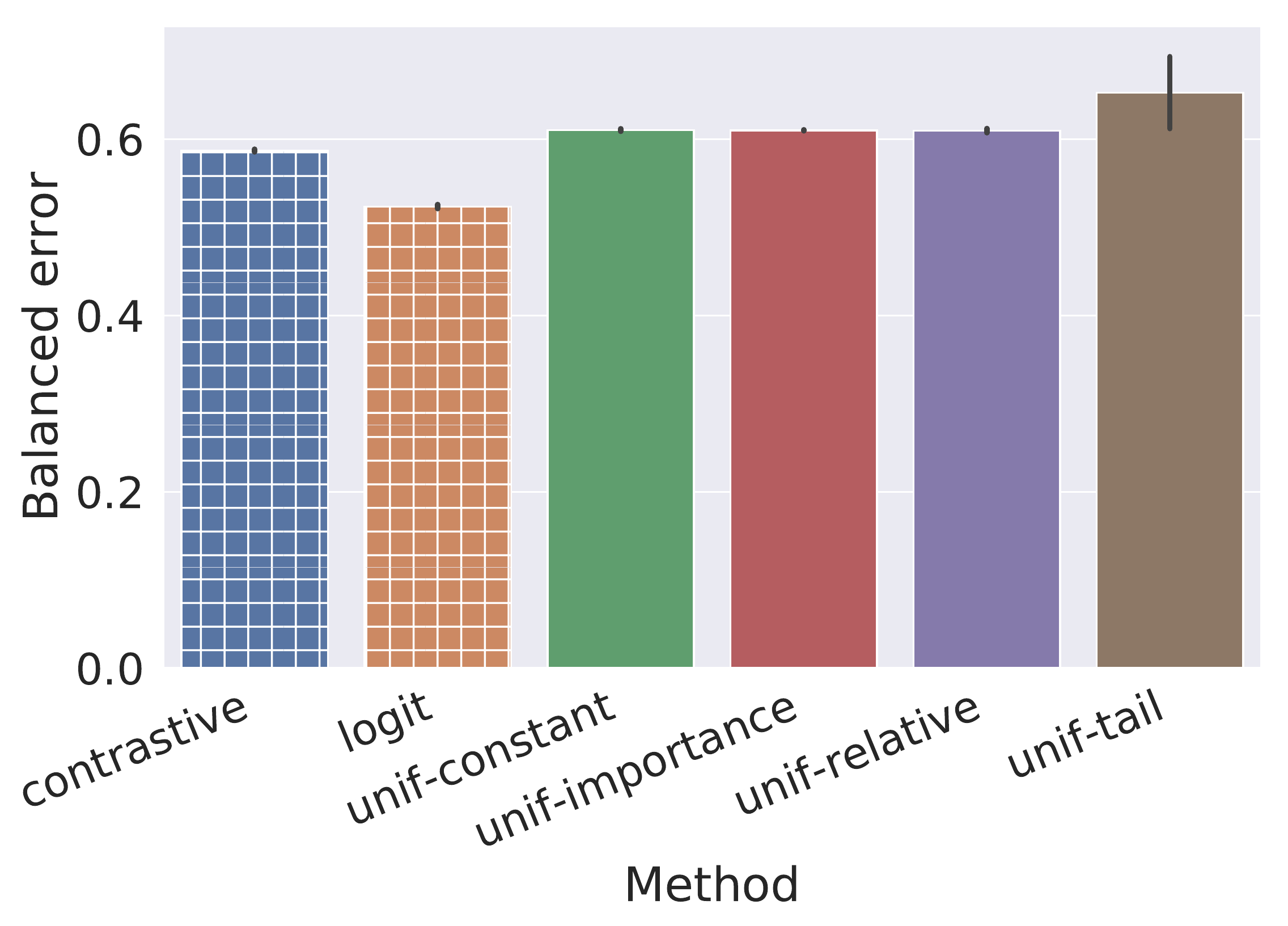}}
    \qquad
    \subcaptionbox{CIFAR-100-LT Exp.}{
    \includegraphics[scale=0.25,valign=t]{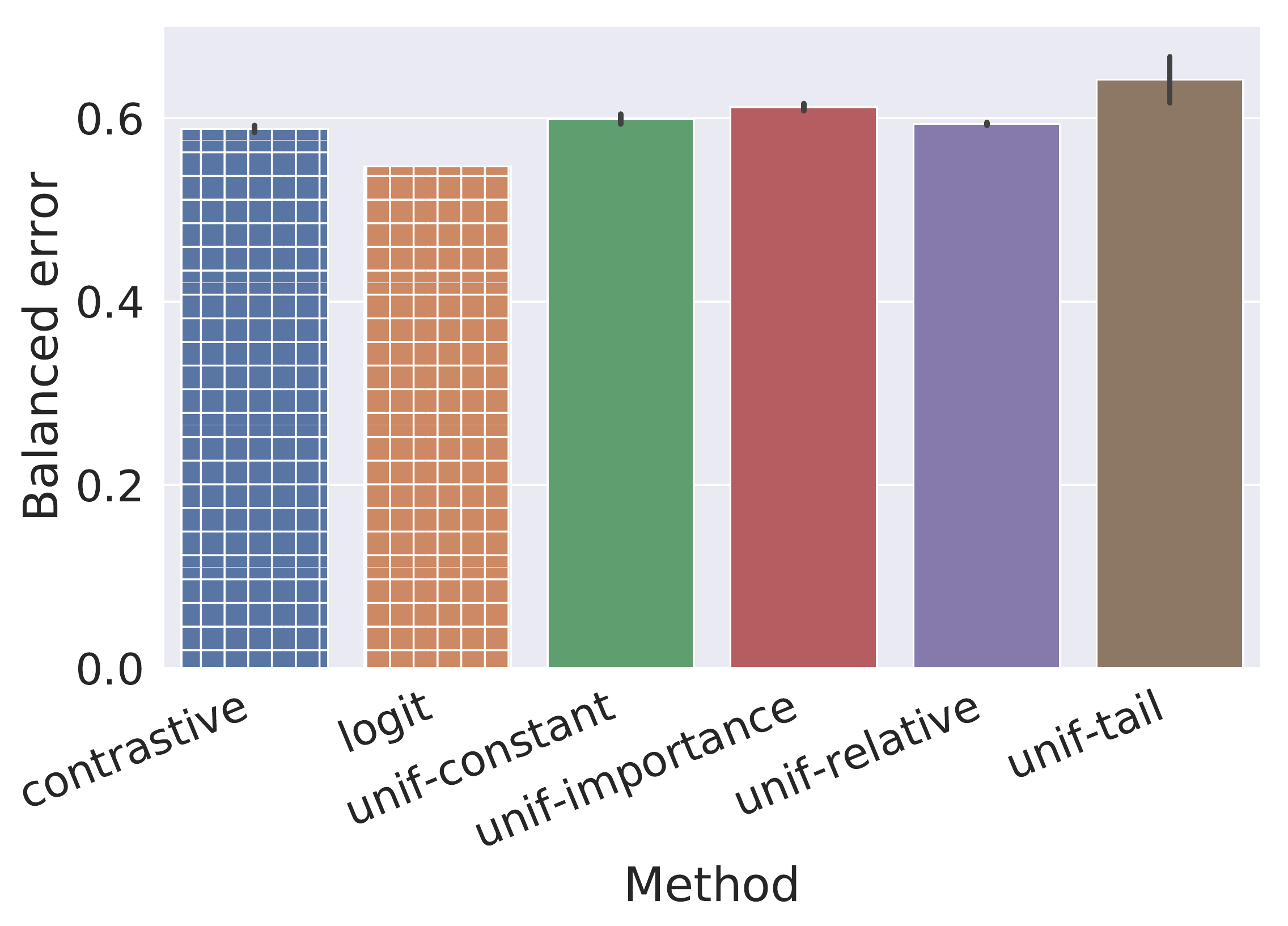}}
    \qquad
    \subcaptionbox{ImageNet-LT.}{
    \includegraphics[scale=0.25,valign=t]{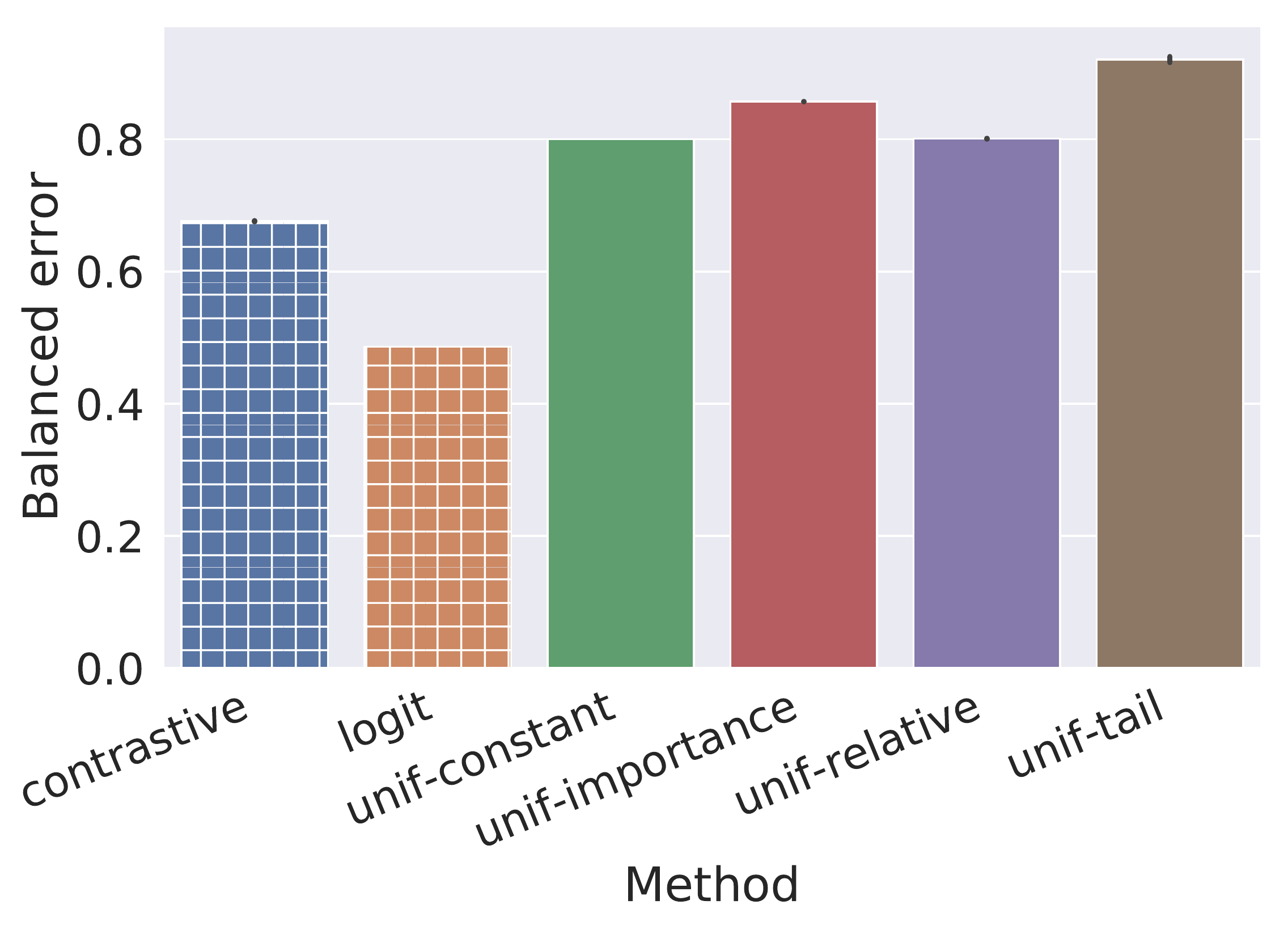}}
    }

    \resizebox{\linewidth}{!}{
    \subcaptionbox{CIFAR-100-LT Step.}{ \includegraphics[scale=0.25,valign=t]{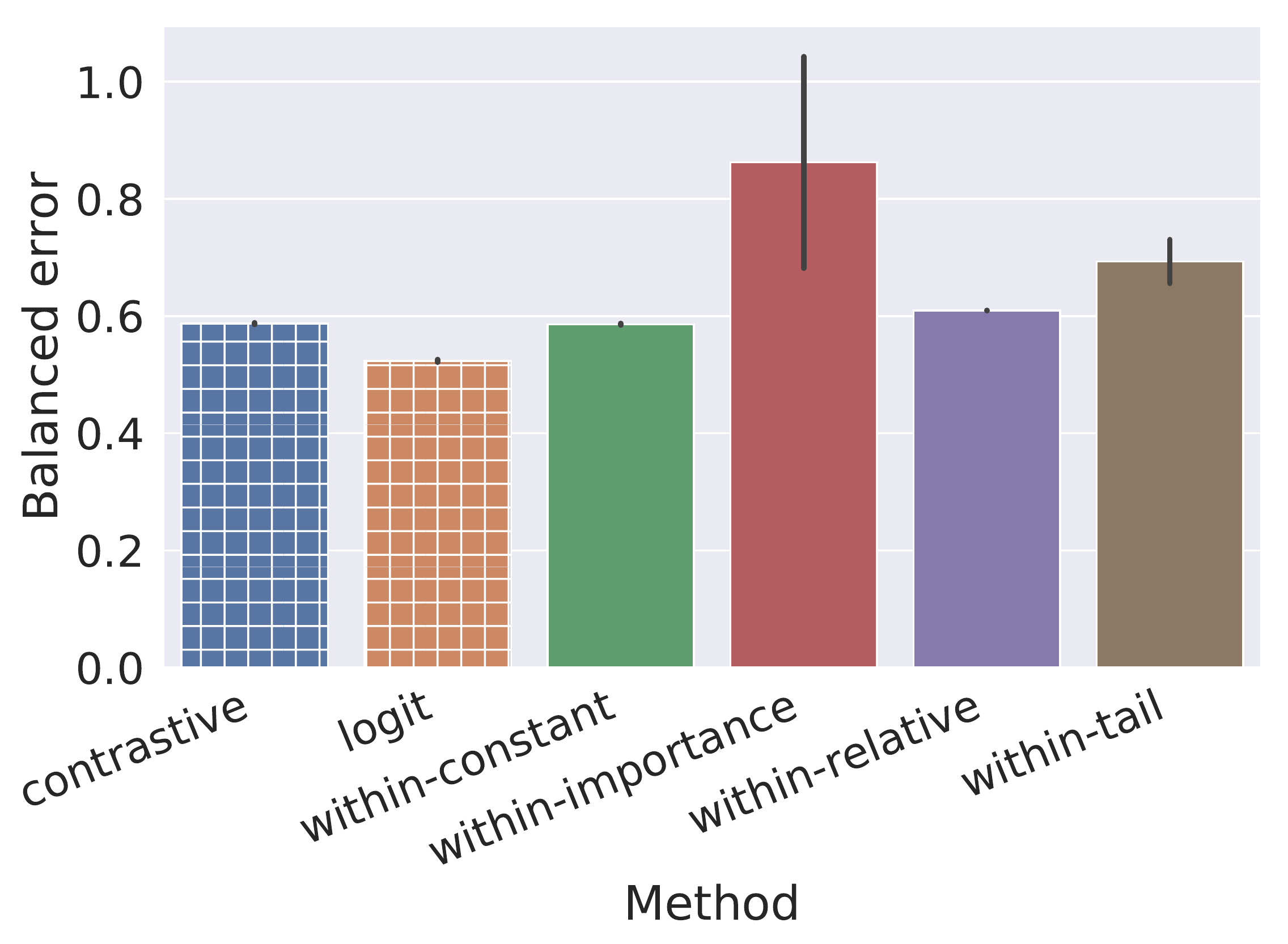}}
    \qquad
    \subcaptionbox{CIFAR-100-LT Exp.}{
    \includegraphics[scale=0.25,valign=t]{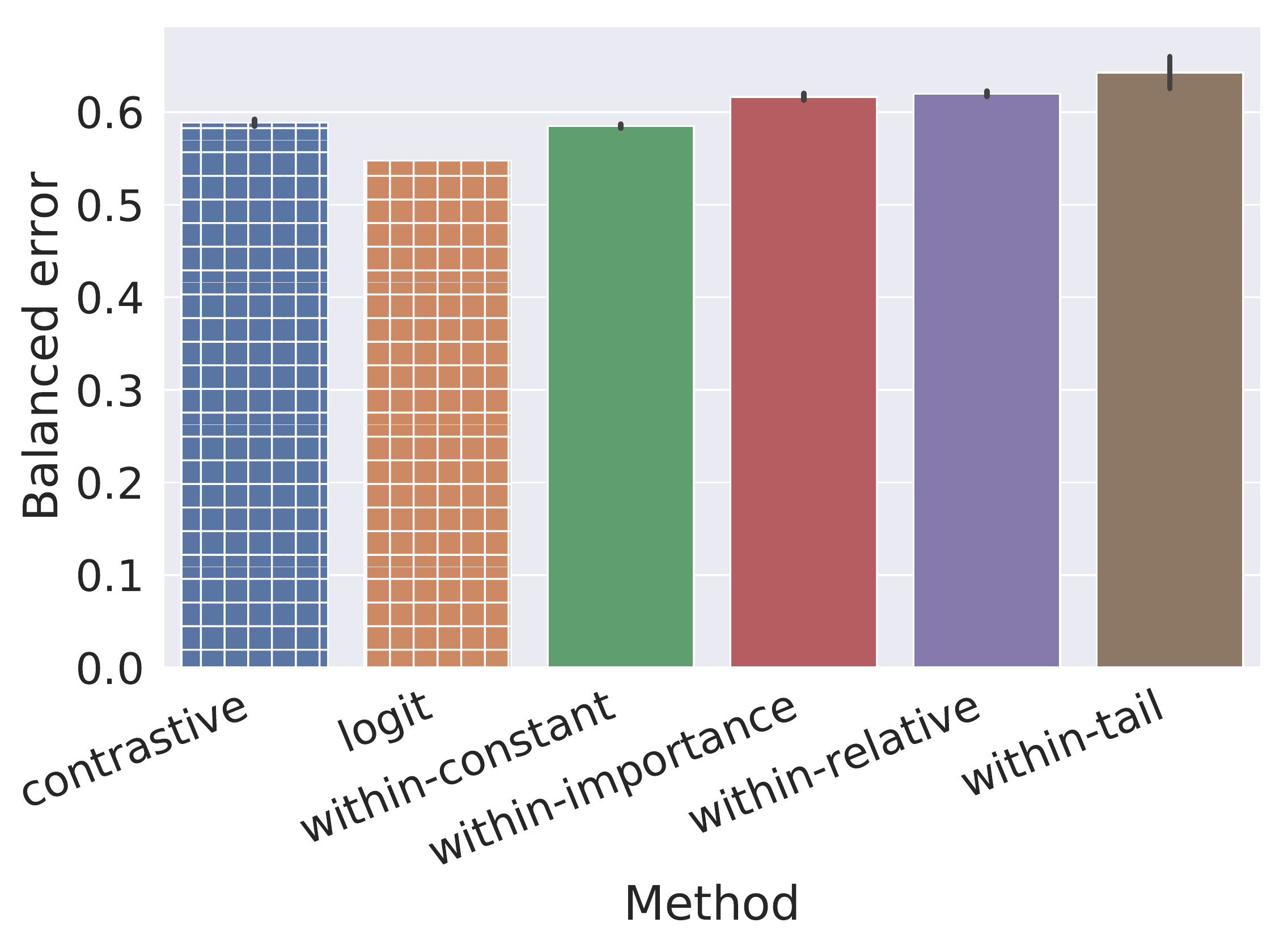}}
    \qquad
    \subcaptionbox{ImageNet-LT.}{
    \includegraphics[scale=0.25,valign=t]{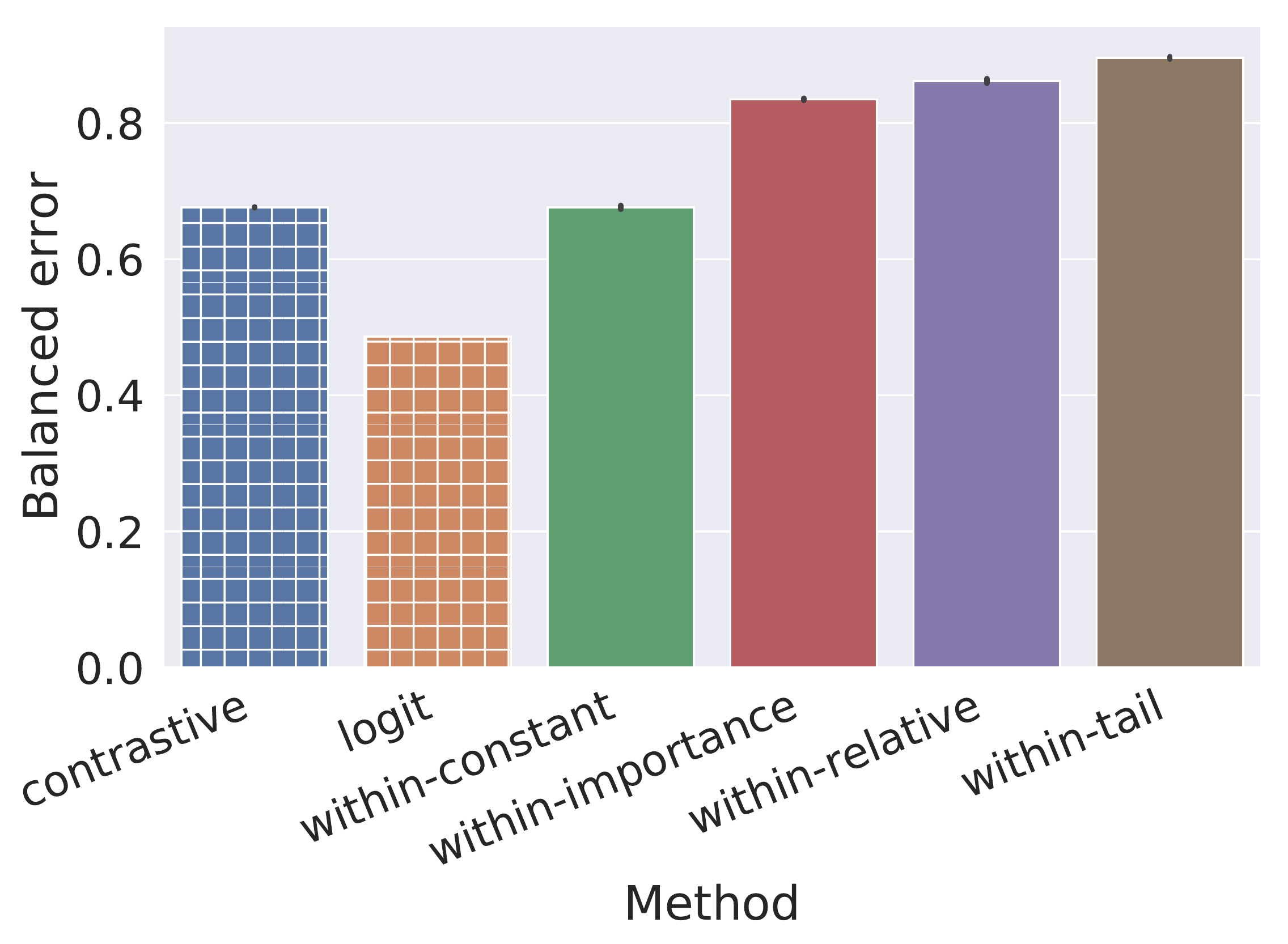}}
    }

    \caption{Balanced error
    on long-tail learning benchmarks
    using the contrastive loss.
    We present results for within-batch (${\tt within}$) and uniform
    (${\tt unif}$)
    negative sampling ,
    using the constant weight (${\tt const}$), importance weighting (${\tt importance}$), and relative weighting (${\tt relative}$) schemes from Table~\ref{tbl:summary-decoupled}.}
    \label{fig:lt_benchmarks_ber_contrastiv_e_unif}
\end{figure*}

\subsection{Results with varying number of sampled negatives}

We present results where the number of sampled negatives varies from $\{ 32, 64, 128, 256 \}$ on ImageNet-LT, using the softmax cross-entropy.
Figure~\ref{fig:lt_benchmarks_nneg} shows that with fewer sampled negatives, performance tends to slightly degrade, as expected.
However, even with a modest number of negatives, the general trends seen in the body are reflected.

\begin{figure*}[!hpt]
    \centering

    \resizebox{\linewidth}{!}{
    \subcaptionbox{ImageNet-LT $m = 32$.}{
    \includegraphics[scale=0.25,valign=t]{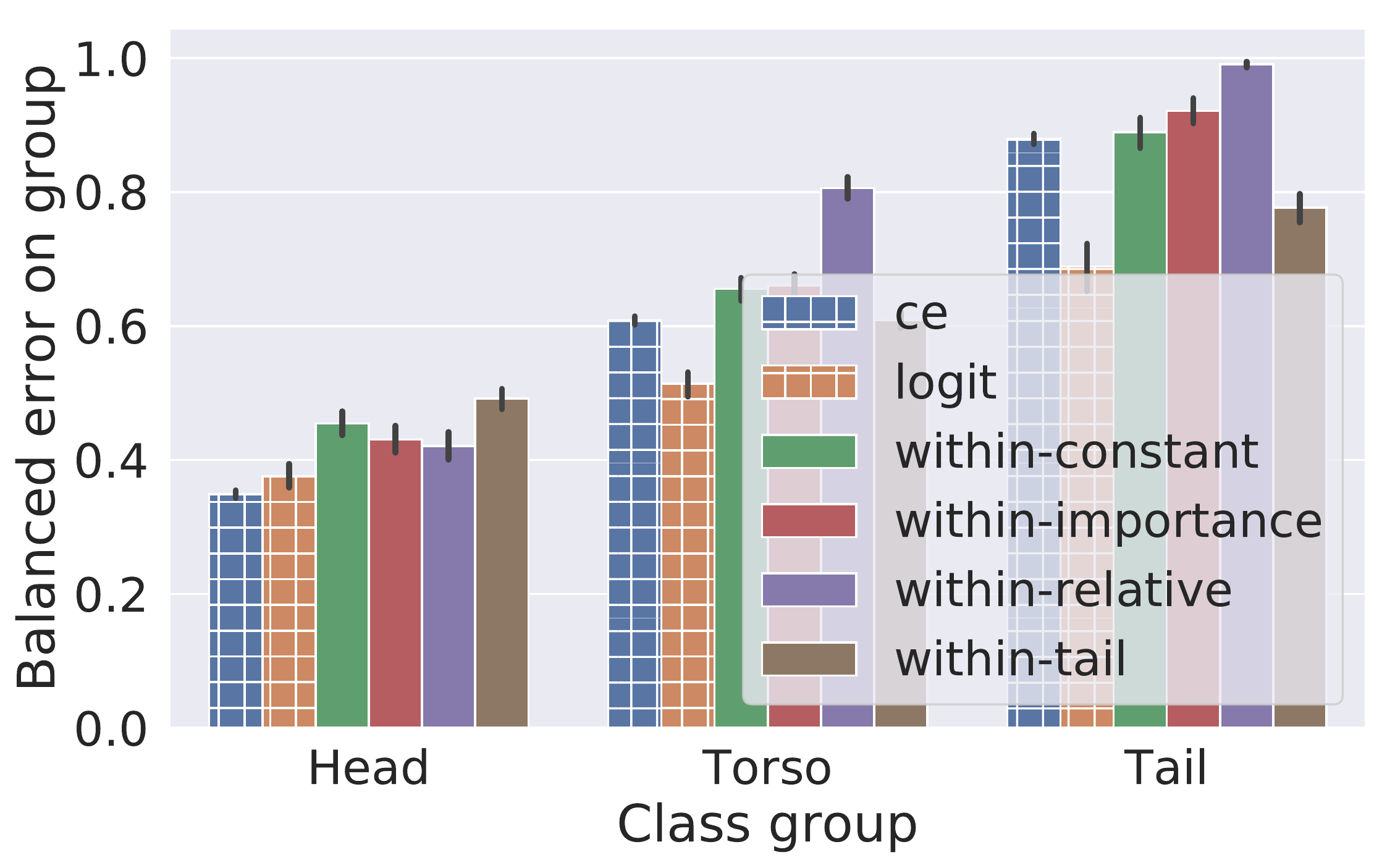}
    }
    \qquad
    \subcaptionbox{ImageNet-LT $m = 64$.}{
    \includegraphics[scale=0.25,valign=t]{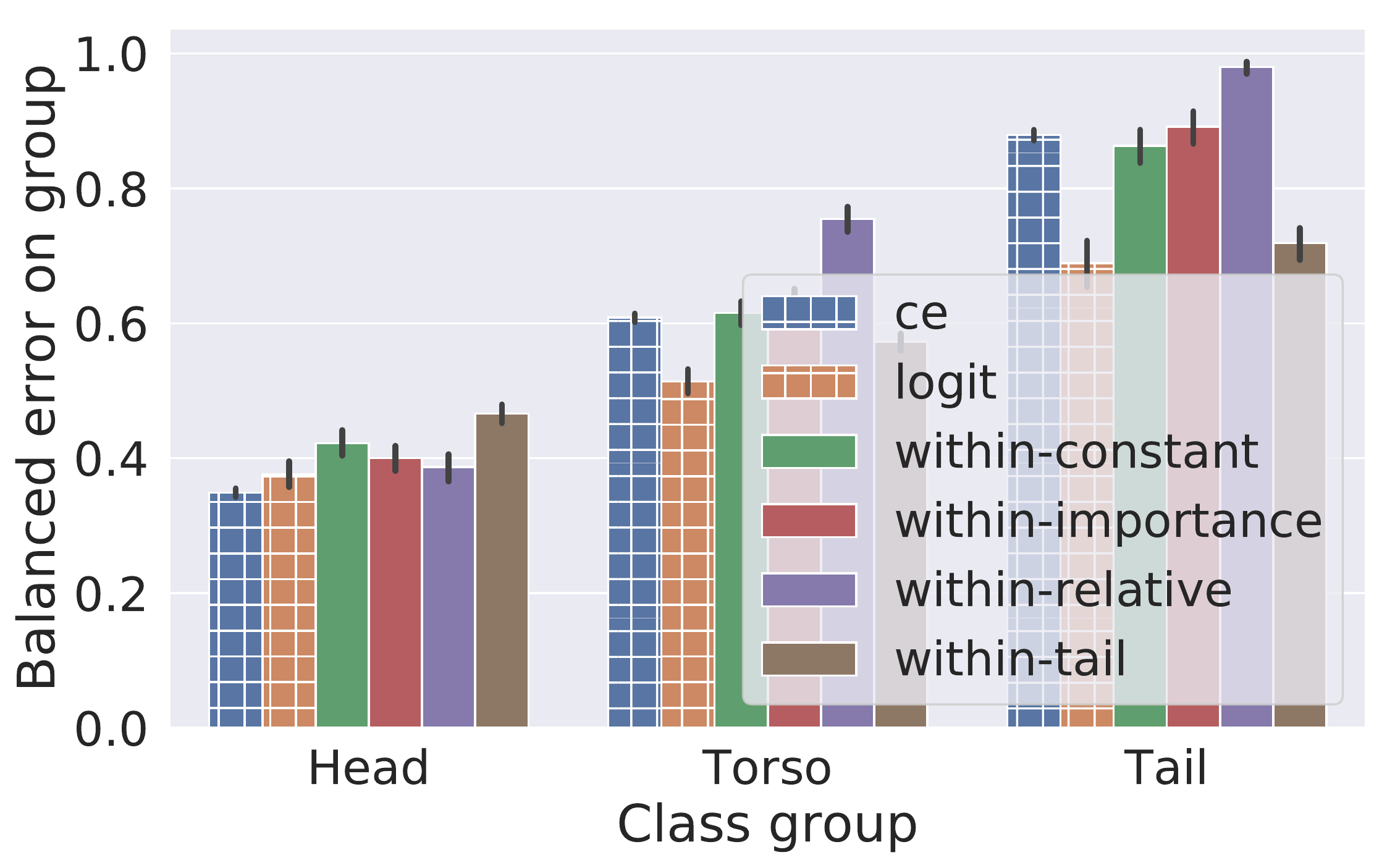}
    }
    \qquad
    \subcaptionbox{ImageNet-LT $m = 128$.}{
    \includegraphics[scale=0.25,valign=t]{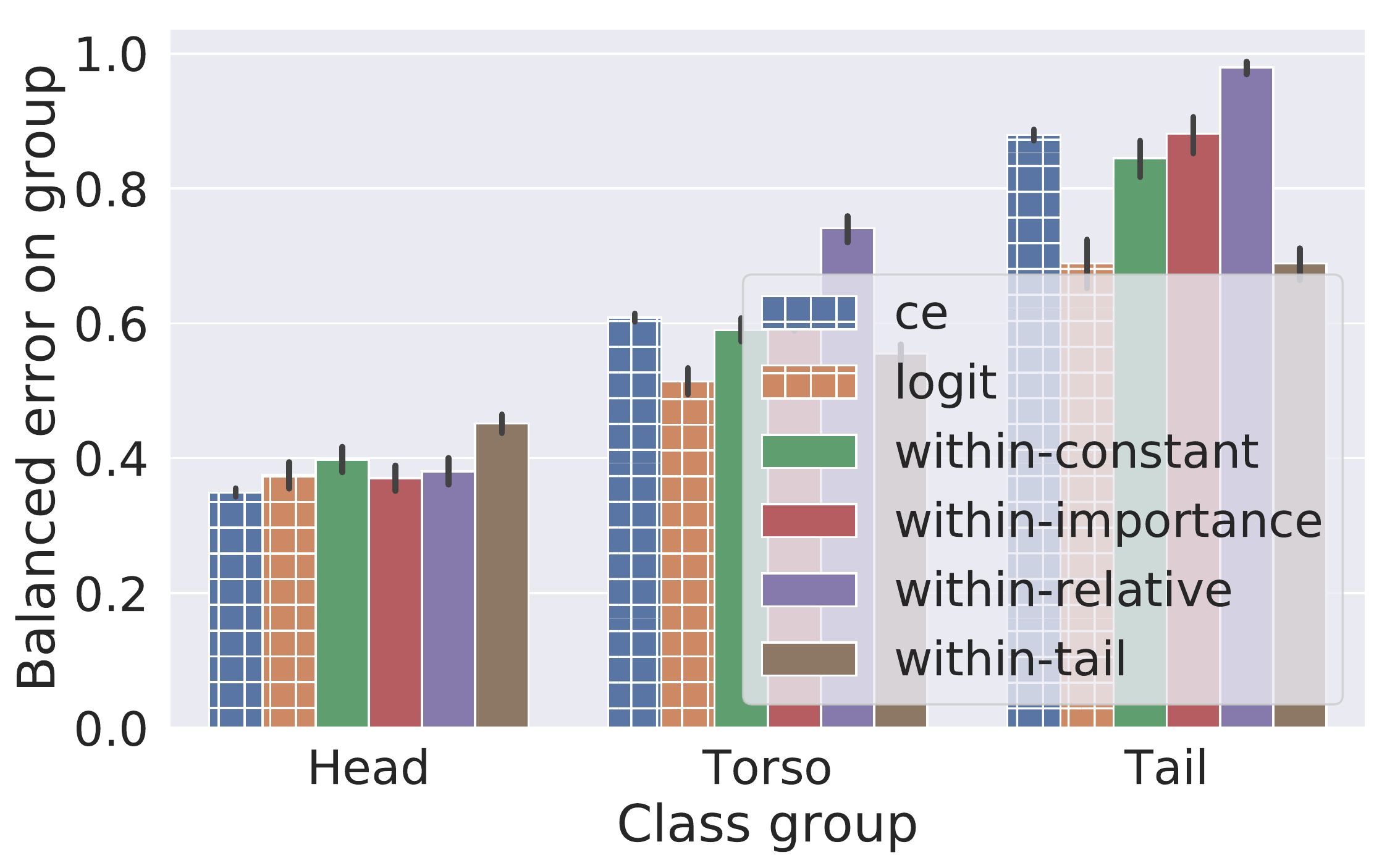}
    }
    \qquad
    \subcaptionbox{ImageNet-LT $m = 256$.}{
    \includegraphics[scale=0.25,valign=t]{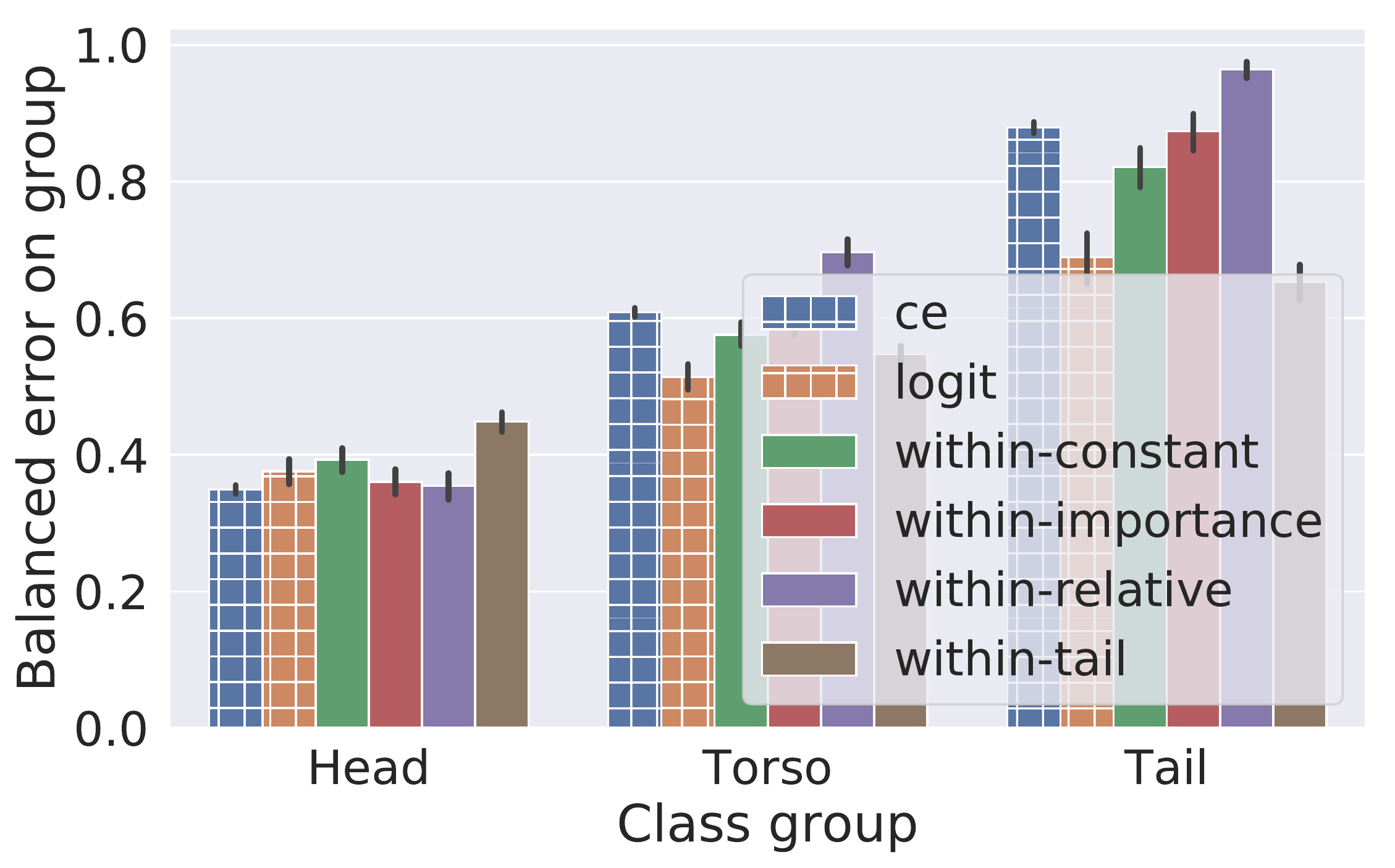}
    }
    }

    \resizebox{\linewidth}{!}{
    \subcaptionbox{ImageNet-LT $m = 32$.}{
    \includegraphics[scale=0.25,valign=t]{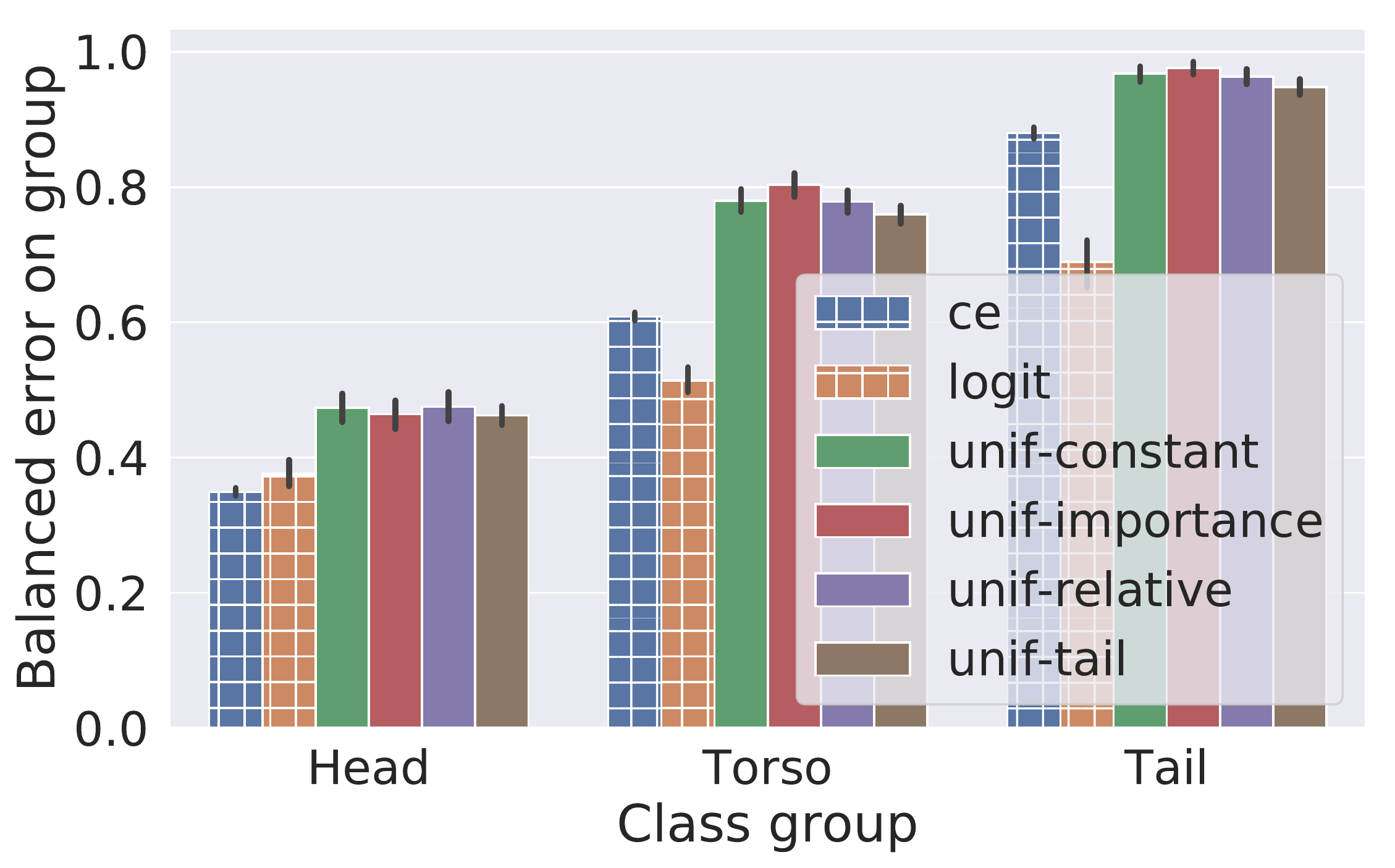}
    }
    \qquad
    \subcaptionbox{ImageNet-LT $m = 64$.}{
    \includegraphics[scale=0.25,valign=t]{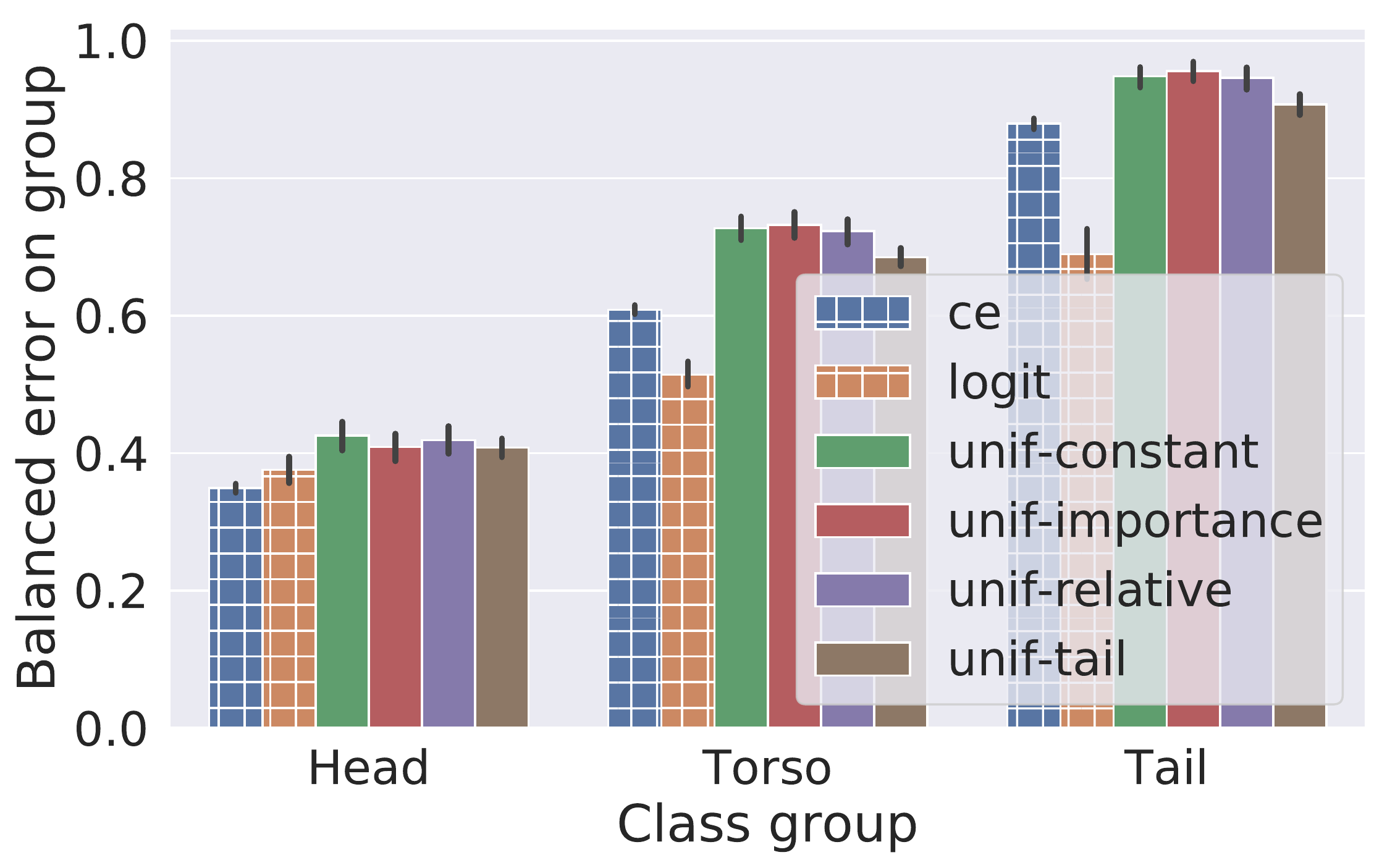}
    }
    \qquad
    \subcaptionbox{ImageNet-LT $m = 128$.}{
    \includegraphics[scale=0.25,valign=t]{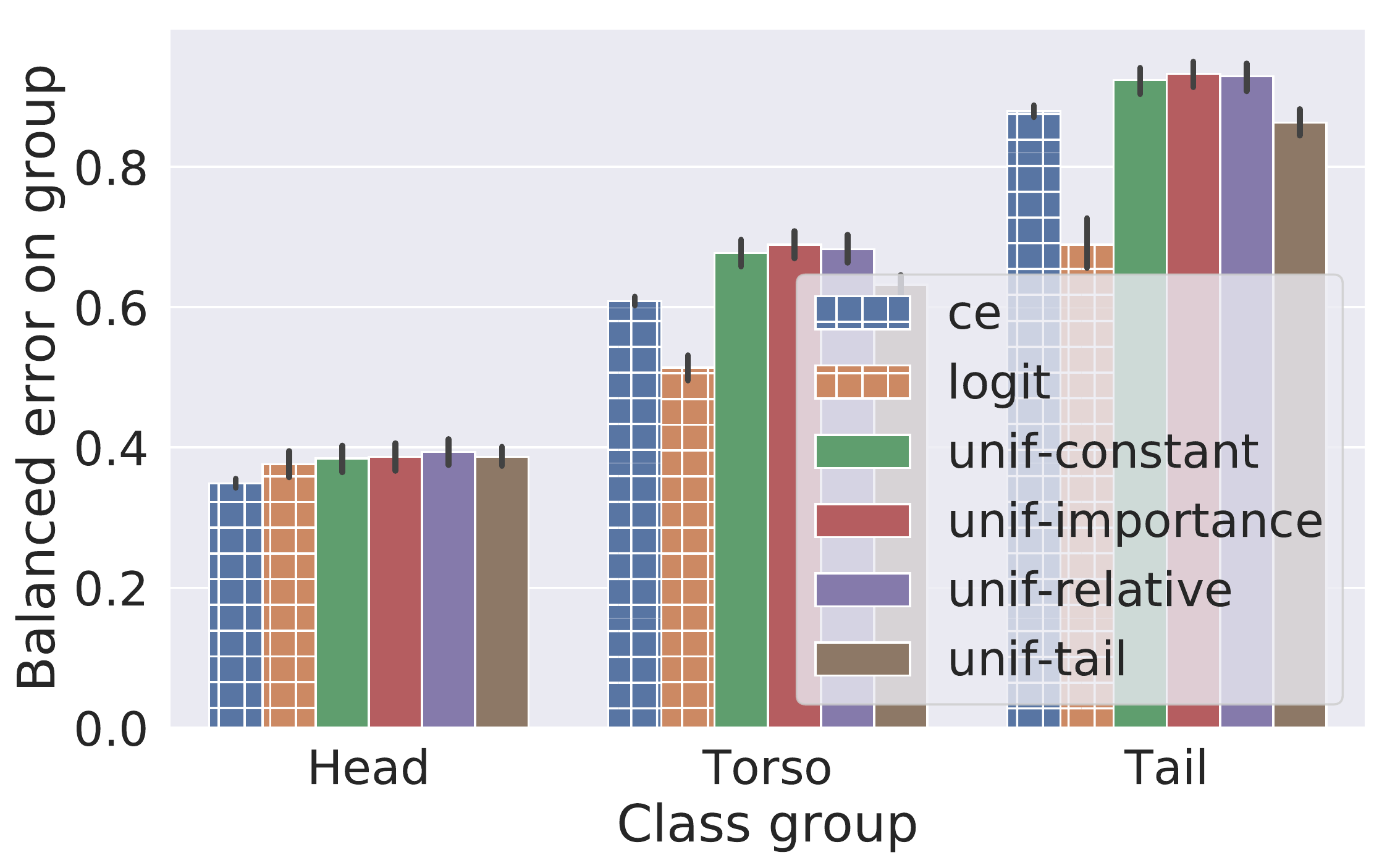}
    }
    \qquad
    \subcaptionbox{ImageNet-LT $m = 256$.}{
    \includegraphics[scale=0.25,valign=t]{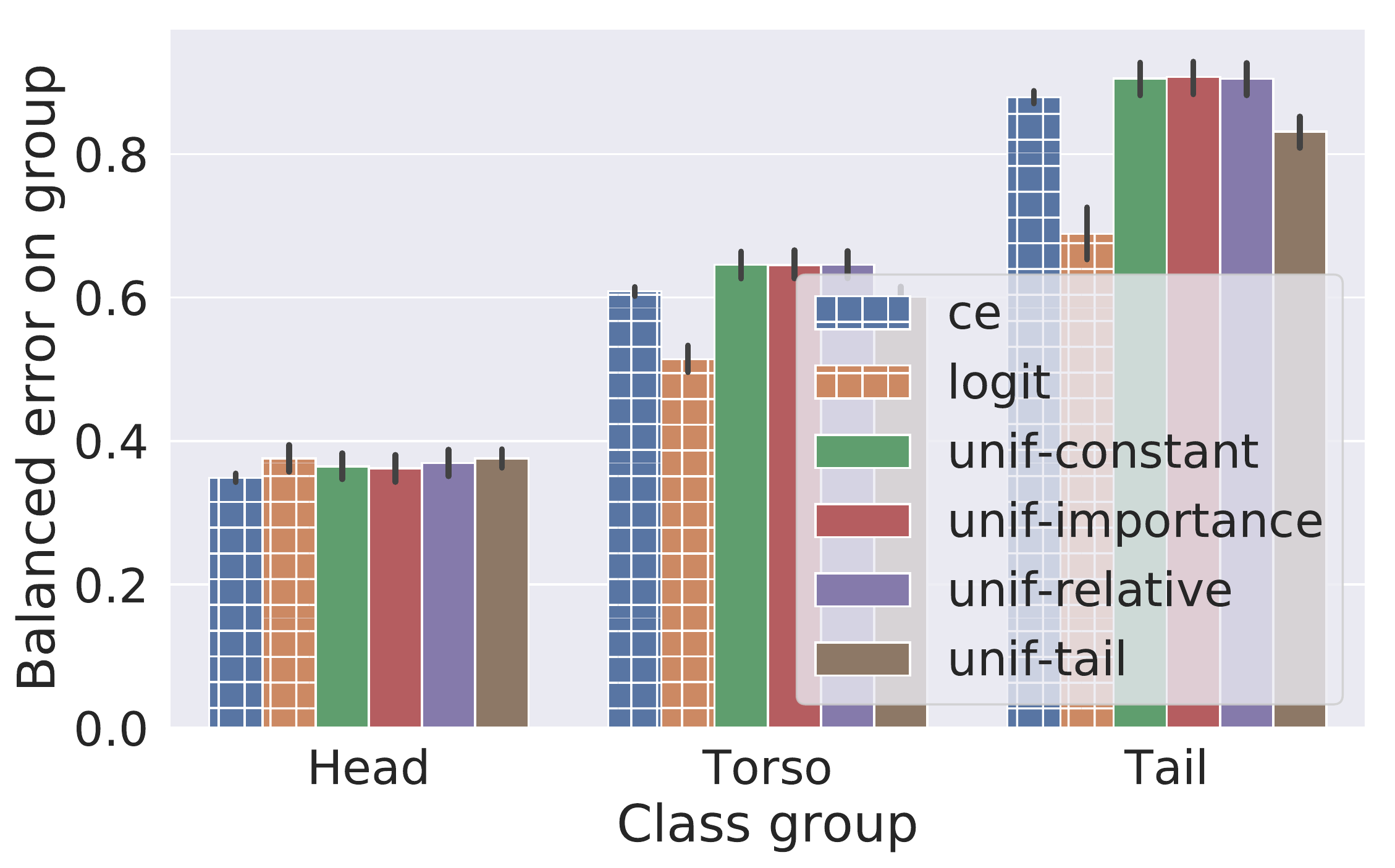}
    }
    }

    \caption{Results on head, torso and tail labels on ImageNet-LT, using varying number of sampled negatives.}
    \label{fig:lt_benchmarks_nneg}
\end{figure*}

\subsection{Results on iNaturalist 2018}
See Figure~\ref{fig:inat18_ce}.
\begin{figure*}[!pt]
    \centering

    \resizebox{\linewidth}{!}{
    \subcaptionbox{CE, Uniform}{
      \includegraphics[scale=0.25,valign=t]{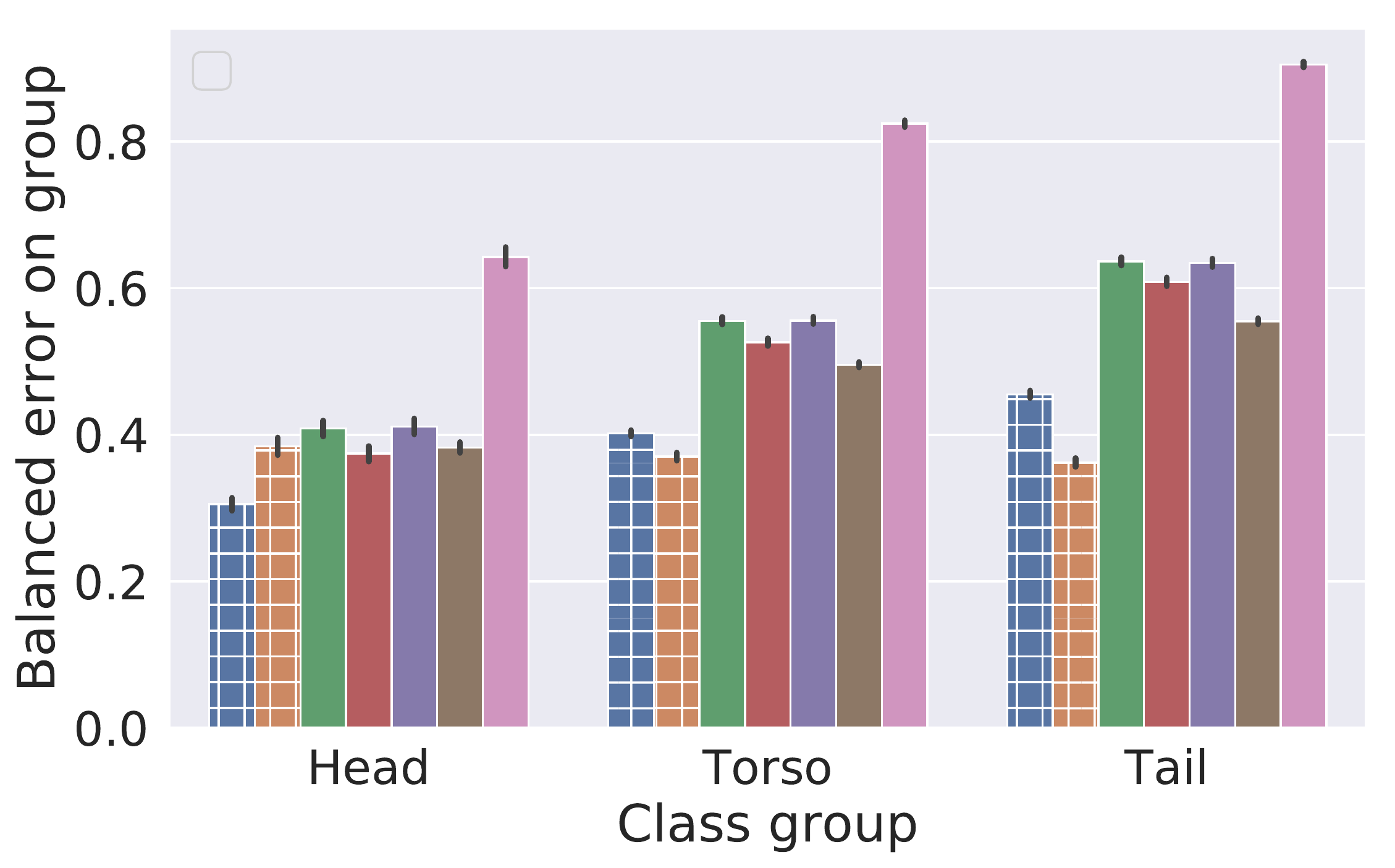}}

    \qquad
    \subcaptionbox{CE, within batch}{
      \includegraphics[scale=0.25,valign=t]{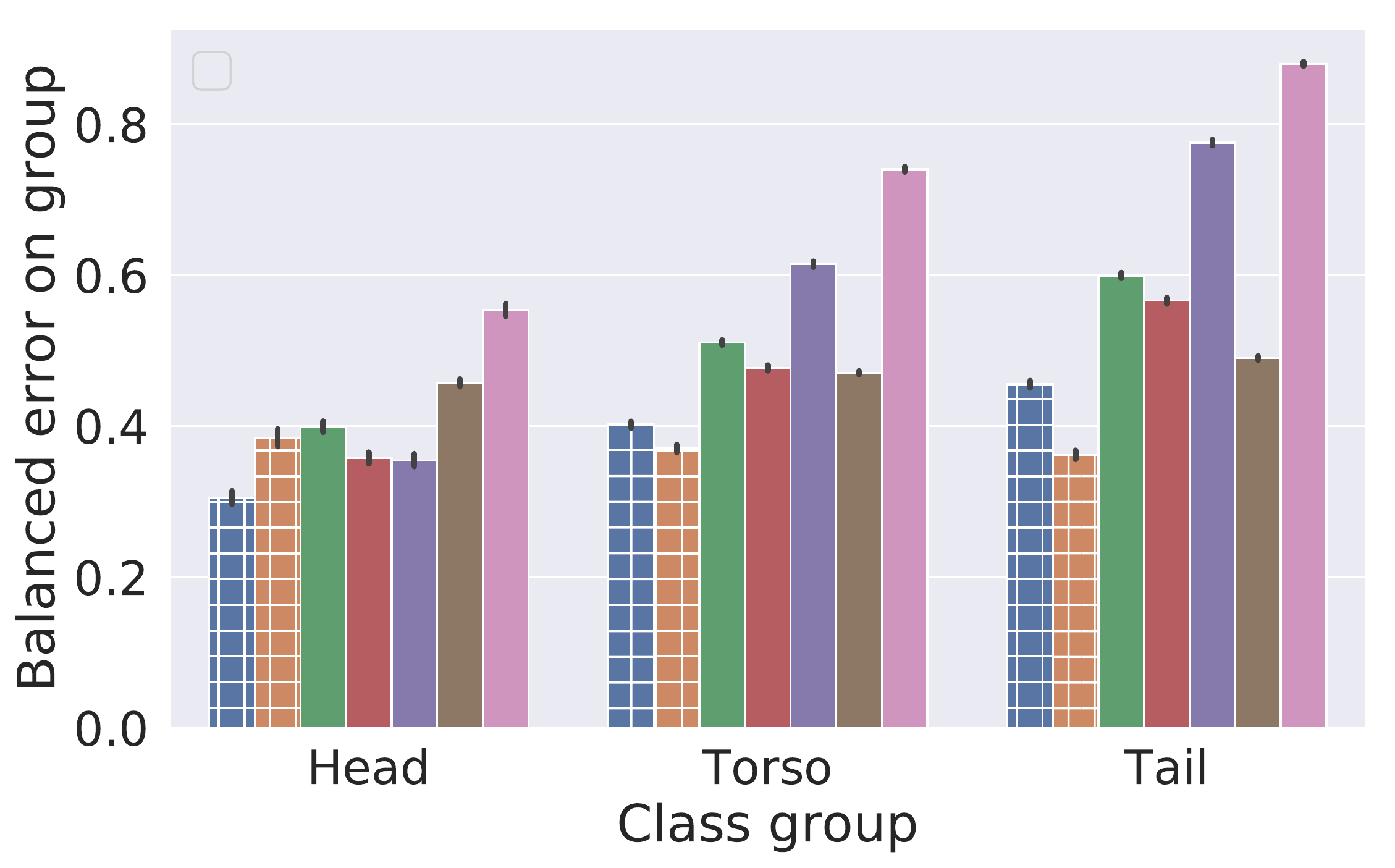}}
    }

    \resizebox{\linewidth}{!}{
    \subcaptionbox{CE. Uniform. Balanced error.}{ \includegraphics[scale=0.25,valign=t]{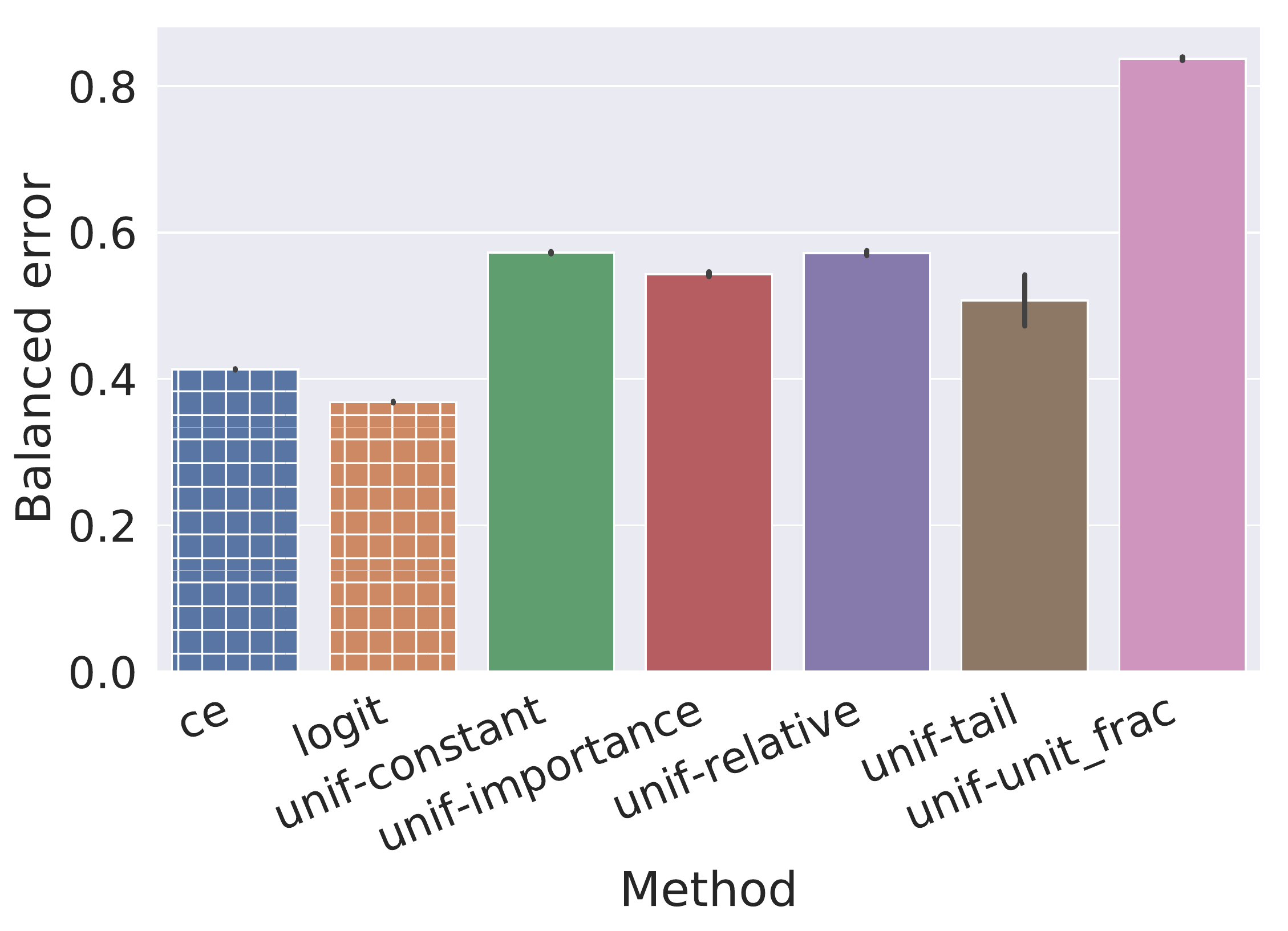}}
    \qquad
    \subcaptionbox{CE, within batch. Balanced error.}{
    \includegraphics[scale=0.25,valign=t]{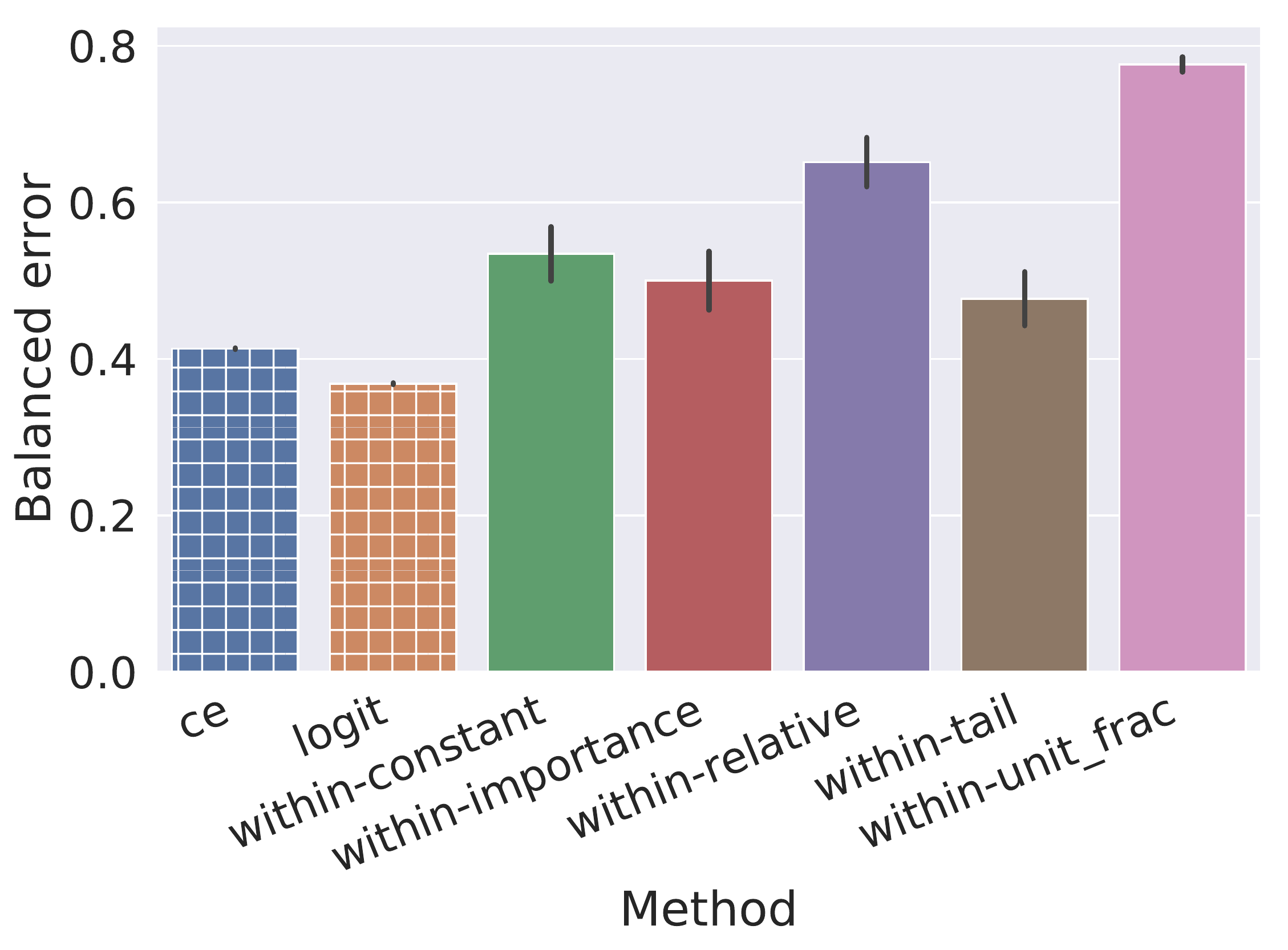}}
    }

    \caption{\textbf{(a)}, \textbf{(b)}: Results on head, torso, and tail labels on iNaturalist 2018 using uniform (${\tt unif}$) (in \textbf{(a)}) and within-batch (${\tt within}$)  (in \textbf{(b)}) negative sampling.
    \textbf{(c)}, \textbf{(d)}:
      Balanced error on iNaturalist 2018.
    We present results for uniform (${\tt unif}$) and within-batch (${\tt within}$) and     negative sampling,
    using the constant weight (${\tt const}$), importance weighting (${\tt importance}$), and relative weighting (${\tt relative}$) schemes from Table~\ref{tbl:summary-decoupled}.}
    \label{fig:inat18_ce}
\end{figure*}

\section{Retrieval datasets}
\label{appen:dataset-details}

\begin{table*}[ht!]
    \centering
    \setlength{\tabcolsep}{4pt}
    \scalebox{0.82}{
    \begin{tabular}{lllllll}
    \toprule
        {\bf Dataset}                  & {\bf \#Features} &  {\bf \#Labels} & {\bf \#Train Points} &{\bf \#Test Points}  & {\bf Average \#I/L} & {\bf Average \#L/I} \\
        \toprule
        \delicioussmall & 500 & 983 & 12920 & 3185 & 311.61 & 19.03 \\
        \amazonsmall  & 203,882     & 13,330    & 1,186,239 & 306,782 & 448.57 & 5.04\\
        \wikilshtc  & 1,617,899    & 325,056   & 1,778,351 & 587,084 & 17.46 & 3.19\\
        \bottomrule
    \end{tabular}
    }
	\caption{Summary of the extreme classification datasets used in this paper~\citep{V18}. \#I/L is the number of instances per label, and \#L/I is the number of labels per instance.}\label{tb:datasets}
    \vspace{-\baselineskip}
\end{table*}

\section{Additional results: Retrieval datasets}
\label{sec:additional-retrieval}

\begin{figure*}[ht!]
    \centering
    \small
    \resizebox{0.9\linewidth}{!}{
    \subcaptionbox{$\recall$~1.\label{fig:am-r1-u}}{
    \includegraphics[scale=0.33,valign=t]{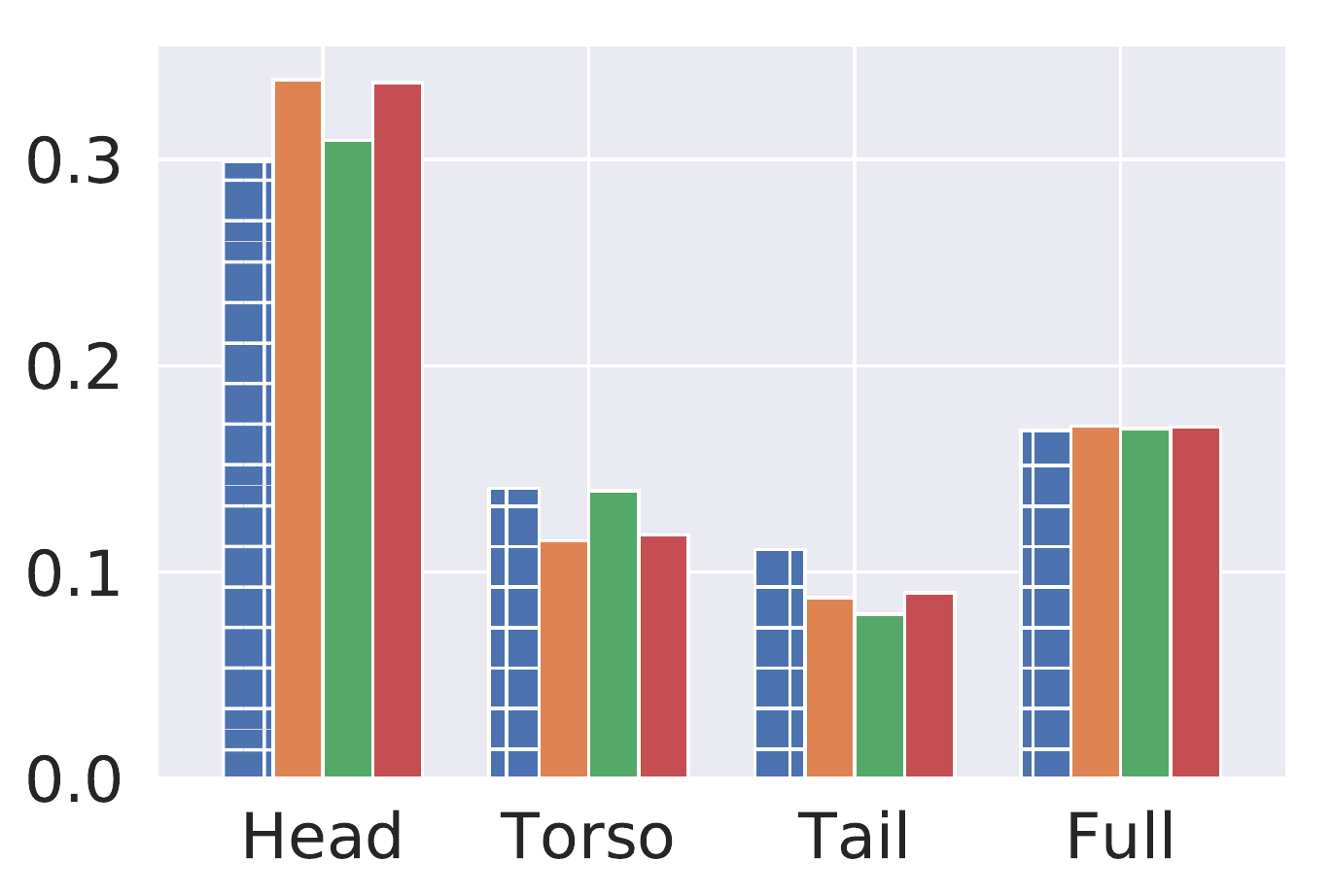}
    }
    \qquad
    \subcaptionbox{$\recall$ 10.\label{fig:am-r10-u}}{
    \includegraphics[scale=0.33,valign=t]{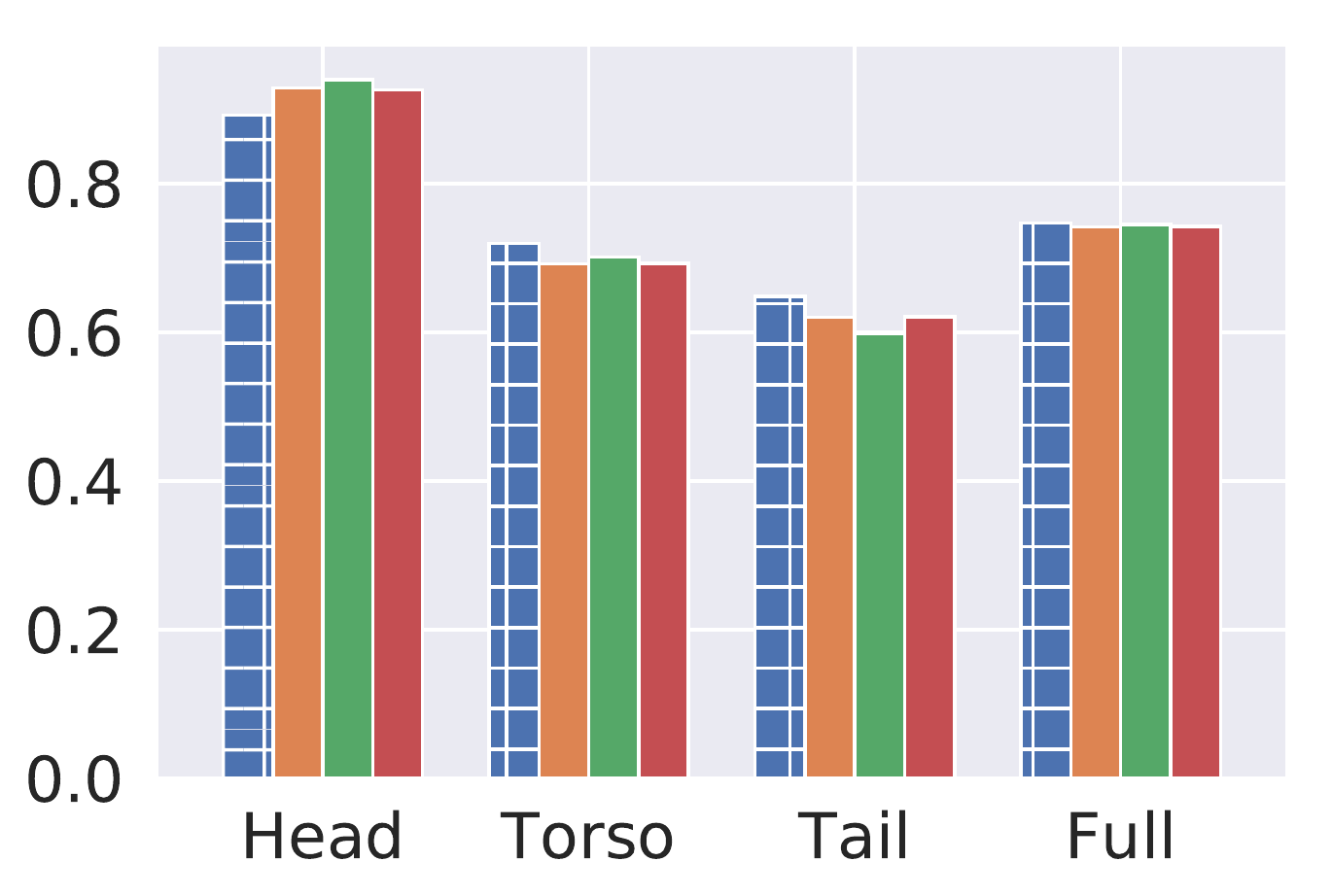}
    }
    \qquad
    \subcaptionbox{$\recall$ 50.\label{fig:am-r50-u}}{
    \includegraphics[scale=0.33,valign=t]{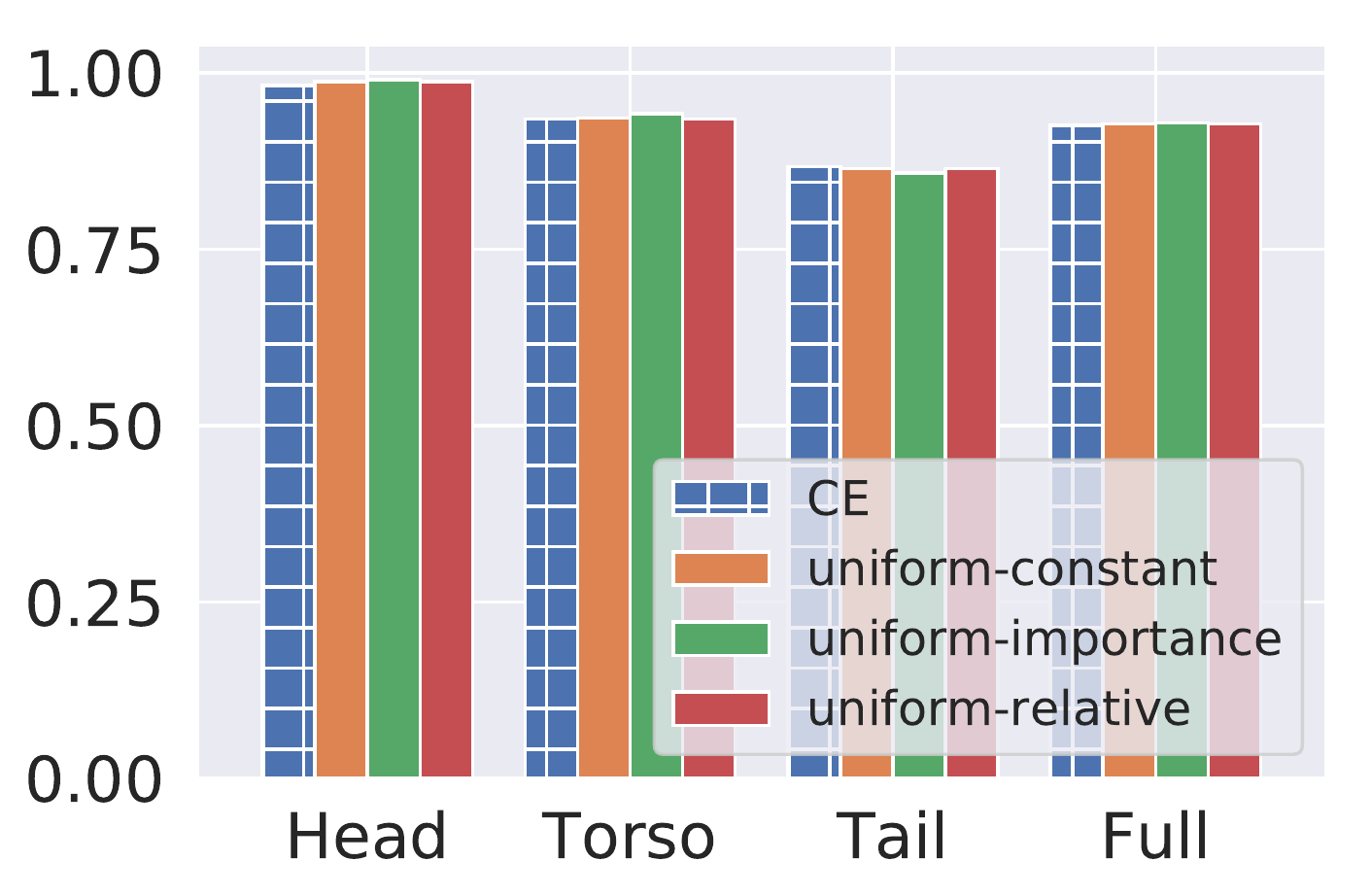}
    }
    }

    \resizebox{0.9\linewidth}{!}{
    \subcaptionbox{$\recall$ 1.\label{fig:wiki-r1-u}}{
    \includegraphics[scale=0.33,valign=t]{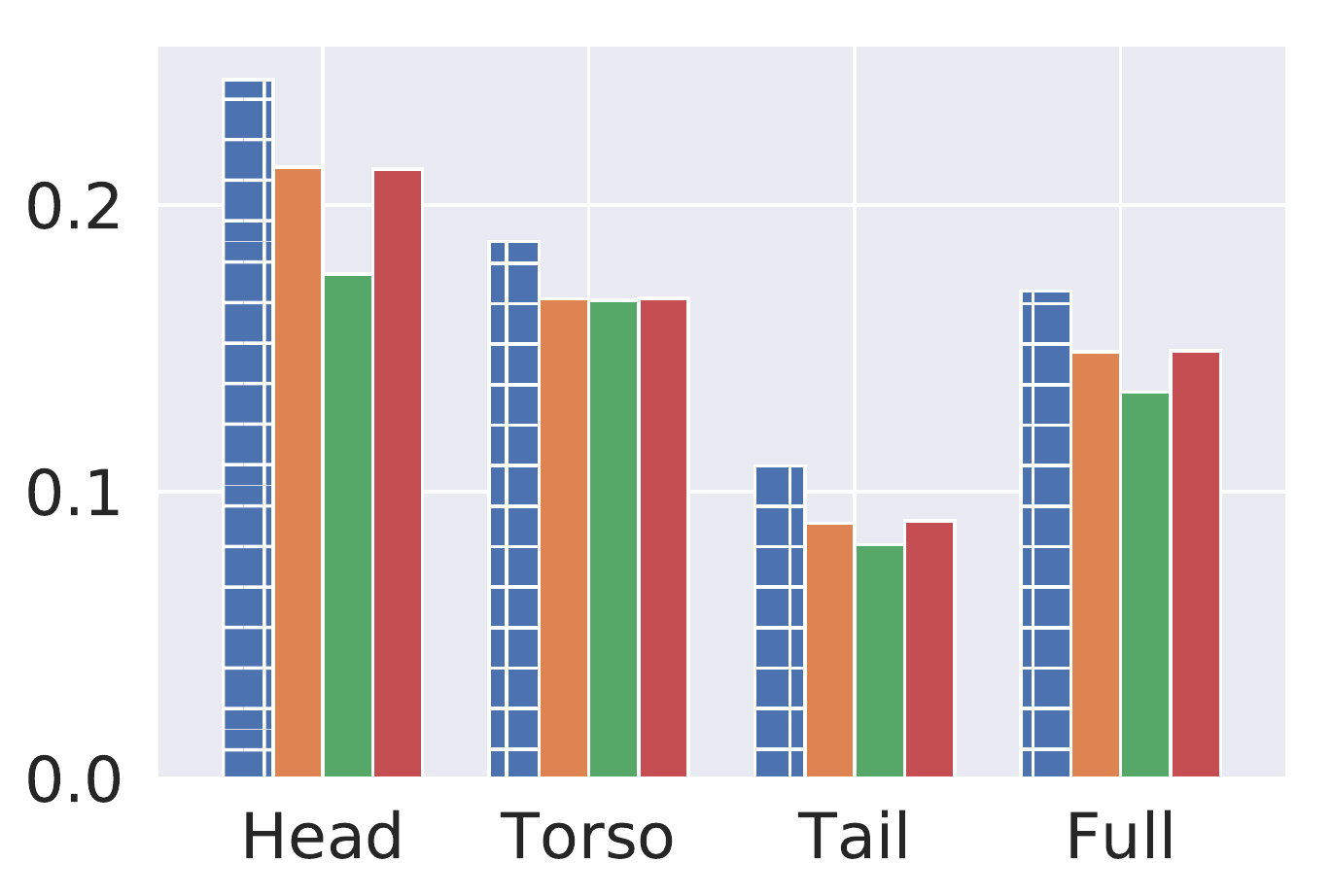}
    }
    \qquad
    \subcaptionbox{$\recall$ 10.\label{fig:wiki-r10-u}}{
    \includegraphics[scale=0.33,valign=t]{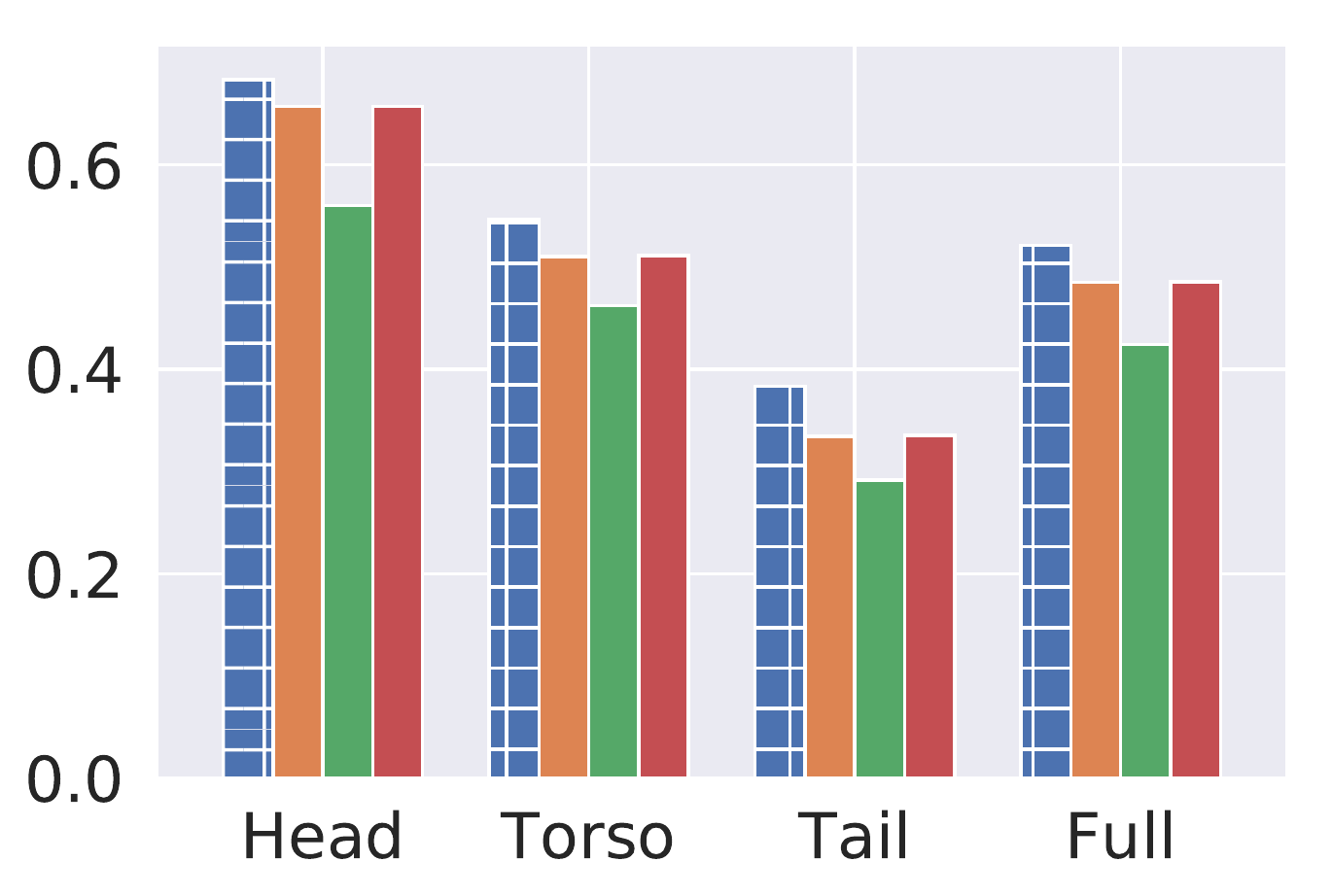}
    }
    \qquad
    \subcaptionbox{$\recall$ 50.\label{fig:wiki-r50-u}}{
    \includegraphics[scale=0.33,valign=t]{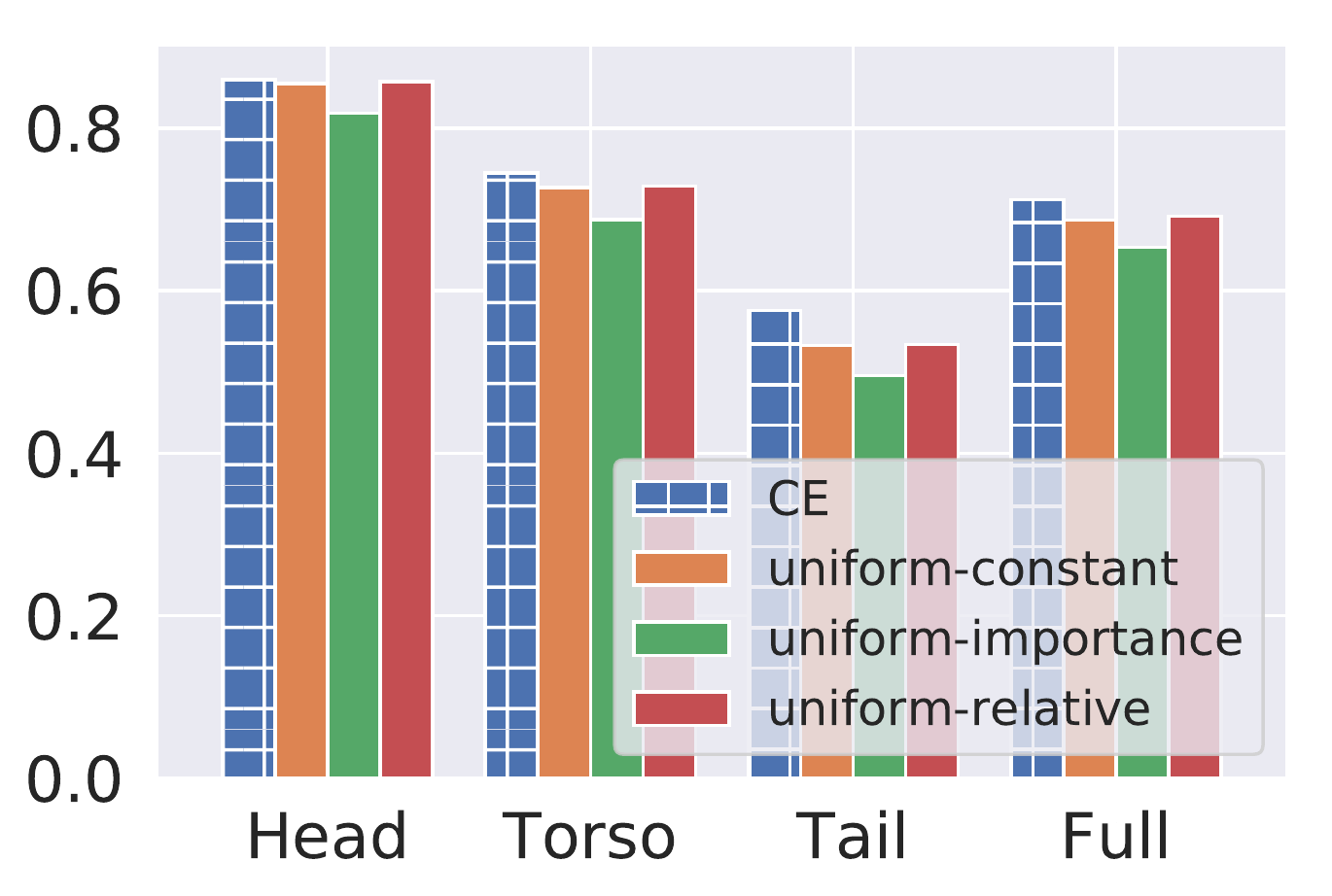}
    }
    }

    \resizebox{0.9\linewidth}{!}{
    \subcaptionbox{$\recall$ 1.\label{fig:deli-r1-u}}{
    \includegraphics[scale=0.33,valign=t]{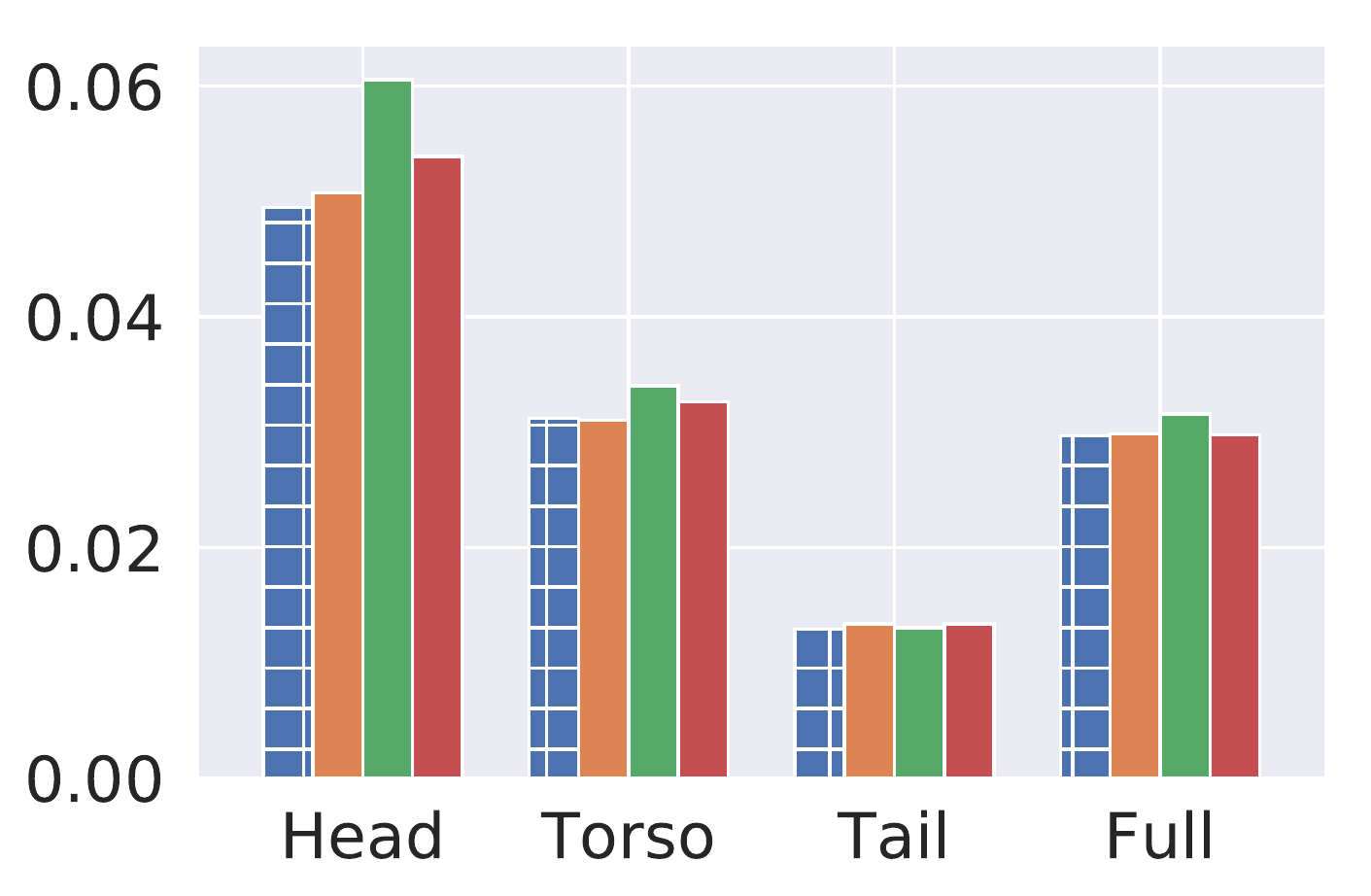}
    }
    \qquad
    \subcaptionbox{$\recall$ 10.\label{fig:deli-r10-u}}{
    \includegraphics[scale=0.33,valign=t]{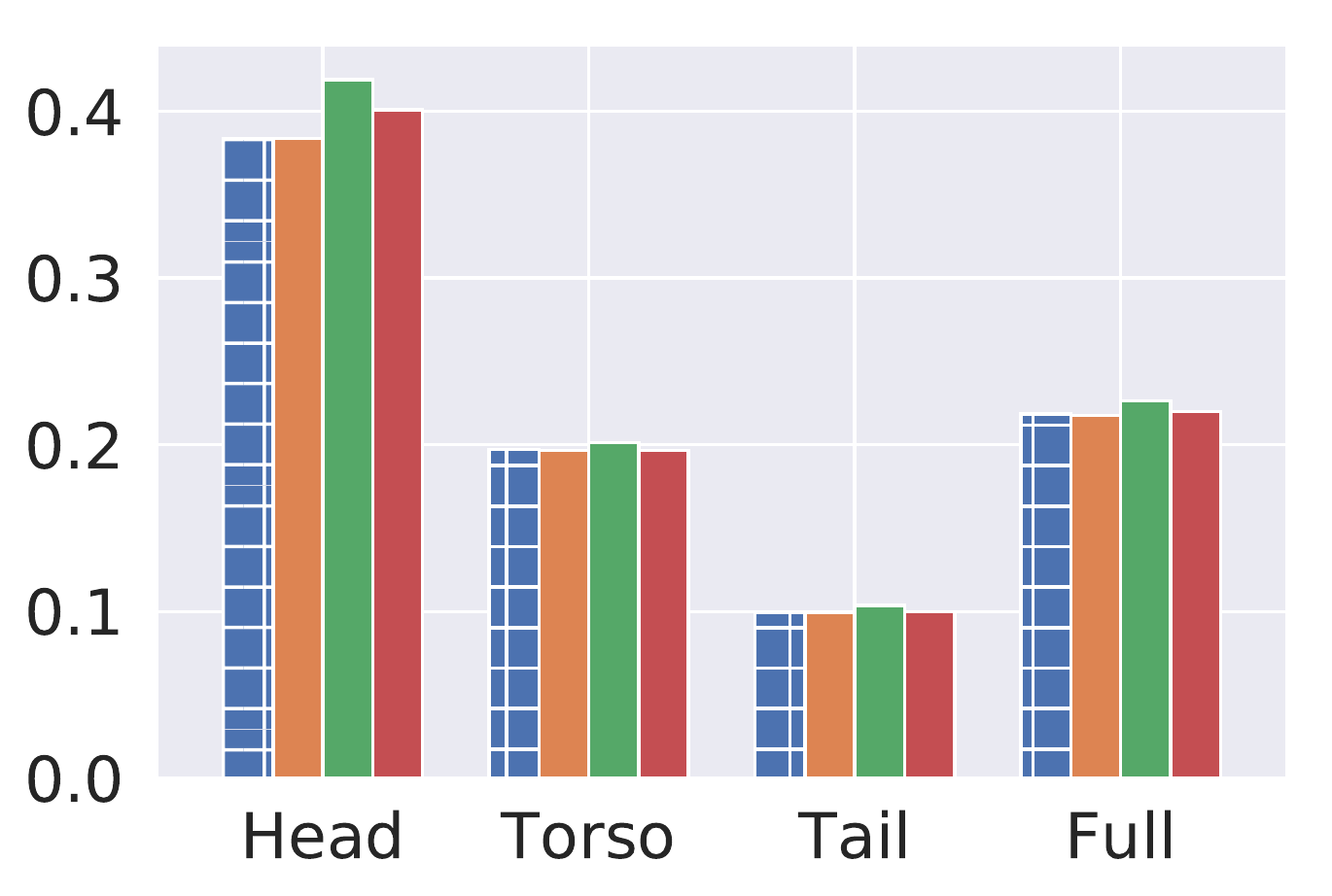}
    }
    \qquad
    \subcaptionbox{$\recall$ 50.\label{fig:deli-r50-u}}{
    \includegraphics[scale=0.33,valign=t]{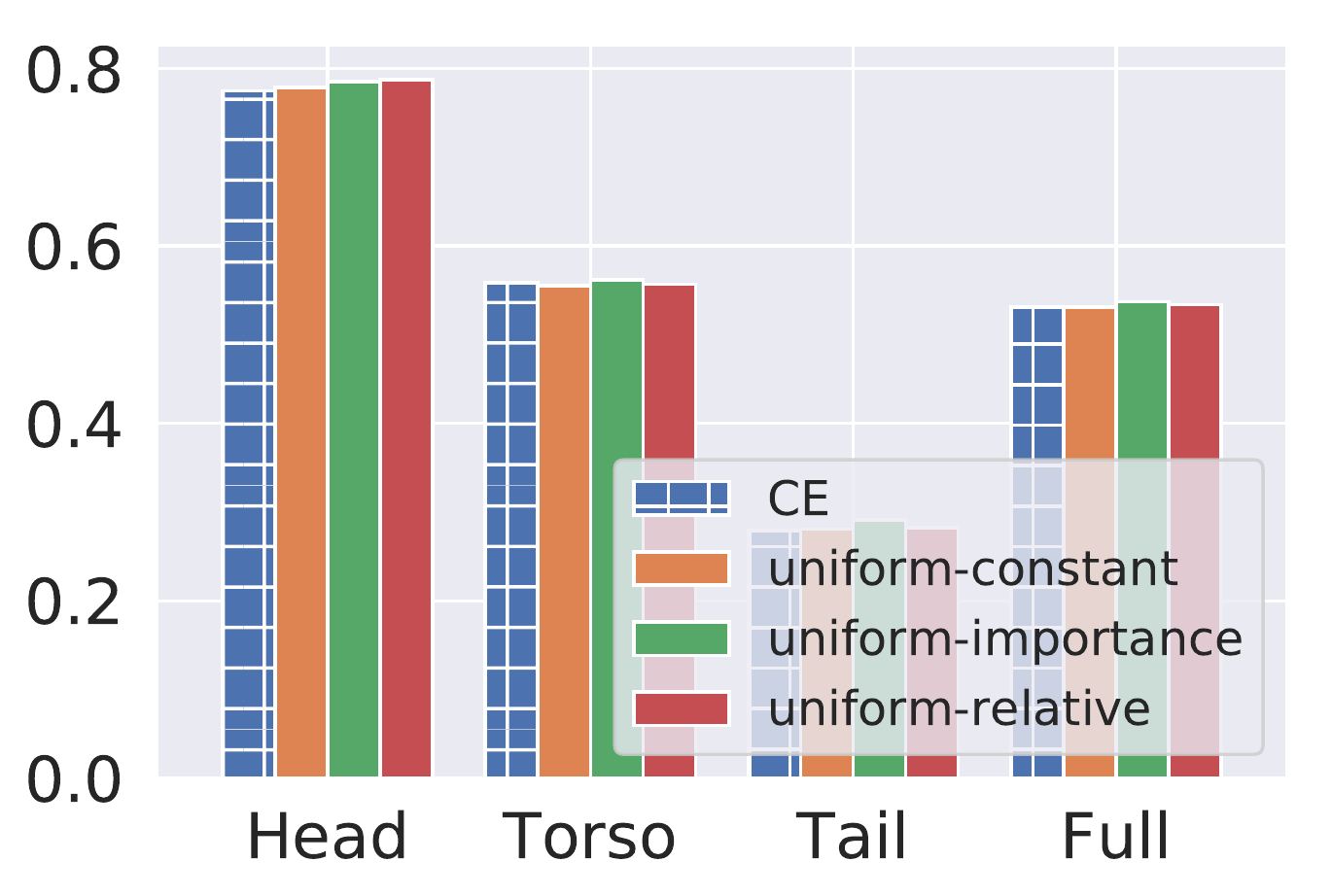}
    }
    }
    \caption{Performance of uniform negative sampling based cross-entropy loss (cf.~\eqref{eqn:weighted-sampled-softmax}) on \amazonsmall~(Figure~\ref{fig:am-r1-u} - \ref{fig:am-r50-u}), \wikilshtc~(Fig.~\ref{fig:wiki-r1-u} - \ref{fig:wiki-r50-u}), and \delicioussmall~(Fig.~\ref{fig:deli-r1-u} - \ref{fig:deli-r50-u}). These experiments utilize $m=256$ negative for \amazonsmall{} and \wikilshtc{}, and $m = 64$ negatives for \delicioussmall{}. We report the performance on three subpopulations (Head, Torso, and Tail) and the entire test set (Full), as measured by $\recall{k}$ for $k$ = 1, 10, and 50. We combine uniform sampling with constant, importance, and relative weighting schemes. For reference, we include the results of standard softmax cross-entropy loss (${\tt ce}$). Note that for the retrieval datasets, the uniform sampling aligns with ${\tt ce}$ as it consistently focuses on Head, Torso, and Tail in that order. This is in contrast with the within-batch sampling (cf.~Figure~\ref{fig:rb}).}
    \label{fig:rb-unif}
\end{figure*}

\begin{figure*}[ht!]
    \centering
    \small

    \resizebox{0.9\linewidth}{!}{
    \subcaptionbox{\amazonsmall.\label{fig:am-p-w}}{
    \includegraphics[scale=0.33,valign=t]{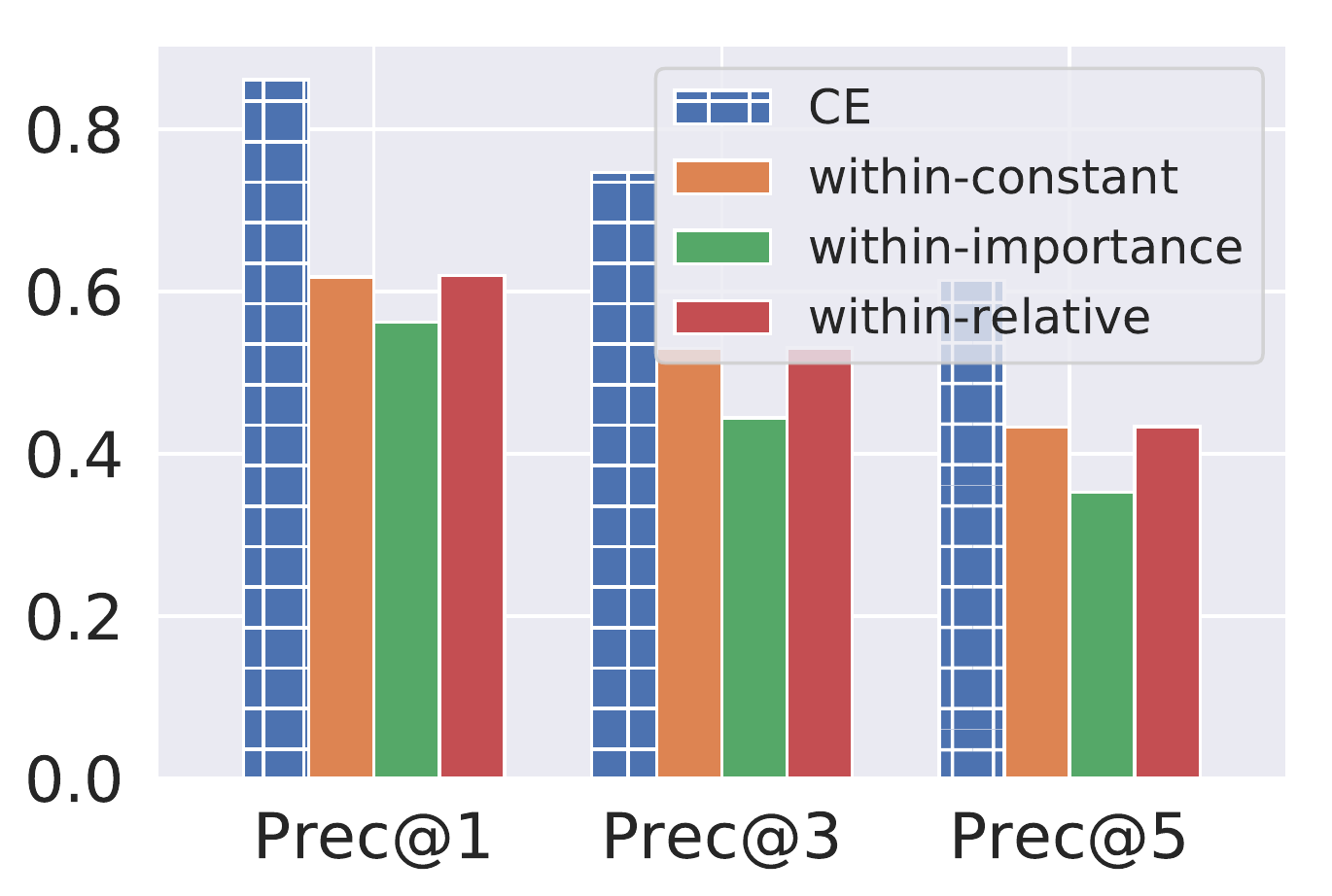}
    }
    \qquad
    \subcaptionbox{\wikilshtc.\label{fig:wiki-p-w}}{
    \includegraphics[scale=0.33,valign=t]{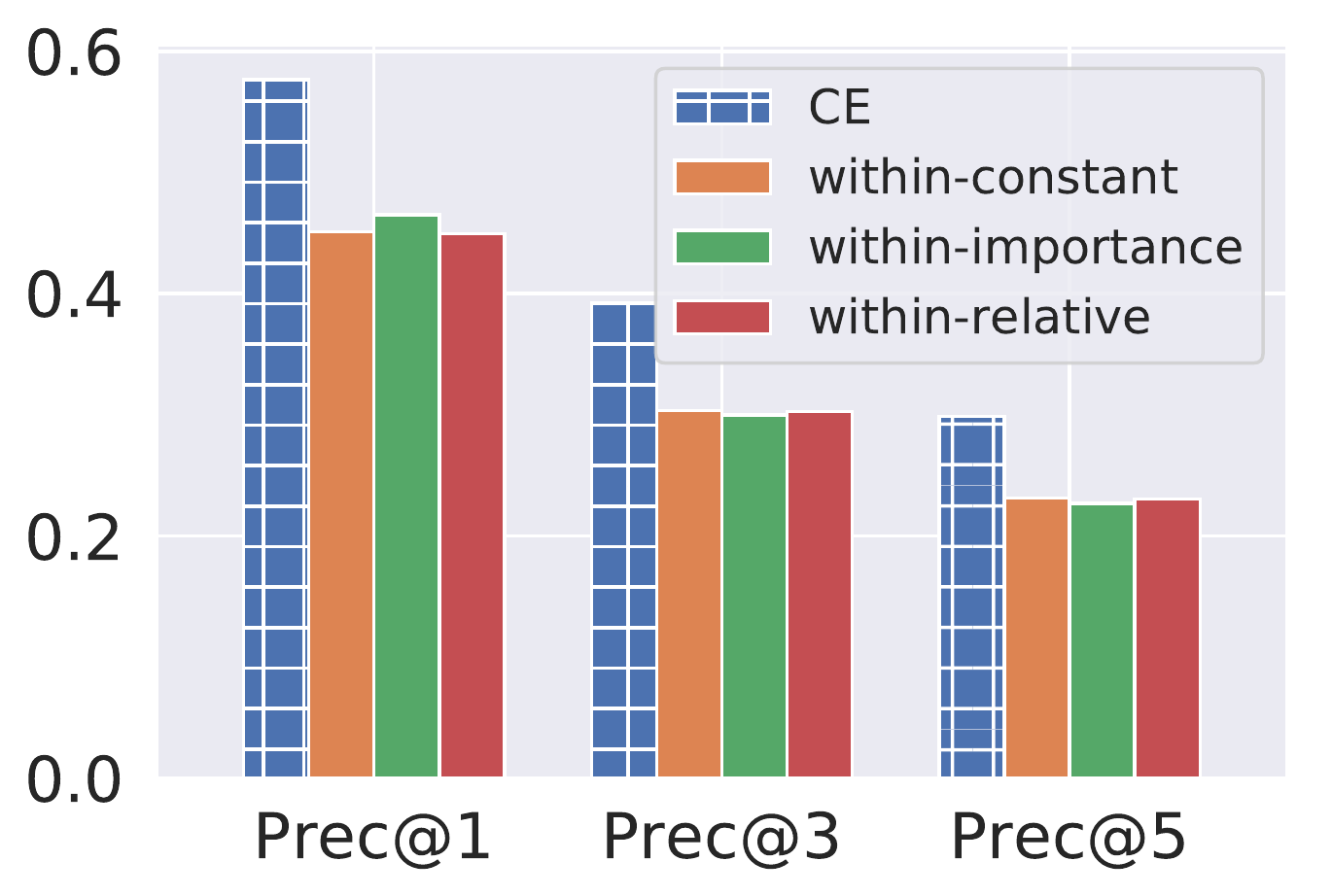}
    }
    \qquad
    \subcaptionbox{\delicioussmall.\label{fig:deli-p-w}}{
    \includegraphics[scale=0.33,valign=t]{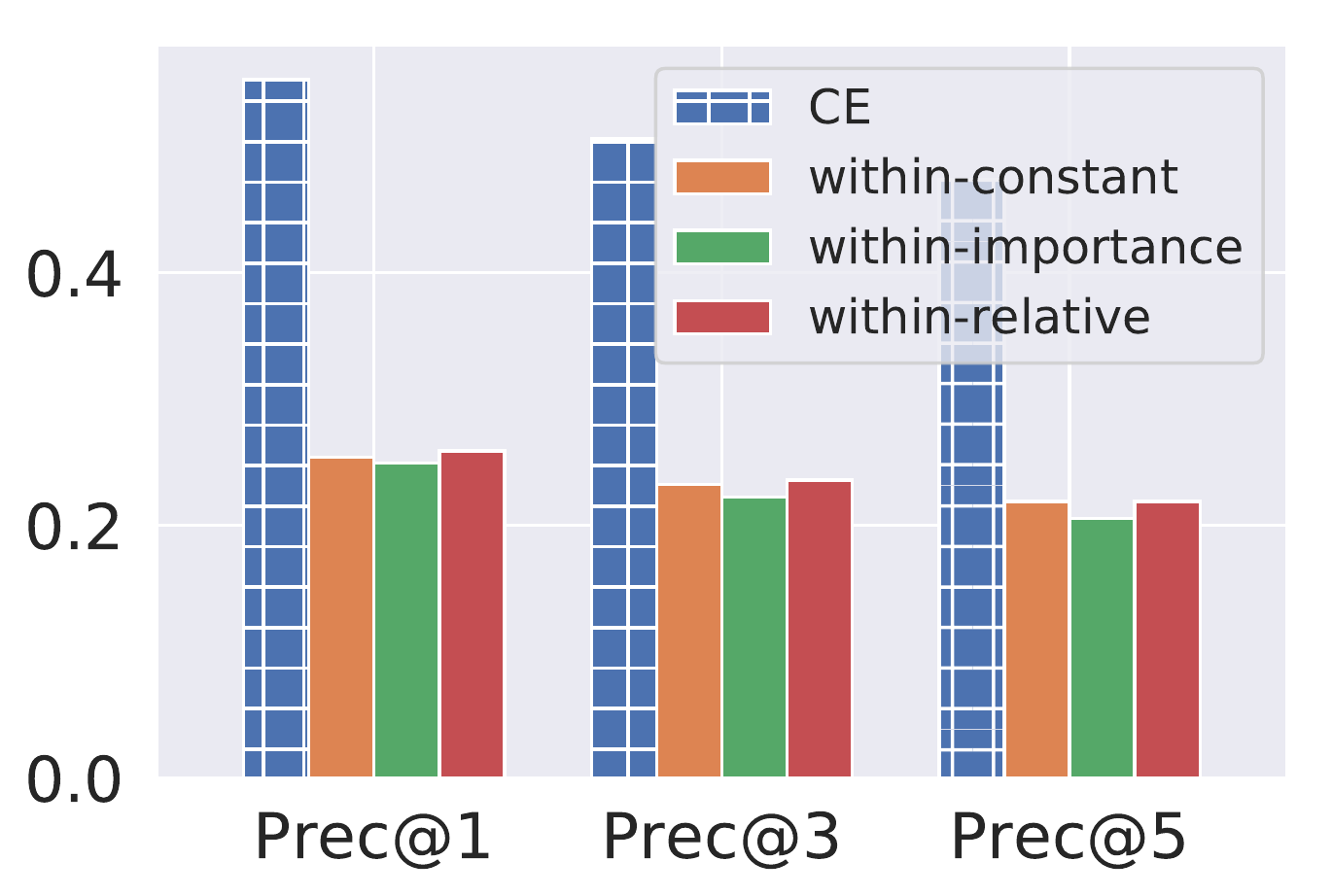}
    }
    }
    \resizebox{0.9\linewidth}{!}{
    \subcaptionbox{\amazonsmall.\label{fig:am-p-u}}{
    \includegraphics[scale=0.33,valign=t]{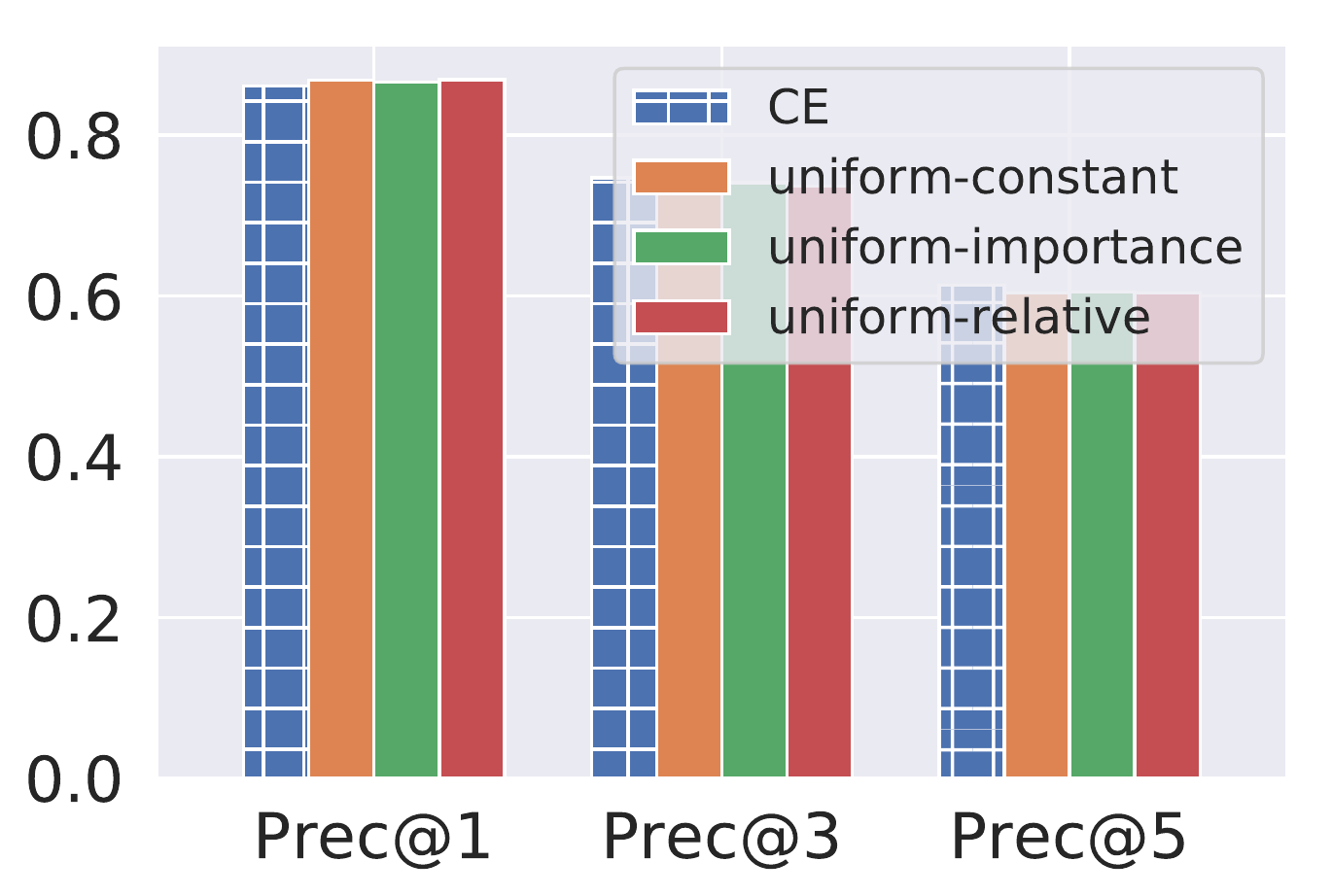}
    }
    \qquad
    \subcaptionbox{\wikilshtc.\label{fig:wiki-p-u}}{
    \includegraphics[scale=0.33,valign=t]{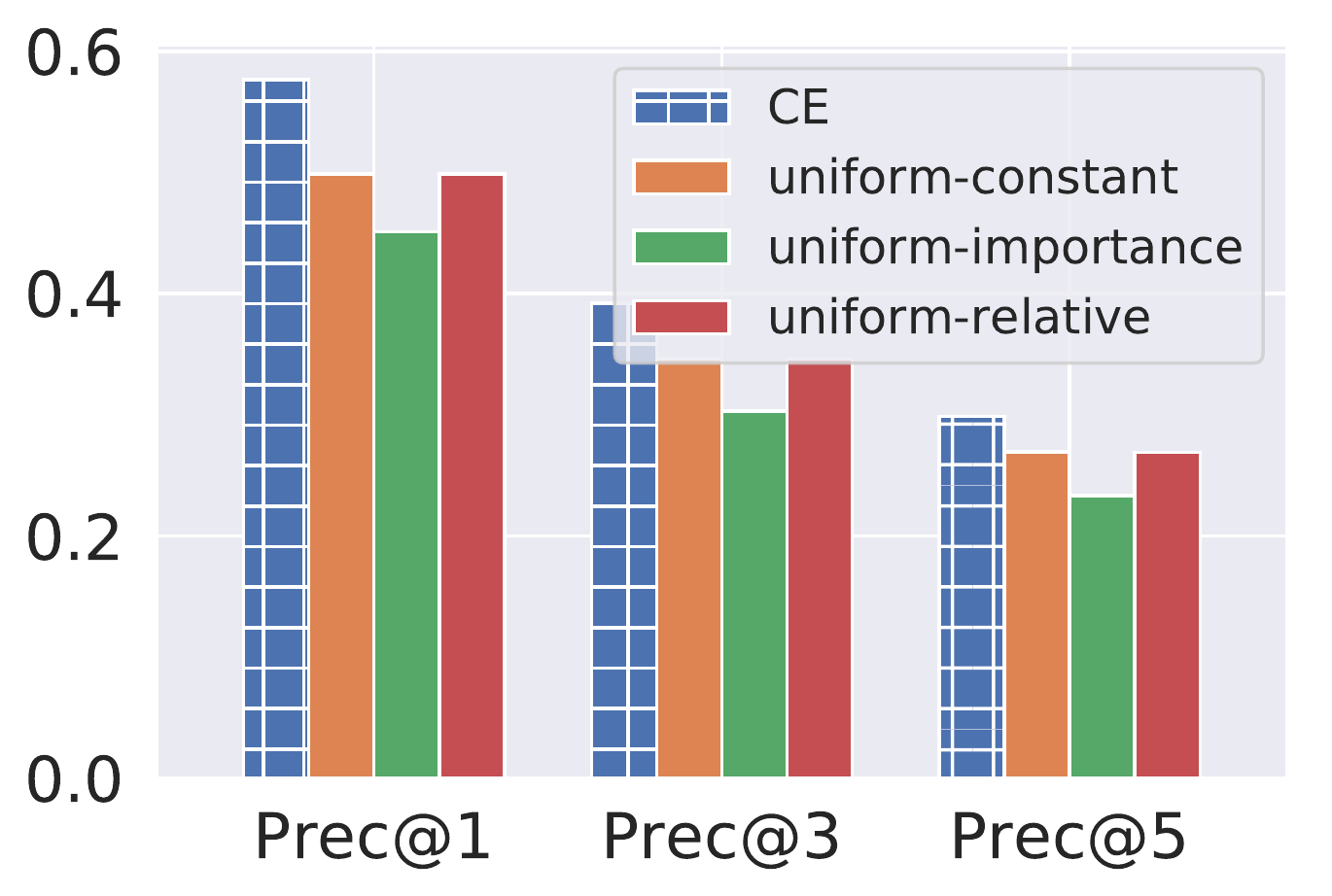}
    }
    \qquad
    \subcaptionbox{\delicioussmall.\label{fig:deli-p-u}}{
    \includegraphics[scale=0.33,valign=t]{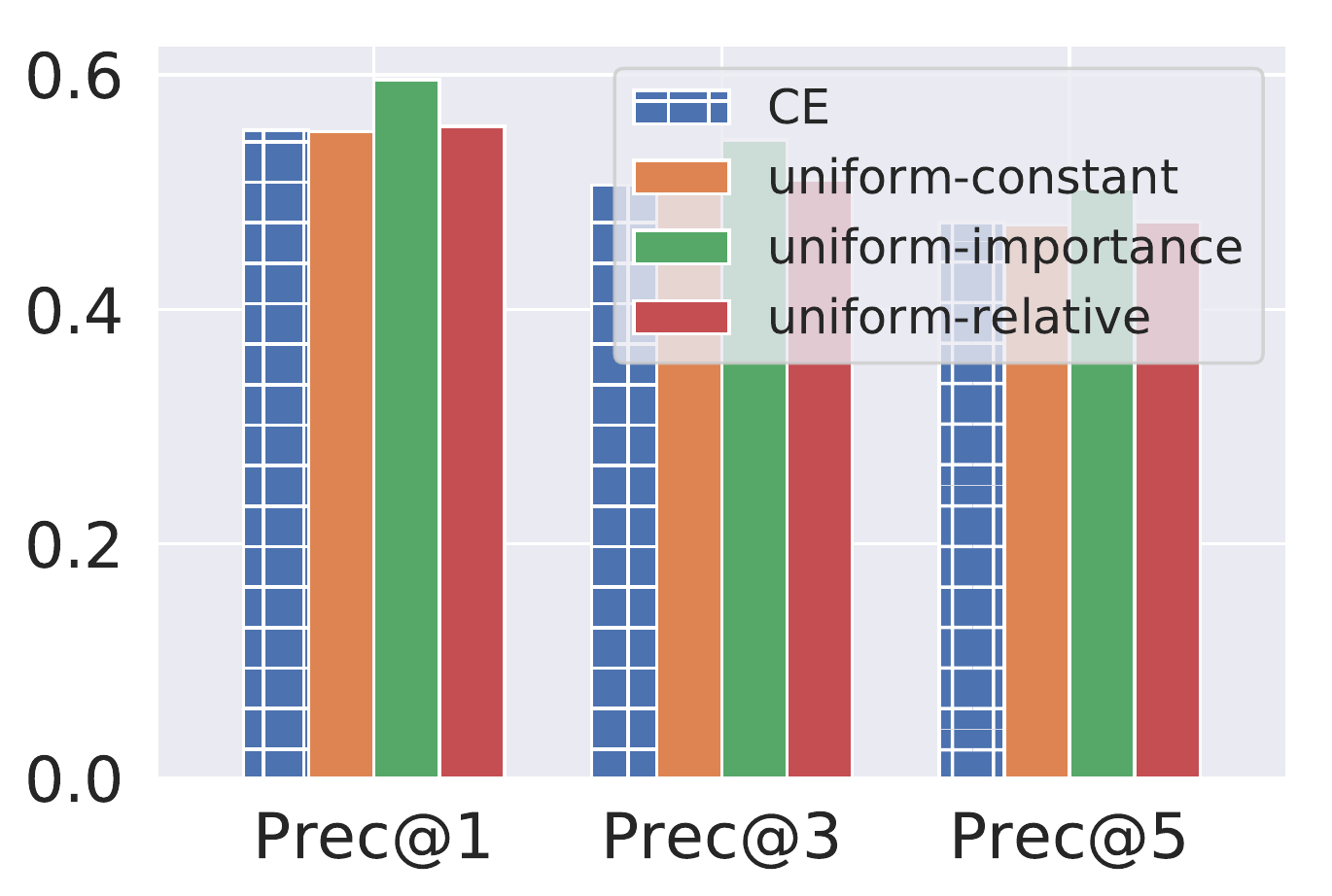}
    }
    }

    \caption{Performance of within-batch and uniform sampling on the entire original (multilabel) test sets of \amazonsmall, \wikilshtc, and \delicioussmall, as measured by $\precision{k}$ for $k$ = 1, 3, and 5. These experiments utilize $m=256$ negative for \amazonsmall{} and \wikilshtc{}, and $m = 64$ negatives for \delicioussmall{}. For reference, we include the results of standard softmax cross-entropy loss (${\tt ce}$). Note that these results are included here for completeness, as they are often reported in the literature as the performance measure. Since these results do not breakdown the performance on different subpopulations, they do not highlight the impact of sampling distribution and weighting schemes on various subpopulations.}
    \label{fig:rb-unif-v2}
\end{figure*}

\end{document}